\documentclass[11pt]{article}

\usepackage[margin=1in]{geometry}
\usepackage[mathscr]{eucal}
\usepackage{amsmath,amsthm,amssymb}
\usepackage{bbm}
\usepackage{authblk}
\usepackage{mathtools}
\usepackage{algorithm2e}
\RestyleAlgo{ruled}
\SetKwComment{Comment}{/* }{ */}
\usepackage{booktabs} 
\usepackage{hyperref}
\usepackage{url}
\usepackage{natbib}
\usepackage{breqn}
\usepackage{footnote}
\usepackage{graphicx}
\usepackage{tipa}
\usepackage{subfig}
\usepackage{bm}
\usepackage{paralist}
\usepackage{calrsfs}
\usepackage{subfiles}
\usepackage{xcolor}
\usepackage{mwe}
\usepackage{dutchcal}
\usepackage{verbatim}

\theoremstyle{plain}
\newtheorem{theorem}{Theorem}
\newtheorem{corollary}{Corollary}
\newtheorem{lemma}{Lemma}
\newtheorem{proposition}{Proposition}
\newtheorem{assumption}{Assumption}
\newtheorem{cond}{Condition}

\newtheorem{definition}{Definition}
\newtheorem{remark}{Remark}

\numberwithin{mytheorem}{section}

\numberwithin{mylemma}{section}
\newtheorem{innercustomgeneric}{\customgenericname}
\providecommand{\customgenericname}{}

\renewcommand{\hat}{\widehat}

     \def\EE{\mathbb{E}}

\DeclareMathOperator{\argmin}{argmin}

\newcommand{\R}{\mathbb{R}}

\DeclareMathAlphabet{\pazocal}{OMS}{zplm}{m}{n}
\newcommand{\ca}[1]{\pazocal{#1}}

\newcommand{\norm}[1]{\lVert#1\rVert}
\newcommand{\Norm}[1]{\left\lVert#1\right\rVert}

\newcommand{\lef}{\left}
\newcommand{\rig}{\right}

\newcommand{\pr}{\mathbb{P}}

\usepackage{mathtools}
\DeclarePairedDelimiterX{\infdivx}[2]{(}{)}{%
  #1\;\delimsize\|\;#2%
}

\newcommand{\TL}{\mbox{\scriptsize TL}}
\newcommand{\FT}{\mbox{\scriptsize FT}}

\newcommand{\SourceModel}{g^P(\boldsymbol{f}_i^P, \boldsymbol{u}_{i, \pazocal{J}^P}^P)}
\newcommand{\SourceFunction}{g^P(\boldsymbol{f}^P, \boldsymbol{u}_{\pazocal{J}^P}^P)}

\title{\huge\bf SMART Fine-tuning Factor Augmented Neural Lasso\footnote{Supported by  NSF Grant DMS-2412029 and the ONR Grant N00014-25-1-2317.}}

\author{Jinhang Chai \qquad Jianqing Fan  \qquad Cheng Gao\qquad  Qishuo Yin\\
Department of Operations Research and Financial Engineering\\
Princeton University}
\begin{document}
	\date{}
\maketitle
\vspace*{-0.3 in}

%\r{Fine-Tuning Factor augmented neural Lasso}
\begin{abstract}
  Fine-tuning is a widely used strategy for adapting pre-trained models to new tasks, yet its methodology and theoretical properties in high-dimensional nonparametric settings with variable selection have not yet been developed. We propose a source-model-augmented residual tuning (SMART) framework, which incorporates the pre-trained source model as an augmented feature into the target learner and estimates only the residual target-specific component. The approach is widely applicable, from parametric and sparse models to neural networks and blackbox machine learning models. We focus on the development of fine-tuning factor-augmented neural Lasso, resulting in SMART-FAN-Lasso. This transfer-learning framework for high-dimensional nonparametric regression with variable selection simultaneously handles covariate and posterior shifts. We use a low-rank factor structure to manage high-dimensional dependent covariates and a residual tuning decomposition in which the target function is expressed as a function of source model and other target-specific variables, thereby reducing the effective complexity of the target task. We derive minimax-optimal excess risk bounds, characterizing the precise conditions, in terms of relative sample sizes and function complexities, under which fine-tuning yields statistical acceleration over single-task learning. Extensive numerical experiments across diverse covariate- and posterior-shift scenarios demonstrate that SMART-FAN-Lasso consistently outperforms standard baselines and achieves near-oracle performance even under severe target sample size constraints, empirically validating the derived rates.
\end{abstract}

\section{Introduction}
\label{sec:intro}

% Transfer learning has emerged as a cornerstone of modern machine learning, with fine-tuning becoming indispensable to AI development---particularly in the era of large language models. The appeal of fine-tuning extends beyond knowledge transfer: it offers a pragmatic solution to the computational burden of training large-scale models from scratch, a process that can consume months of computation and millions of dollars. In practice, deployment data often follows distributions that differ from those seen during pre-training, yet the wealth of knowledge encoded in pre-trained weights remains valuable. Fine-tuning exploits this by adapting a small set of parameters to new domains while preserving the bulk of learned representations.

Transfer learning has revolutionized modern machine learning by enabling the transfer of learned representations from data-rich source domains to sample-limited target tasks. This paradigm is most prominently realized through fine-tuning, where pre-trained weights from large-scale models serve as a sophisticated starting point for specialized downstream applications. The effectiveness of this approach hinges on the hypothesis that pre-training captures intrinsic latent structures—such as shared features or factors—that provide a foundational representation of the data. Fine-tuning thus acts as a pragmatic adaptation layer, refining learned representations to align with the specific nuances of the target domain.

% Despite its empirical success, fine-tuning lacks a rigorous theoretical foundation. Practitioners routinely fine-tune pre-trained models, yet fundamental questions remain open: How can we leverage pre-trained models to accelerate learning on new tasks? How can we quantify the statistical gains from leveraging pre-trained knowledge? How should we design estimators that automatically exploit helpful source information while remaining robust to negative transfer? These questions demand precise, quantitative answers---not rules of thumb, but provable guarantees that reveal exactly when, why, and by how much fine-tuning succeeds.

While the empirical success of fine-tuning is undeniable, providing a unified and sharp theoretical characterization of its performance remains a formidable task. Existing literature has established foundational guarantees for transfer learning, yet these often focus on linear regimes or assume a restricted class of distribution shifts. There remains a critical need for a framework that can quantify the statistical gains of fine-tuning in the presence of high-dimensional covariates and complex nonparametric structures—settings where classical parametric intuition often fails. Specifically, it remains unclear how to construct a unified estimator that simultaneously achieves minimax optimality under both covariate and posterior shifts while maintaining robustness to negative transfer.

% Building such a theory confronts several formidable obstacles. Modern applications routinely involve \emph{high-dimensional covariates}, where the ambient dimension $p$ dwarfs the available sample size $n$, rendering classical statistical tools inadequate. The functions governing real-world phenomena are seldom simple; they exhibit \emph{complex nonparametric structure}---hierarchical compositions, intricate nonlinearities---that defies parametric modeling. This motivates factor-adjusted nonparametric Lasso (FAN-Lasso) type of techniques for statistical modeling by using neural networks, which is called factor-augmente sparse throughput neural networks (FAST-NN) in \cite{fan2024factor}.  Moreover, the distribution shift between source and target domains is rarely one-dimensional. In practice, we face \emph{covariate shift}, where the marginal distribution of inputs changes, alongside \emph{posterior shift}, where the conditional relationship between inputs and outputs evolves. Existing theoretical frameworks address these challenges piecemeal: some focus on low-dimensional or parametric settings, others assume only covariate shift with an invariant regression model. A unified treatment capable of handling high-dimensional, nonparametric regression under simultaneous covariate and posterior shift has remained elusive.

In this paper, we bridge these theoretical and practical gaps by proposing a source-model-augmented residual tuning (SMART) framework. We posit that the pretrained source model is a very strong feature in the target domain and augment it in the target-domain learning.  SMART incorporates a pre-trained source model into the target learner and estimates only the residual target-specific component. While this SMART framework is widely applicable, we focus on its development in a factor-augmented neural Lasso setting, resulting in SMART-FAN-Lasso. We assume that the shared knowledge between domains can be effectively encapsulated by a low-dimensional factor structure, which provides a stable backbone for knowledge transfer even under significant distribution shifts. Our approach provides an optimal mechanism to transfer pre-trained factor-augmented sparse throughput neural networks (FAST-NN, \cite{fan2024factor})—a class of models that combine latent factor extraction with sparse nonparametric estimation—to new environments.

% To address aforementioned  challenges,  we propose a fine-tuning factor-augmented neural Lasso (FAN-Lasso) framework that optimally transfers the pretrained FAST-NN model to new environments with possibly covariates and posterior shifts. This framework weaves together three powerful threads. First, the low-dimensional factor structure enables tractable estimation and provides a natural mechanism for transfer from source to target domain, though differing in distributions may share similar latent drivers. Second, \emph{deep ReLU networks} capture dense and sparse nonparametric complexity by leveraging their effective capacity to adaptively learn the intrinsic structure of the regression function, thereby sidestepping the curse of dimensionality.  Third, the flexibility of transfer functions allows the method to adapt to different applications with varying environments and sample sizes.

This framework weaves together three powerful components to address the aforementioned challenges. First, the low-dimensional factor structure enables tractable estimation in high-dimensional settings by capturing the latent drivers shared across domains, acting as a bridge for representation transfer. Second, we employ deep ReLU networks to model the dense and sparse nonparametric complexities, leveraging their adaptive capacity to unknown composition structures to sidestep the curse of dimensionality. Third, by introducing flexible transfer functions, our method explicitly accounts for both covariate and posterior shifts. This mechanism allows the estimator to automatically calibrate the degree of source-target similarity, effectively ``gating'' the amount of source information to utilize. This synergy not only ensures robustness against negative transfer by pruning irrelevant source signals but also achieves minimax optimality under general nonparametric regimes.

%\textbf{Current Challenges.}

%\textbf{Our Approach.}  
%To address aforementioned  challenges, 

%Let us use the subscripts $P$ and $Q$ to denote, respectively, the quantities related to the source and target.  First, we tame high-dimensional covariates and their dependence by positing that the factor model: $\boldsymbol{x} \in \mathbb{R}^p$ is governed by a handful of latent factors $\boldsymbol{f} \in \mathbb{R}^r$ via $\boldsymbol{x}^* = \boldsymbol{B}^* \boldsymbol{f}^* + \boldsymbol{u}^*$ for $* \in \{P, Q\}$.  Our key insight is a \emph{residual fine-tuning formulation} that decomposes the target regression function as
%\[
%g^Q(f^Q, u^Q_{J^Q}) = h\bigl(f^Q, u^Q_J, g^P(x^Q_{J^P})\bigr),
%\]
%where the pre-trained source nonparametric function $g^P$ provides a foundation for transfer learning and $h$ captures the residual transformation needed to accommodate the target distribution. When $h$ belongs to a simpler function class than the full target function $g^Q$, we need only estimate the simpler residual function --- this is precisely when fine-tuning yields provable gains over learning from scratch.  The versatility of the model stems from that of FAST-NN and the general form of the residual function.  This will be further elaborated in Section 2.

\subsection{Problem Formulation}\label{formulation}

%We formalize the fine-tuning problem within a neural factor-augmented regression framework. 
Consider two datasets: source data $\{(\boldsymbol{x}_i^P, y_i^P)\}_{i=1}^{n^P} \overset{\text{i.i.d.}}{\sim} P$ and target data $\{(\boldsymbol{x}_j^Q, y_j^Q)\}_{j=1}^{n^Q} \overset{\text{i.i.d.}}{\sim} Q$, where $(\boldsymbol{x}, y) \in \mathbb{R}^p \times \mathbb{R}$. The central question is: \emph{how and under what conditions does access to source data accelerate learning on the target distribution $Q$?}

\textbf{Factor Model for Covariates.}
To handle the high-dimensional regime where $p \gtrsim n^Q$, we assume the covariates admit a low-rank factor structure:
\begin{equation}
    \boldsymbol{x}^P = \boldsymbol{B}^P\boldsymbol{f}^P + \boldsymbol{u}^P, \qquad \boldsymbol{x}^Q = \boldsymbol{B}^Q\boldsymbol{f}^Q + \boldsymbol{u}^Q,  \label{eq:factorModel2}
\end{equation}
where $\boldsymbol{B}^P, \boldsymbol{B}^Q \in \mathbb{R}^{p \times r}$ are unknown loading matrices with $p \gg r$, $\boldsymbol{f}^P, \boldsymbol{f}^Q \in \mathbb{R}^r$ are latent factors, and $\boldsymbol{u}^P, \boldsymbol{u}^Q \in \mathbb{R}^p$ are idiosyncratic components.  No factor structure of covariates is allowed by taking $r=0$ or $\boldsymbol B = 0$.%This structure implies that although the observed covariates are high-dimensional, the relevant variation is captured by $r$ latent factors---enabling tractable estimation and providing a natural mechanism for transfer when source and target domains share similar latent structure.

\textbf{Regression Functions.}
%The response in each domain is generated by an unknown regression function.
For the source domain, we assume a factor augmented sparse nonparametric model:
\begin{equation}
    y_i^P = \SourceModel + \epsilon_i^P, \quad i \in [n^P],
    \label{eq:Pdata}
\end{equation}
where $g^P: \mathbb{R}^{r + |\pazocal{J}^P|} \to \mathbb{R}$ depends on both the latent factors and a sparse subset $\pazocal{J}^P \subset [p]$ of idiosyncratic components, and $\{\epsilon_i^P\}$ are i.i.d.\ noise. 
Note that the source model can also be written in the form $\tilde g^P(\boldsymbol{f}_j^P, \boldsymbol{x}_{j, \pazocal{J}^P}^P)$ using \eqref{eq:factorModel2} and the factor augmentation in variable selection can be seen.
%As shown in Section~\ref{sec:factor-transfer}, this form facilitates transfer from the source to the target domain.
Similarly, for the target domain, we assume a factor-augmented nonparametric model:
\begin{equation}
    y_j^Q = g^Q(\boldsymbol{f}_j^Q, \boldsymbol{u}_{j, \pazocal{J}^Q}^Q) + \epsilon_j^Q, \quad j \in [n^Q],
    \label{eq:Qdata}
\end{equation}
depending on both the latent factors and a sparse subset $\pazocal{J}^Q \subset [p]$ of idiosyncratic components.   
%The inclusion of $\boldsymbol{u}_{\pazocal{J}^Q}^Q$ captures settings where some covariates or idiosyncratic entries carry predictive information beyond what the factors encode.   
The versatility of the FAN model is discussed in Section~\ref{sec:versality}.

\textbf{Two Types of Distribution Shift.}
This framework accommodates two fundamental types of shift in transfer learning:
\begin{enumerate}
    \item \textbf{Covariate Shift:} The marginal distribution of covariates differs across domains, manifested through $\boldsymbol{B}^P \neq \boldsymbol{B}^Q$. The methodological challenge is to estimate the target factors $\boldsymbol{f}^Q$ accurately by leveraging source data when the loading matrices are similar but not identical.
    \item \textbf{Posterior Shift:} The conditional distribution of responses differs, i.e., $g^P \neq g^Q$. The statistical challenge is to learn the target function $g^Q$ efficiently by exploiting the pre-trained source model $g^P$.
\end{enumerate}

\textbf{Source model augmented residual tuning} (SMART).
To formalize how source knowledge transfers to the target, we  assume that the source model $s^P(\boldsymbol{x}) = \SourceFunction$ is a strong signal for the target regression and augment it in the target regression function via a residual structure:
\begin{equation}
    g^Q(\boldsymbol{f}^Q, \boldsymbol{u}_{\pazocal{J}^Q}^Q) = h\left(\boldsymbol{f}^Q, \boldsymbol{u}_{\pazocal{J}}^Q, s^P(\boldsymbol{x}^Q)\right),
    \label{eq:Fine-tune}
\end{equation}
where $h: \mathbb{R}^{r + |\pazocal{J}| + 1} \to \mathbb{R}$ is a low-complexity \emph{residual fine-tuning function}
and $\pazocal{J}^Q = \pazocal{J} \cup \pazocal{J}^P$. This formulation embeds the source model as an input to $h$, capturing how target predictions refine source predictions. Our model is more general than additive residual models $g^Q = g^P + h_0$ or  multiplicative $g^Q = g^P (1+ h_0)$  and permits nonlinear compositions, including regime-switching behaviors such as $h = h_0 \cdot \mathbf{1}\{g^P > 0\}$. Crucially, the source model remains frozen during fine-tuning; only $h$ is estimated from target data. This helps reducing considerably the complexity of $h$ by augmenting only a single variable $s(\boldsymbol{x})$, with the complexity of $s(\boldsymbol{x})$ learned from the plentyful of the source data.  This complexity reduction in transfer is the basic principle behind the SMART framework.  When $h$ is a simpler (including linear model or sparse linear in the low $n^Q$ setting) than $g^P$, this decomposition yields both computational savings and faster statistical convergence. This is particularly relevant when $\boldsymbol{u}_{\pazocal{J}^Q}^Q$ is absent from \eqref{eq:Fine-tune}, so that $h$ can be learned using a low-dimensional neural network. This point is elaborated further in Section~\ref{sec:model}.

\subsection{Preview of Main Results}

Our theoretical analysis makes the intuition behind SMART fine-tuning precise. We establish that the SMART-FAN-Lasso estimator achieves the minimax optimal excess risk
\[
\left[\frac{\log(n^P + n^Q)}{n^P + n^Q}\right]^{\frac{2\gamma^P}{2\gamma^P+1}} + \left(\frac{\log n^Q}{n^Q}\right)^{\frac{2\gamma}{2\gamma+1}} + \frac{\log p}{n^Q} + \frac{1}{p},
\]
where $\gamma^P$ and $\gamma$ quantify the complexity, such as the dimensionality-adjusted degree of smoothness,  of the source function $g^P$ and the residual function $h$, respectively. This rate admits a revealing decomposition: the first term reflects the joint contribution of source and target data to learning $g^P$---the complex backbone that both domains share---while the second term captures the irreducible cost of estimating the residual $h$ from target data alone. When source data is plentiful ($n^P \gg n^Q$) and the source function is genuinely more complex than the residual ($\gamma^P < \gamma$), fine-tuning delivers a provable acceleration over direct estimation of $g^Q$. We complement this upper bound with a matching minimax lower bound.
%, confirming that our rates are sharp up to logarithmic factors.

A distinctive virtue of our method is its \emph{automatic robustness to negative transfer}. When source data offers no genuine advantage---either because samples are scarce ($n^P \ll n^Q$) or because the residual $h$ is no simpler than the target function---our estimator gracefully degrades to the optimal rate achievable without any source data. This safeguard requires no oracle knowledge of the source-target relationship; it emerges organically from the structure of SMART-FAN-Lasso.

\subsection{Related Works}

%\textcolor{blue}{(Need to remove some of the reference because it is make it too long. )}

Our work lies at the intersection of transfer learning, neural network fine-tuning, deep learning theory, nonparametric variable selection, and factor models for high-dimensional data. While each of these fields has seen significant advancements, a unified theoretical framework for high-dimensional nonparametric transfer learning remains an open challenge.

\textbf{Transfer Learning and Distribution Shift.}
Transfer learning aims to leverage source domain knowledge to improve target performance, with foundations laid by \citet{pan2009survey} and \citet{ben2010theory}. In the context of high-dimensional regression, recent works have explored various facets of this problem: \citet{li2022transfer} establishes minimax rates under sparse parameter differences, while \citet{cai2024transfer} and \citet{tian2023transfer} extend these guarantees to nonparametric and generalized linear models. \citet{fan2025robust} proposes a TAB technique for transfer learning.  A primary obstacle is \emph{covariate shift}, where marginal distributions diverge across domains \citep{quinonero2022dataset, gretton2009covariate}. Recent advances in this area include optimal RKHS-based rates \citep{ma2023optimally}, characterizations of source-label utility \citep{kpotufe2021marginal}, robust estimation techniques \citep{yang2024doubly, cai2025semi}, and fundamental insights into well-specified covariate shift \citep{ge2023maximum}. Furthermore, theoretical inquiries into the value of data \citep{hanneke2019value}, task diversity \citep{tripuraneni2020theory}, and the provable advantages of pre-training \citep{ge2023provable} have significantly deepened our understanding. Building on these foundations, our SMART-FAN-Lasso framework provides a unified treatment for both covariate and posterior shift under complex nonparametric structures.

\textbf{Foundations of Fine-Tuning.}
Fine-tuning has become the foundational standard for deploying large-scale models, yet its theoretical properties are still to be unraveled. \citet{kumar2022fine} famously demonstrated that naive fine-tuning can distort pre-trained features, motivating the rise of parameter-efficient methods such as LoRA \citep{hu2022lora, dettmers2023qlora} and prompt tuning \citep{lester2021power}. Our work aligns with a growing ``residual" perspective on fine-tuning, where the target model is viewed as a refinement of the source model. This approach has gained traction across diverse domains, including proxy-based prediction \citep{bastani2021predicting}, cross-fitted residual regression \citep{zhou2023cross}, few-shot learning \citep{zhao2024residual}, and reinforcement learning \citep{ankile2025residual}. We formalize this intuition through our residual fine-tuning function $h$ in Assumption~\ref{ass:transferability}, providing a rigorous mathematical bridge between empirical residual-based methods and theory.

\textbf{Nonparametric Deep Learning Theory.}
The success of our approach relies on the representation power of deep ReLU networks. Building on optimal rates for smooth functions \citep{petersen2018optimal, lu2021deep}, recent studies have shown that deep networks can circumvent the curse of dimensionality by automatically exploiting hierarchical compositions \citep{schmidt2020nonparametric, kohler2021rate, fan2024noise, zhong2024neural}. \citet{farrell2021deep} further established high-probability bounds for such estimators. Our work utilizes these strengths to estimate the complex functions $g^P$ and $h$. By integrating complexity control through sparsity and regularization \citep{bartlett2019nearly, ohn2022nonconvex}, we connect modern deep learning theory with classical high-dimensional factor model frameworks to achieve minimax optimality.

\textbf{High-Dimensional Factor Models.}
Factor models provide the necessary low-rank structure to handle high-dimensional covariates $p \gg n$. It has various applications in econometrics \citep{stock2002forecasting, stock2002macroeconomic, forni2005generalized, bai2008large}, and since then, the asymptotic behavior of these models has been extensively characterized \citep{paul2007asymptotics, johnstone2009consistency, onatski2012asymptotics, chudik2011weak, wang2017asymptotics}. Recent work combines pretrained factor estimation \citep{fan2022learning} with deep learning. \citet{fan2024factor} established the minimax optimality of FAST-NN for single-domain high-dimensional nonparametric regression. Our SMART-FAN-Lasso extends this framework to transfer learning and provides a unified theory for residual fine-tuning under simultaneous distribution shifts.

\subsection{Notation and preliminaries}

%We introduce notation used throughout the paper. For any positive integer $m$, define $[m]=\{1, \dots, m\}$. We use $\boldsymbol{x} = (x_1, \dots, x_p)^\top$ to denote a $d$-dimensional vector and $\boldsymbol{A} = [a_{ij}]_{i\in[n], j\in[m]}$ to denote an $n\times m$ matrix. The $p\times p$ identity matrix is denoted by $\boldsymbol{I}_p$. 
For vectors $\boldsymbol{x} \in \R^p$ and $\boldsymbol{y} \in \R^{d}$, $[\boldsymbol{x}, \boldsymbol{y}]$ concatenates them into a $(p + d)$-dimensional vector. Let $\norm{\boldsymbol{x}}_q = (\sum_{i}|x_i|^q)^{1/q}$ be the vector $\ell_q$ norm. For matrices, let $\norm{\boldsymbol{A}}_2 = \sup_{\norm{\boldsymbol{x}}_2 = 1} \norm{\boldsymbol{Ax}}_2$, $\norm{\boldsymbol{A}}_F=\sqrt{\sum_{i,j}a_{ij}^2}$, and $\norm{\boldsymbol{A}}_\max = \max_{i,j}|a_{ij}|$. We use $\lambda_{\min}(\boldsymbol{A})$ and $\nu_{\min}(\boldsymbol{A})$ to denote the minimum eigenvalue and singular value of $\boldsymbol{A}$, respectively.
% Let $\mathbbm{1}\{\cdot\}$ denote the indicator function and $\mathrm{supp}(\boldsymbol{X})$ the support of random variable $\boldsymbol{X}$. Let $\boldsymbol{X} \sim \boldsymbol{Y}$ indicate that $\boldsymbol{X}$ and $\boldsymbol{Y}$ have the same distribution. We write $a\land b = \min(a,b)$ and $a\lor b = \max(a,b)$. The notation $a_{n}\lesssim b_{n}$ means $|a_{n}|\leq c|b_{n}|$ for some constant $c>0$ when $n$ is sufficiently large; $a_{n}\gtrsim b_{n}$ and $a_{n}\asymp b_{n}$ are defined analogously. We write $a_{n}\ll b_{n}$ if $|a_{n}|/|b_{n}|\rightarrow 0$. These symbols suppress constant dependencies on both sides. 

\textbf{ReLU Neural Networks.}
%We use neural networks as a scalable, nonparametric technique for model fitting.  Specifically, 
We consider fully connected neural networks with ReLU activation $\sigma(x) = x_+$. Given depth $L$ and width $N$, a deep ReLU neural network maps from $\mathbb{R}^d$ to $\mathbb{R}$ by
\begin{align}
\label{eq:nn-def}
    g(\boldsymbol{x}) = \pazocal{L}_{L+1} \circ \bar{\sigma}_L \circ \pazocal{L}_L \circ \bar{\sigma}_{L-1} \circ \cdots \circ \pazocal{L}_2 \circ \bar{\sigma}_1 \circ \pazocal{L}_1(\boldsymbol{x}),
\end{align}
where $\pazocal{L}_\ell(\boldsymbol{z}) = \boldsymbol{W}_\ell \boldsymbol{z} + \boldsymbol{b}_\ell$ is an affine transformation with weight matrix $\boldsymbol{W}_\ell \in \mathbb{R}^{d_\ell \times d_{\ell - 1}}$ and bias vector $\boldsymbol{b}_\ell \in \mathbb{R}^{d_\ell}$. $(d_0, d_1, \dots, d_L, d_{L+1}) = (d, N, \dots, N, 1)$, and $\bar{\sigma}_\ell$ applies ReLU element-wise.

\begin{definition}[Deep ReLU network class]
    Define the family of deep ReLU networks with input dimension $d$, depth $L$, width $N$, output truncated by $M$, and weights bounded by $T$ as:
    \begin{align*}
        \pazocal{G}(d, L, N, M, T) = \{\tilde{g}(\boldsymbol{x}) = \mathrm{trun}_M (g(\boldsymbol{x})): g(\boldsymbol{x}) \text{ as in } \eqref{eq:nn-def}\text{ with } \norm{\boldsymbol{W}_\ell}_\max \leq T, \norm{\boldsymbol{b}_\ell}_\max \leq T\},
    \end{align*}
    where $\mathrm{trun}_M(z) = \max\{|z|, M\} \cdot \mathrm{sign}(z)$ is the truncation operator.
\end{definition}

\textbf{H\"older Smoothness.}
Let $\beta > 0$ and $C > 0$. A $d$-variate function $g$ is $(\beta, C)$-smooth if for every $\boldsymbol{\alpha} \in \mathbb{N}^d$ with $\norm{\boldsymbol{\alpha}}_1 = \lfloor \beta \rfloor$, the partial derivative $\partial g^{\lfloor \beta \rfloor} / \partial \boldsymbol{x}^{\boldsymbol{\alpha}}$ exists and satisfies
\begin{align*}
    \Big| \frac{\partial g^{\lfloor \beta \rfloor}}{\partial \boldsymbol{x}^{\boldsymbol{\alpha}}}(\boldsymbol{x}) - \frac{\partial g^{\lfloor \beta \rfloor}}{\partial \boldsymbol{x}^{\boldsymbol{\alpha}}}(\boldsymbol{z}) \Big| \leq C \norm{\boldsymbol{x} - \boldsymbol{z}}_2^{\beta-\lfloor\beta\rfloor}.
\end{align*}
We denote the set of all $d$-variate $(\beta, C)$-smooth functions by $\pazocal{F}_{d, \beta, C}$. Throughout this paper, a linear transformation $\boldsymbol{Ax} + \boldsymbol{b}$ is treated as $(\infty, \norm{\boldsymbol{A}}_2)$-smooth.

%\subsection{Organization}

%The remainder of this paper is organized as follows. Section~\ref{sec:model} introduces the high-dimensional transfer learning framework, Section~\ref{sec:method} presents our methodology, Section~\ref{sec:factorTransfer} provides theoretical guarantees for factor transfer, Section~\ref{sec:fineTuning} develops the theory for the fine-tuning estimator, and Section~\ref{sec:sim} validates the effectiveness of our proposed methods through numerical studies. Technical proofs and additional results are deferred to the appendices.

% nonparametric regression framework with factor modeling for both source and target domains, establishing the hierarchical decomposition of regression functions that motivates our fine-tuning approach. 
%Section~\ref{sec:method} presents our methodology: the diversified projection matrix for robust factor estimation, the transfer factor estimation method for handling covariate shift, and the fine-tuning procedure for deep ReLU neural networks under posterior shift.
%Section~\ref{sec:factorTransfer} provides theoretical guarantees for transferred factor estimation , establishing convergence rates under appropriate conditions on factor loading similarity. Section~\ref{sec:fineTuning} develops the theory for the fine-tuning estimator, deriving both oracle-type and optimal upper bounds on excess risk. Section~\ref{sec:sim} presents simulation studies validating the effectiveness of our proposed methods. Technical proofs and additional results are deferred to the appendices.

% !TEX root = ../main.tex
\section{Model}
\label{sec:model}

\subsection{Factor augmented nonparametric  model}

We continue to use the notation introduced in \S\ref{formulation} and focus on a factor-augmented nonparametric (FAN) regression framework.  Specifically, we assume
\begin{align*}
    \mathbb{E}[y^P | \boldsymbol{f}^P, \boldsymbol{u}^P] &= \SourceFunction, \qquad
    \mathbb{E}[y^Q|\boldsymbol{f}^Q, \boldsymbol{u}^Q] = g^Q(\boldsymbol{f}^Q, \boldsymbol{u}_{\pazocal{J}^Q}^Q),
\end{align*}
where $\pazocal{J}^P, \pazocal{J}^Q \subset \{1, \ldots, p\}$ are unknown subsets of indices representing the covariates that possess additional predictive power beyond the common factors.  Note that given the latent factor $\boldsymbol{f}^Q$, $\boldsymbol{u}_{\pazocal{J}^Q}^Q$ is equivalent to $\boldsymbol{x}_{\pazocal{J}^Q}^Q$ so that $g^Q(\boldsymbol{f}^Q, \boldsymbol{u}_{ \pazocal{J}^Q}^Q) \equiv  \tilde g^Q(\boldsymbol{f}^Q, \boldsymbol{x}_{\pazocal{J}^Q}^Q)$, a factor augmented sparse nonparametric model.  Let $\epsilon^P = y^P - \mathbb{E}[y^P | \boldsymbol{f}^P, \boldsymbol{u}^P]$ and $\epsilon^Q = y^Q - \mathbb{E}[y^Q | \boldsymbol{f}^Q, \boldsymbol{u}^Q]$ denote the noise terms. Then, the data-generating process is summarized by 
\begin{equation}
    \label{eq:setting}
    \begin{split}
        \boldsymbol{x}_i^P &= \boldsymbol{B}^P\boldsymbol{f}_i^P + \boldsymbol{u}_i^P, \quad y_i^P = \SourceModel + \epsilon_i^P, \quad \forall i \in [n^P],\\
        \boldsymbol{x}_j^Q &= \boldsymbol{B}^Q \boldsymbol{f}_j^Q + \boldsymbol{u}_j^Q, \quad y_j^Q = g^Q(\boldsymbol{f}_j^Q, \boldsymbol{u}_{j,\pazocal{J}^Q}^Q) + \epsilon_j^Q, \quad \forall j \in [n^Q].
    \end{split}
\end{equation}
Our primary objective is to improve the estimation of the target regression function by leveraging information from the source domain through fine-tuning. 

For any candidate estimator $m: \mathbb{R}^p \to \mathbb{R}$, its performance is evaluated by the population $L_2$ risk on the target distribution:
\begin{align*}
    R(m) = \mathbb{E}[|y^Q - m(\boldsymbol{x}^Q)|^2].
\end{align*}
%By \eqref{eq:factorModel2}, the target regression mean can be written as $g^Q(\boldsymbol{f}^Q, \boldsymbol{u}_{\pazocal{J}^Q}^Q)$. Although $\boldsymbol{u}^Q$ is latent, once $\boldsymbol{f}^Q$ is estimated as in Section~\ref{sec:method}, $\boldsymbol{u}_{\pazocal{J}^Q}^Q$ is determined by $\boldsymbol{x}_{\pazocal{J}^Q}^Q - [\boldsymbol{B}_{\pazocal{J}^Q,:}^Q]\boldsymbol{f}^Q$. 
We therefore define the excess risk of an empirical estimator $\widehat{m}$ on the target domain as
\begin{align}
    \label{eq:excessRisk}
    \pazocal{E}^Q(\widehat{m}) \coloneqq \mathbb{E} \left[ |\widehat{m}(\boldsymbol{x}^Q) - g^Q(\boldsymbol{f}^Q, \boldsymbol{u}_{\pazocal{J}^Q}^Q)|^2 \right] = R(\widehat{m}) - R(g^Q).
\end{align}
Our goal is to develop a fine-tuning strategy that effectively leverages the source domain information to enhance the target estimation. 
%By successfully transferring the shared structure from the source data, we aim to improve the learning efficiency and performance of the resulting estimator $\widehat{m}$, making the fine-tuning process more effective than training solely on the target data.
%, especially when the target sample size is much smaller than the source sample size.

\subsection{Versatility of the FAN model}
\label{sec:versality}

The FAN model \eqref{eq:setting}, or specifically \eqref{eq:Qdata}, uses both common factors and idiosyncratic components in the regression. Every covariate contributes through the latent common factors, and some covariates (or equivalently, idiosyncratic components) provide additional contributions. It is very versatile and includes the sparse nonparametric model $g^Q(\boldsymbol{f}_j^Q, \boldsymbol{u}_{j, \pazocal{J}^Q}^Q) 
= \tilde g (\boldsymbol{x}_{j, \pazocal{J}^Q}^Q)$ as a specific case. When the covariates are weakly correlated, we can take $r = 0$ (no factors), in which case the model reduces to a sparse nonparametric model; when $g^Q$ is linear, it reduces to a sparse linear model. When $\pazocal{J}^Q = \emptyset$, it becomes a nonparametric factor model; when $g^Q$ is linear, it reduces to a principal component regression model. In general, the model allows dependence among covariates and is insensitive to slight overselection of $r$, which is equivalent to the last few columns of $\boldsymbol{B}^Q$ being zero.  %See \cite{fan2024factor} for details.

\subsection{Hierarchical decomposition of regression functions}
\label{sec:mainAss}

The convergence rate of the excess risk is fundamentally determined by the structural properties of the regression functions $g^P$ and $g^Q$.  A specific structure is to assume these functions are compositions of several lower-dimensional functions. This is formally captured by the \emph{hierarchical composition model} (HCM) \citep{kohler2021rate}. Roughly speaking, the HCM is a finite ($l$-fold) composition of $t$-dimensional functions with $(\beta, C)$-smoothness, for $(t, \beta)$ in a prescribed set $\pazocal{P}$. This structure is particularly well-suited to deep ReLU networks, which can adaptively learn such unknown structures.  %hierarchical compositions without requiring explicit knowledge of the functional forms.  \cite{fan2025factor, chai2025deep}.

\begin{definition}[Hierarchical composition model] 
\label{def:hcm}
The function class $\pazocal{H}(d, l, \pazocal{P}, C)$, where $l, d \in \mathbb{N}^+$, $C > 0$, and $\pazocal{P} \subset [1, \infty] \times \mathbb{N}^+$, is defined as follows. For $l = 1$,
\begin{align*}
    \pazocal{H}(d, 1, \pazocal{P}, C) = \{h : \R^d \rightarrow \R : h(x) = g(x_{\pi(1)}, \dots, x_{\pi(t)}), \text{ where } \\
    g : \R^t \rightarrow \R \text{ is } (\beta, C)\text{-smooth for some } (\beta, t) \in \pazocal{P} \text{ and } \pi: [t] \rightarrow [d]\}.
\end{align*}
For $l > 1$,
\begin{align*}
    \pazocal{H}(d, l, \pazocal{P}, C) = \{h : \R^d \rightarrow \R : h(x) = g(f_1(\boldsymbol{x}), \dots, f_t(\boldsymbol{x})), \text{ where } \\
    g : \R^t \rightarrow \R \text{ is } (\beta, C) \text{-smooth for some } (\beta, t) \in \pazocal{P} \text{ and } f_i \in \pazocal{H}(d, l - 1, \pazocal{P}, C)\}.
\end{align*}
\end{definition}

To simplify the presentation, we assume all component functions have a smoothness parameter $\beta \geq 1$, namely, $ \min_{(\beta, t) \in \pazocal{P}} \beta \geq 1$. The complexity of the regression function is determined by its ``hardest'' component. We define the effective smoothness $\gamma(\pazocal{P})$ as
\begin{equation}
\label{eq:hardestComponent}
    \gamma(\pazocal{P}) = \min_{(\beta, t) \in \pazocal{P}} \frac{\beta}{t}.
\end{equation}
It represents the smallest dimensionality-adjusted degree of smoothness \citep{fan2024factor} and serves as a measure of functional complexity.

\subsection{Source model transferability}
\label{sec:transfer}

We introduce our main assumption regarding the relationship between $g^P$ and $g^Q$, which formalizes the rationale behind the SMART framework and allows considerable flexibility in transfer.

\begin{assumption}[Source model augmented residual tuning]
\label{ass:transferability}
The target regression function admits
\begin{equation}
    \label{eq:transfer}
    g^Q(\boldsymbol{f}^Q, \boldsymbol{u}_{\pazocal{J}^Q}^Q) = h\left(\boldsymbol{f}^Q, \boldsymbol{u}_{\pazocal{J}}^Q, s^P(\boldsymbol{x}^Q)\right),
\end{equation}
where $s^P(\boldsymbol{x}) =  g^P(\boldsymbol{f}^P, \boldsymbol{u}_{\pazocal{J}^P}^P)$  is the source model 
\footnote{To clarify notation, in $s^P(\boldsymbol{x})$, $\boldsymbol{x}$ is decomposed into factor $\boldsymbol{f}$ and $\boldsymbol{u}$ according to the source model.  
%Namely, under the identifiable condition $ (\boldsymbol{B}^P)^\top  \boldsymbol{B}^P =  \boldsymbol{I}$,  we take $\boldsymbol{f} = (\boldsymbol{B}^P)^\top \boldsymbol{x}$ and $\boldsymbol{u} = (\boldsymbol{I}- \boldsymbol{B}^P (\boldsymbol{B}^P)^\top) \boldsymbol{x}$.
{Namely, under the identifiable condition $ (\boldsymbol{B}^P)^\top  \boldsymbol{B}^P$ is full-rank,  we take $\boldsymbol{f} = ((\boldsymbol{B}^P)^\top  \boldsymbol{B}^P)^{-1}(\boldsymbol{B}^P)^\top \boldsymbol{x}$ and $\boldsymbol{u} = (\boldsymbol{I}- \boldsymbol{B}^P ((\boldsymbol{B}^P)^\top  \boldsymbol{B}^P)^{-1}(\boldsymbol{B}^P)^\top) \boldsymbol{x}$.}
This applies to the covariate $\boldsymbol{x}^Q$, with the subset of variables $\pazocal{J}^P$ also transferred to the target domain. In the specific case $s^P(\boldsymbol{x}) =  g^P(\boldsymbol{x}_{\pazocal{J}^P}^P)$, there is no ambiguity in notation.
%  In our FAST-NN implementation, this part of the decomposition is implicitly written as a sparse linear combination of $ \boldsymbol{x}$, and hence the decomposition is not needed.
} 
applied to $ \boldsymbol{x}$ and $\pazocal{J}^Q = \pazocal{J} \cup \pazocal{J}^P$ for some subset $\pazocal{J} \subset [p]$.  We refer to $h$ as the \emph{residual fine-tuning function}. In addition, we assume $g^P \in \pazocal{H}(r+|\pazocal{J}^P|, l^P, \pazocal{P}^P, c_0)$ and $h \in \pazocal{H}(r + |\pazocal{J}| + 1, l, \pazocal{P}, c_0)$.
\end{assumption}

Assumption \ref{ass:transferability} models the scenario where the source model provides a useful feature or ``template'' for the target task. This formulation generalizes several popular transfer learning and domain adaptation models \citep{zhou2023cross, zhao2024residual, ankile2025residual}. A specific instance is $g^Q(\boldsymbol{f}^Q, \boldsymbol{u}^Q_{\pazocal{J}^Q}) = h \lef(\boldsymbol{f}^Q,  s^P(\boldsymbol{x}^Q)\rig)$ in which $h$ is only a $(r+1)$-dimensional function and can easily be learned by using neural networks. It includes the models $g^Q(\boldsymbol{f}^Q, \boldsymbol{u}^Q_{\pazocal{J}^Q}) = a^Q(\boldsymbol{f}^Q) s^P(\boldsymbol{x}^Q) + b^Q(\boldsymbol{f}^Q)$  as a specific example. In particular, $a^Q = 1$ yields a translation model and $b^Q = 0$ gives a multiplicative model.
%In the latter case, we need only to fit the residuals in the target data against the learned factors $\boldsymbol{f}^Q$.

A key feature of SMART is the use of the univariate feature $s^P(\boldsymbol{x}^Q)$, the source model evaluated directly on the observable target covariates.  This delegates the complexity of learning  $g^P(\boldsymbol{f}^Q, \boldsymbol{u}^Q_{\pazocal{J}^P})$ to the source data and reduces the complexity in the target domain modeling.  This also allow us to incorporate any off-the-shelf model, including a black-box machine learning model, as a pre-trained backbone.
%\begin{itemize}
%    \item \textbf{Computational Efficiency:} Using observable covariates $\boldsymbol{x}^Q$  avoids the step of estimating latent factors for the target domain before applying the pre-trained model.
%    \item \textbf{Robustness to Distribution Shift:} This formulation is more robust to the loading matrix shifts ($\boldsymbol{B}^Q \neq \boldsymbol{B}^P$). By bypassing factor estimation, we prevent the propagation of errors that would occur if the source model were applied to incorrectly estimated target factors.
%    \item \textbf{Implicit Information Capture:} Since $\boldsymbol{x}^Q_{\pazocal{J}^P} = [\boldsymbol{B}^Q_{\pazocal{J}^P, :}]\boldsymbol{f}^Q + \boldsymbol{u}^Q_{\pazocal{J}^P}$, the observable covariates already encode the relevant factor information. Explicitly including factors as separate arguments would introduce redundancy without increasing the model's expressive power.
%    \item \textbf{Compatibility with Standard Pre-trained Models:} The formulation does not necessitate that the source model be trained within a factor-augmented framework. Because $g^P$ operates directly on observable features, our approach can seamlessly incorporate any off-the-shelf regression model as a pre-trained backbone, significantly expanding the range of applicable source models beyond those specifically designed for latent factor analysis.
%\end{itemize}

The proposition below shows that the full target function $g^Q$ also belongs to an HCM class, and its complexity is determined by its most complex component.

\begin{proposition}
\label{prop:transferabilityBenefit}
Under Assumption \ref{ass:transferability}, we have
\begin{align*}
    %g^P(\boldsymbol{f},  \boldsymbol{u}_{\pazocal{J}^P}  ) & \in \pazocal{H}(r + |\pazocal{J}^P|, l^P + 1, \pazocal{P}^P \cup \{(\infty, r + |\pazocal{J}^P|)\}, c_0 \lor \norm{\boldsymbol{B}^P_{\pazocal{J}^P, :}}_2), \\
    g^Q(\boldsymbol{f}, \boldsymbol{u}_{\pazocal{J} \cup \pazocal{J}^P}) & \in \pazocal{H}(r + |\pazocal{J} \cup \pazocal{J}^P|, l + l^P + 1, \pazocal{P} \cup \pazocal{P}^P \cup \{(\infty, r + |\pazocal{J}^P|)\}, c_0 \lor \norm{\boldsymbol{B}^P_{\pazocal{J}^P, :}}_2).
\end{align*}
Furthermore, the combined complexity of the latter (target function class) satisfies $\gamma(\pazocal{P} \cup \pazocal{P}^P \cup \{(\infty, r + |\pazocal{J}^P|)\}) = \min\{\gamma, \gamma^P\}$.
\end{proposition}

Proposition \ref{prop:transferabilityBenefit} provides the theoretical foundation for the benefits of fine-tuning. It shows that while the target function $g^Q$ is a composition of both $g^P$ and $h$, and its overall complexity—which dictates the minimax convergence rate—is determined by the ``bottleneck" or the most complex component, $\min\{\gamma, \gamma^P\}$.  This result reveals a significant opportunity for transfer learning. If we were to estimate $g^Q$ from scratch using only target data, the convergence rate of the excess risk would be limited by this joint complexity. However, in the fine-tuning paradigm, we leverage the fact that $g^P$ has already been learned from a potentially much larger source dataset ($n^P \gg n^Q$). By substituting a high-quality pre-trained estimator $\widehat{g}^P$ into the decomposition \eqref{eq:transfer}, the remaining learning task reduces to estimating the residual function $h$, which is of lower dimension and typically lower complexity.

%The strategic advantage of fine-tuning emerges when the residual function $h$ is ``simpler" than the original source function, i.e., $\gamma > \gamma^P$. In such cases, the fine-tuning process focuses on a function class with higher effective smoothness (or lower complexity), leading to a faster convergence rate for the excess risk on the target domain. This logic suggests that fine-tuning is not merely about starting from a better initialization, but about fundamentally changing the complexity of the non-parametric estimation task from one governed by $\gamma^P$ to one governed by $\gamma$.

% !TEX root = ../main.tex
\section{Methodology}\label{sec:method}

In this section, we present our methodology for high-dimensional factor-augmented nonparametric variable selection using transfer learning. Our approach addresses two primary types of domain shifts: (i) \emph{covariate shift}, where source data is used to improve the estimation of target latent factors $\boldsymbol{f}^Q$, and (ii) \emph{posterior shift}, where the pre-trained source regression function $g^P$ is leveraged to enhance the estimation of the target regression function $g^Q$ through SMART fine-tuning.

\subsection{Diversified projection matrix for factor estimation}

%Traditional neural networks often use the high-dimensional covariates $\boldsymbol{x}$ directly as input. In contrast, our method incorporates a factor-augmented structure by introducing a \emph{diversified projection matrix} $\boldsymbol{W}$ as a pre-processing module. This matrix facilitates the estimation of the latent factors $\boldsymbol{f}$ from the observable $\boldsymbol{x}$.

Let $\overline{r} \geq r$ be the working number of latent factors. Following the construction in \cite{fan2022learning}, we define the diversified projection matrix as follows.

\begin{definition}[Diversified projection matrix]
\label{def:dpm}
    For any domain $* \in \{P, Q\}$, a matrix $\boldsymbol{W}^* \in \R^{p \times \overline{r}}$ is a \emph{diversified projection matrix} if it satisfies:
    \begin{enumerate}
    \item (Boundedness) $\norm{\boldsymbol{W}^*}_{\max} \leq C_1$ for a universal constant $C_1 > 0$;
    \item (Exogeneity) $\boldsymbol{W}^*$ is independent of the data $\{\boldsymbol{x}^*_i, y^*_i\}$;
    \item (Significance) The matrix $\boldsymbol{H}^* = p^{-1} (\boldsymbol{W}^*)^\top \boldsymbol{B}^* \in \mathbb{R}^{\overline{r}\times r}$ satisfies $\nu_{\min}(\boldsymbol{H}^*) \gg p^{-1/2}$. 
    \end{enumerate}
    Each column of $\boldsymbol{W}^*$ is referred to as a \emph{diversified projection weight}, and $\overline{r}$ denotes the number of projection directions.
\end{definition}
Given the source and target diversified projection matrices $\boldsymbol{W}^P$ and $\boldsymbol{W}^Q$, we define the $\bar r$-dimensional surrogate factor vectors $\widetilde{\boldsymbol{f}}^P$ and $\widetilde{\boldsymbol{f}}^Q$ as
\begin{equation}
    \label{eq:surrogateF}
    \widetilde{\boldsymbol{f}}^P = p^{-1}(\boldsymbol{W}^P)^\top \boldsymbol{x}^P, \quad    \widetilde{\boldsymbol{f}}^Q = p^{-1}(\boldsymbol{W}^Q)^\top \boldsymbol{x}^Q.
\end{equation}
The pre-defined \emph{diversified projection matrix} 
$\boldsymbol{W}$ is introduced to solve the issue that the number of factors, 
${r}$, may not be accurately determined, and 
methods incorporating the \emph{diversified projection 
matrix} $\boldsymbol{W}$ are robust to overestimation 
of the factor number $r$. By substituting the factor model \eqref{eq:factorModel2} into the surrogate factor estimation \eqref{eq:surrogateF}, we have 
\begin{align*}
    \widetilde{\boldsymbol{f}}^* &= \underbrace{\left(p^{-1}(\boldsymbol{W}^*)^\top \boldsymbol{B}^*\right)}_{\text{affine transformation}} \boldsymbol{f}^* + \underbrace{p^{-1}(\boldsymbol{W}^*)^\top \boldsymbol{u}^*}_{\text{estimation error}}.
\end{align*}
The significance assumption ensures that the signal from the latent factors (the affine transformation) remains dominant, while the weak dependence of the idiosyncratic components $\boldsymbol{u}^*$ ensures that the estimation error vanishes at rate $O_P(p^{-1/2}) $ as $p \to \infty$.

\subsection{Transfer factor estimation for covariate shift}
\label{sec:factor-transfer}

We leverage the source data to improve the quality of the surrogate factors $\widetilde{\boldsymbol{f}}^Q$. This is particularly beneficial when the target sample size $n^Q$ is small, leading to noisy estimates of the target covariance structure. When the source sample size $n^P$ is large, the pooled covariance $\widehat{\boldsymbol{\Sigma}}^A$ provides a more stable estimate due to the larger effective sample size.

We first define the sample covariance matrices for the source, target, and pooled data:
\begin{align*}
    \begin{cases}
        \widehat{\boldsymbol{\Sigma}}^P = \frac{1}{n^P}\sum_{i = 1}^{n^P} \boldsymbol{x}^P_i(\boldsymbol{x}^P_i)^\top ,\\
        \widehat{\boldsymbol{\Sigma}}^Q = \frac{1}{n^Q}\sum_{j = 1}^{n^Q} \boldsymbol{x}^Q_j(\boldsymbol{x}^Q_j)^\top ,\\
        \widehat{\boldsymbol{\Sigma}}^A 
        %= \frac{1}{n^P + n^Q} \Big( \sum_{i = 1}^{n^P} \boldsymbol{x}^P_i(\boldsymbol{x}^P_i)^\top + \sum_{j = 1}^{n^Q} \boldsymbol{x}^Q_j(\boldsymbol{x}^Q_j)^\top \Big) 
        = \frac{n^P}{n^P + n^Q} \widehat{\boldsymbol{\Sigma}}^P + \frac{n^Q}{n^P + n^Q}\widehat{\boldsymbol{\Sigma}}^Q.
    \end{cases}
\end{align*}
For any $* \in \{P, Q, A\}$, let $\widehat{\boldsymbol{v}}^*_1, \dots, \widehat{\boldsymbol{v}}^*_{\overline{r}} \in \R^p$ be the top-${\overline{r}}$ eigenvectors of the sample covariance matrix $\widehat{\boldsymbol{\Sigma}}^*$, and define $\widehat{\boldsymbol{W}}^* = p^{1/2}[\widehat{\boldsymbol{v}}^*_1, \dots, \widehat{\boldsymbol{v}}^*_{\overline{r}}]$ as the corresponding diversified projection matrix. In the absence of source data, a common approach is to rely exclusively on the target domain's top eigenvectors, $\widehat{\boldsymbol{W}}^Q$, to estimate the latent factors. %This construction is motivated by the spectral analysis techniques discussed in \cite{van2020survey}, with further implementation details found in \cite{fan2024factor}.
However, when source data is available, we can incorporate additional information through a model-selection process for covariance estimation.  The rationale is that when the target and source domains share a similar latent structure, the pooled covariance $\widehat{\boldsymbol{\Sigma}}^A$ provides a more stable estimate than $\widehat{\boldsymbol{\Sigma}}^Q$ alone.
%, particularly in the high-dimensional regime where $n^Q$ may be small. %In our process, we use the conventional covariance estimate $\widehat{\boldsymbol{\Sigma}}^Q$ if the distance $\norm{\widehat{\boldsymbol{\Sigma}}^Q - \widehat{\boldsymbol{\Sigma}}^A}_F$ is large, indicating a significant discrepancy between the domains. Conversely, if the distance is small, we assume the source data provides valuable information and utilize $\widehat{\boldsymbol{\Sigma}}^A$ instead. 

Specifically, we extract the factors by bounding the difference between the target and aggregation covariance with a pre-determined threshold $\delta$: %\r{Q: change $TL$ to $\TL$ + $FT$  to $\FT$}
\begin{align*}
    \widehat{\boldsymbol{\Sigma}}^{\TL}=
        \begin{cases}
            \widehat{\boldsymbol{\Sigma}}^A, \quad p^{-1} \norm{\widehat{\boldsymbol{\Sigma}}^Q - \widehat{\boldsymbol{\Sigma}}^A}_F \leq \delta, \\
            \widehat{\boldsymbol{\Sigma}}^Q, \quad \text{ otherwise.}
        \end{cases}
\end{align*}
This strategy follows the ``Transfer Around Boundary (TAB)'' philosophy \citep{fan2025robust}, where source information is utilized only when there is high confidence that the domains are sufficiently close, effectively protecting against biased estimation when domain shift is significant. Letting $\widehat{\boldsymbol{v}}^{\TL}_1, \dots, \widehat{\boldsymbol{v}}^{\TL}_{\overline{r}}$ be the top-$\overline{r}$ eigenvectors of $\widehat{\boldsymbol{\Sigma}}^{\TL}$, we define the model-selected projection matrix as $\widehat{\boldsymbol{W}}^{\TL} = p^{1/2}[\widehat{\boldsymbol{v}}^{\TL}_1, \dots, \widehat{\boldsymbol{v}}^{\TL}_{\overline{r}}]$
%This construction follows the Diversified Projection Matrix (DPM) design. We remark that an alternative DPM construction is proposed in \citet{fan2025factor}, which effectively removes the incoherence condition required for spectral-based estimates. Given their conceptual similarity, we adopt the current spectral-based formulation for its simplicity and direct applicability to the model-selection procedure in \eqref{eq:WTL}. 
and the surrogate factor $\widetilde{\boldsymbol{f}}^Q$  as 
\begin{equation}
    \label{eq:WTL}
    \widetilde{\boldsymbol{f}}^Q = p^{-1}(\widehat{\boldsymbol{W}}^{\TL})^\top \boldsymbol{x}^Q
=
        \begin{cases}
            p^{-1}(\widehat{\boldsymbol{W}}^A)^\top \boldsymbol{x}^Q, \quad p^{-1} \norm{\widehat{\boldsymbol{\Sigma}}^Q - \widehat{\boldsymbol{\Sigma}}^A}_F \leq \delta\\
            p^{-1}(\widehat{\boldsymbol{W}}^Q)^\top \boldsymbol{x}^Q, \quad \text{ otherwise.}\\
        \end{cases}
\end{equation}
The factors for the source domain $\widetilde{\boldsymbol{f}}^P$ can be defined similarly, with superscript $Q$ in \eqref{eq:WTL}replaced by $P$.
%\begin{align*}
%    \widetilde{\boldsymbol{f}}^P = 
%        \begin{cases}
%            p^{-1}(\widehat{\boldsymbol{W}}^A)^\top \boldsymbol{x}^P, \quad p^{-1} \norm{\widehat{\boldsymbol{\Sigma}}^P - \widehat{\boldsymbol{\Sigma}}^A}_F \leq \delta\\
%            p^{-1}(\widehat{\boldsymbol{W}}^P)^\top \boldsymbol{x}^P, \quad \text{ otherwise.}\\
%        \end{cases}
%\end{align*}
For simplicity in the subsequent nonparametric analysis, we assume $\boldsymbol{W}^P$ and $\boldsymbol{W}^Q$ are obtained through this procedure or provided by prior knowledge.

\subsection{Fine-tuning nonparametric variable selection for posterior shift}
\label{sec:fine-tuning}

We use SMART fine-tuning to address the posterior shift in transfer learning. The intuition behind this approach is straightforward: the complex source model $g^P$ can be more accurately learned using the source data due to its typically large sample size. By providing the resulting estimate $\widehat{g}^P$ as an additional input feature for estimating $g^Q$, we significantly reduce the functional complexity of the target learning task. 
%In this framework, the target network only needs to learn the simpler residual component $h$ from the limited target data.

To handle high-dimensionality, we employ a sparse selection mechanism using a continuous relaxation of the $L_0$ penalty. Define the clipped-$L_1$ function with threshold $\tau > 0$ as
\begin{align*}
    \psi_\tau(x) = \frac{|x|}{\tau} \land 1.
\end{align*}
This function serves as a continuous approximation of the indicator function $\mathbbm{1}\{x \neq 0\}$. The parameter $\tau$ is typically chosen to be very small to more precisely approximate the selection indicator. This clip function was first proposed by \cite{fan2001variable} and has been  adopted  in 
%widely adopted in high-dimensional linear regression \cite{zhang2010analysis}, 
deep learning with sparse weights \cite{ohn2022nonconvex}, and the FAST-NN model \cite{fan2024factor}.
%, factor analysis \cite{fan2024factor}, and causal effect estimation in high-dimensional nonparametric models \cite{fan2025factor}.

Specifically, our estimator for the source regression function is obtained by training a deep ReLU network $\pazocal{G}^P = \pazocal{G}(L^P, r + N^P, N^P, M, T^P)$:
\begin{equation}
    \label{eq:FTFAST1}
    (\widehat{g}^P, \widehat{\boldsymbol{\Theta}}^P) = \argmin_{g \in \pazocal{G}^P, \boldsymbol{\Theta}^P \in \mathbb{R}^{p \times N^P}} \frac{1}{n^P} \sum_{i = 1}^{n^P} \Big(y^P_i - g\big( \big[ \widetilde{\boldsymbol{f}}^P_i, \mathrm{trun}_M ((\boldsymbol{\Theta}^P)^\top \boldsymbol{x}^P_i) \big] \big)\Big)^2 + \lambda^P \sum_{i,j} \psi_\tau(\boldsymbol{\Theta}^P_{i,j}).
 \end{equation}
Following \cite{fan2024factor}, we estimate the idiosyncratic throughput $\boldsymbol{u}_i^P$ by using sparse linear combinations $(\boldsymbol{\Theta}^P)^\top \boldsymbol{x}^P_i$ to select a subset of relevant features. This results in  the pre-trained FAST-NN model for the source data:
\begin{align}
    \label{eq:FTFAST2}
    \widehat{s}(\boldsymbol{x}) = \widehat{g}^P\Big(\Big[p^{-1}(\boldsymbol{W}^P)^\top \boldsymbol{x}, \mathrm{trun}_M \Big( ( \widehat{\boldsymbol{\Theta}}^P )^\top \boldsymbol{x} \Big) \Big]\Big),
\end{align}
% Note that even though we assume the source model is $\SourceFunction$, we learn it through FAST-NN. A similar lifting idea appears in high-dimensional sparse linear regression \citep{fan2020factor}, which uses the weakly correlated latent factors and idiosyncratic components $[\boldsymbol{f}^P, \boldsymbol{u}^P]$ rather than the highly correlated original variable $\boldsymbol x^P$ as input variables.

The pre-trained function from the source data can now be transferred to the target data:
\begin{equation}
    \label{eq:FTFAST3}
    \widehat{s}_j = \widehat{s}(\boldsymbol{x}^Q_j), \quad \forall j \in [n^Q].
\end{equation}
Upon acquiring these signals, we incorporate them as an additional input feature for fitting the target regression function,  using a deep ReLU network $\pazocal{G} = \pazocal{G}(L, r + N + 1, N, M, T)$:
\begin{equation}
    \label{eq:FTFAST4}
    (\widehat{h}, \widehat{\boldsymbol{\Theta}}) = \argmin_{h \in \pazocal{G}, \boldsymbol{\Theta} \in \mathbb{R}^{p \times N}}\frac{1}{n^Q} \sum_{j = 1}^{n^Q} \Big( y_j^Q -h\big( \big[ \widetilde{\boldsymbol{f}}^Q_j, \mathrm{trun}_M(\boldsymbol{\Theta}^\top \boldsymbol{x}^Q_j), \widehat{s}_j \big] \big) \Big)^2 + \lambda \sum_{i,j} \psi_\tau(\boldsymbol{\Theta}_{i,j}),
\end{equation}
where $\{M, \tau, L, L^P, N, N^P, T, T^P, \lambda, \lambda^P\}$ are tuning hyper-parameters. The final fine-tuned factor-augmented estimator is:
\begin{equation}
    \label{eq:FTFAST5}
    \widehat{m}_{\FT}(\boldsymbol{x}) = \widehat{h}\Big( \big[p^{-1}(\widehat{\boldsymbol{W}}^{\TL})^\top \boldsymbol{x}, \mathrm{trun}_M ( \widehat{\boldsymbol{\Theta}}^\top \boldsymbol{x}), \widehat{s}(\boldsymbol{x}) \big]\Big).
\end{equation}

By including $\widehat{s}(\boldsymbol{x})$ as an input, the target network $\widehat{h}$ only needs to learn the residual relationship, which is typically simpler and more amenable to estimation with a smaller sample size $n^Q$. In particular, when there is no idiosyncratic component in fine-tuning, we can set $\boldsymbol{\Theta}=0$ in \eqref{eq:FTFAST4}, so the procedure reduces to a neural network regression with $(r+1)$-dimensional inputs $\big[ \widetilde{\boldsymbol{f}}^Q_j, \widehat{s}_j \big]$.  A further simplification is that $h(\cdot)$ depends linearly on $\big[ \widetilde{\boldsymbol{f}}^Q_j, \widehat{s}_j \big]$ or even just univariate $ \widehat{s}_j $.  This is particularly suitable when $n^Q$ is very small.

Note that \eqref{eq:FTFAST1} and \eqref{eq:FTFAST4} both involve the Factor Augmented Neural Lasso. We refer to the estimator \eqref{eq:FTFAST5} as SMART-FAN-Lasso to distinguish it from the FAST-NN estimator \eqref{eq:FTFAST1}.

% !TEX root = ../main.tex
\section{Theory for factor transfer}
\label{sec:factorTransfer}

In this section, we present the theoretical results for the methods in Section \ref{sec:method}. The following assumption is essential for transferring useful source information for factor models. It controls the discrepancy between the factor loading matrices $\boldsymbol{B}^P$ and $\boldsymbol{B}^Q$, thereby ensuring the similarity between the generating distributions of the covariates $\boldsymbol{x}^P$ and $\boldsymbol{x}^Q$.

\begin{assumption}[Factor loading similarity] 
\label{ass:factorLoadingDiff}
We assume the factor decomposition \eqref{eq:factorModel2} satisfies
$p^{-1} \norm{\boldsymbol{B}^P (\boldsymbol{B}^P)^\top -\boldsymbol{B}^Q (\boldsymbol{B}^Q)^\top }_F \leq \varepsilon$
for some constant $\varepsilon>0$.
\end{assumption}

Next, we impose several regularity conditions on the covariate and factor model settings.

\begin{assumption}[Boundedness]
\label{ass:boundedness}
There exist universal constants $c_1$ and $b$ such that
\begin{enumerate} 
    \item $\norm{\boldsymbol{B}^P}_{max} \lor \norm{\boldsymbol{B}^Q}_{max} \leq c_1$;
    \item $\EE[\boldsymbol{f}^*] = 0, \ \EE[\boldsymbol{f}^* (\boldsymbol{f}^*)^\top] = \boldsymbol{I}_r,\ \mathrm{supp}(\boldsymbol{f}^*) \subset [-b, b]^r$ for any $* \in \{P, Q\}$;
    \item $\EE[\boldsymbol{u}^*] = 0, \ \mathrm{supp}(\boldsymbol{u}^*) \subset [-b, b]^p$ for any $* \in \{P, Q\}$.
\end{enumerate}
\end{assumption}

\begin{assumption}[Weak dependence] 
\label{ass:weakDependence}
There exists some universal constant $c_1$ such that for any $* \in \{P, Q\},$
\begin{enumerate}
    \item (Between entries of $\boldsymbol{u}^*$) $\sum_{1 \leq j < k \leq p} | \EE[\boldsymbol{u}^*_j \boldsymbol{u}^*_k]| \leq c_1 \cdot p$;
    \item (Between $\boldsymbol{f}^*$ and $\boldsymbol{u}^*$) $\norm{\boldsymbol{B}^* \EE[\boldsymbol{f}^* (\boldsymbol{u}^*)^\top ]}_F \leq c_1 \sqrt{p}$.
\end{enumerate}
\end{assumption}

\begin{assumption}[Pervasiveness] 
\label{ass:pervasiveness}
There exists some universal constant $c_1 \geq 1$ such that
\begin{align*}
    \lambda_{\mathrm{min}}((\boldsymbol{B}^P)^\top \boldsymbol{B}^P) \land \lambda_{\mathrm{min}}((\boldsymbol{B}^Q)^\top \boldsymbol{B}^Q) \geq p / c_1
\\
    \lambda_{\mathrm{max}}((\boldsymbol{B}^P)^\top \boldsymbol{B}^P) \lor \lambda_{\mathrm{max}}((\boldsymbol{B}^Q)^\top \boldsymbol{B}^Q) \leq c_1 \cdot p.
\end{align*}
\end{assumption}

Assumptions \ref{ass:boundedness}–\ref{ass:pervasiveness} are standard in factor models. We replace the sub-Gaussian condition for $\boldsymbol{f}^P, \boldsymbol{f}^Q, \boldsymbol{u}^P$, and $\boldsymbol{u}^Q$ with uniform boundedness in Assumption \ref{ass:boundedness}, purely to facilitate the technical arguments.  Assumptions \ref{ass:weakDependence} and \ref{ass:pervasiveness} on the distribution of $(\boldsymbol{f}^*, \boldsymbol{u}^*)$ for any $* \in \{P, Q\}$ lay the foundation of the validity of the chosen diversified projection matrices, ensuring that the condition $\nu_{\min}(\boldsymbol{H}^Q) \asymp 1$ holds, as shown in the following theorem.

\begin{theorem}[Factor transfer]
\label{thm:factorTransfer}
Under Assumptions \ref{ass:factorLoadingDiff}-\ref{ass:pervasiveness}, the matrix $\widehat{\boldsymbol{W}}^{\TL}$ proposed in \eqref{eq:WTL} satisfies $
    \nu_{max} \left(p^{-1} (\widehat{\boldsymbol{W}}^{\TL})^\top \boldsymbol{B}^Q \right) \leq c_3$.
Moreover, for any $t>0$,
 such that $\delta\geq c_3 \varepsilon^Q(t)$, we have, with probability at least $1 - 6e^{-t}$ that $\nu_{min}\left(p^{-1} (\widehat{\boldsymbol{W}}^{\TL})^\top \boldsymbol{B}^Q \right) \geq c_1 - c_2 \delta \land \varepsilon^A(t)$. 
for some universal constants $c_1$-$c_3$ that are independent of $n^P, n^Q, p, t, r, \overline{r}$, where $\varepsilon^Q(t) = r \sqrt{\frac{\log p + t}{n^Q}} + r^2 \sqrt{\frac{\log r + t}{n^Q}} + \frac{1}{\sqrt{p}}$ and $\varepsilon^A(t) = r \sqrt{\frac{\log p + t}{n^P + n^Q}} + r^2 \sqrt{\frac{\log r + t}{n^P + n^Q}} + \frac{1}{\sqrt{p}} + \varepsilon$.
\end{theorem}

Theorem \ref{thm:factorTransfer} indicates that the source data can enhance transfer learning of the target latent factor loadings by enabling the construction of a more statistically robust diversified projection matrix, $\widehat{\boldsymbol{W}}^{\TL}$.
The choice of $t$ can be of order $\log n^Q$. The process of learning the source diversified projection matrix $\boldsymbol{W}^P$ is similar.
A precise choice of $\delta$ is unnecessary due to only requiring the right-hand side of $\nu_{min}\left(p^{-1} (\widehat{\boldsymbol{W}}^{\TL})^\top \boldsymbol{B}^Q \right)$ to be lower bounded. 

In the high-dimensional regime with $p \gg 1$, and without source data, we can construct surrogates of the latent factors $\widetilde{\boldsymbol{f}}^Q$ when $n^Q \gg (\log p + t)$. However, even when $n^Q$ is small, a robust estimate can still be achieved if $n^P \gg (\log p + t)$ and the factor loading difference $\epsilon$ is sufficiently small. Thus, a small subset of the target and source data can be used to learn the diversified weights with $\nu_\min(\boldsymbol{H}^Q) \asymp 1$.
We provide the perturbation results regarding the covariance matrix estimate $\widehat{\boldsymbol{\Sigma}}^{\TL}$ and the eigenspaces.

\begin{proposition}[Perturbation bounds]
\label{prop:pertubation}
Under the assumptions and notations of Theorem \ref{thm:factorTransfer}, for some universal constant $c_4$ we have

\begin{enumerate}
    \item  $\norm{\widehat{\boldsymbol{\Sigma}}^{\TL} - \boldsymbol{B}^Q (\boldsymbol{B}^Q)^\top}_F \lesssim p \Big( \delta \land \varepsilon^A(t) \Big)$;
    \item  $\underset{\boldsymbol{R} \in \pazocal{O}(r)}{\min} \norm{\widehat{\boldsymbol{V}}_r^{\TL} \boldsymbol{R} - \boldsymbol{V}^Q}_F \lesssim \delta \land \varepsilon^A(t)$.
\end{enumerate}
with probability at least $1 - 6e^{-t}$ given $t > 0$ such that $\delta \geq c_4 \varepsilon^Q(t)$. Here, $\pazocal{O}(r)$ denotes the space of orthogonal matrices of size $r \times r$, and $\widehat{\boldsymbol{V}}_r^{\TL}$ and $\boldsymbol{V}^Q$ represent the top-$r$ eigenvectors of $\widehat{\boldsymbol{\Sigma}}^{\TL}$ and $\boldsymbol{B}^Q (\boldsymbol{B}^Q)^\top$, respectively.
\end{proposition}

%%%%%%%%%%%%%%%%%%%%%%%%%%%%%%%%%%%%%%%%%%%%%%%%%%%%%%%%%%%%%%%%%%%%%%%%
\section{Theory for regression fine-tuning}
\label{sec:fineTuning}

In this section, we aim to establish high probability upper bounds for both the out-of-sample excess risk $\pazocal{E}^Q(\widehat{m}_{\FT})$ and the in-sample mean squared error  $\widehat{\pazocal{E}}^Q (\widehat{m}_{\FT})$ for SMART-FAN-Lasso, where $\widehat{\pazocal{E}}^Q(\widehat{m}_{\FT}) \overset{def}{=} \frac{1}{n^Q} \sum_{i =1 }^{n^Q} |\widehat{m}_{\FT}(\boldsymbol{x}_i^Q) - g^Q(\boldsymbol{f}_i^Q, \boldsymbol{u}^Q_{i, \pazocal{J} \cup \pazocal{J}^P})|^2$.
For justifications on the importance of deriving high-probability error bounds for both in-sample and out-of-sample \(l_2\) errors, see \cite{farrell2018DeepNeural}. 

\subsection{Oracle-type upper bound}
Before presenting the results about how the fine-tuning factor augmented estimator in \eqref{eq:FTFAST5} improves the risk convergence rate, we impose two additional assumptions as follows.

\begin{assumption}
\label{ass:subGaussianNoiseAndBoundedFunction}
The model inputs satisfy that
\begin{enumerate}
    \item (Sub-Gaussian noise) There exists a universal constant $c_1$ such that 
$\mathbb{P}(|\epsilon^*| \geq t \, | \, \boldsymbol{x}^*) \leq 2\exp(-c_1 t^2)$ for all $t > 0$ and $* \in \{P, Q\}$ almost surely.
    \item (Bounded Regression function) We have $\max\{\norm{g^P}_\infty,\norm{h}_\infty\}\leq M^*$ for some universal constant $M^*.$ The choice of $M$ satisfies that $1\leq M^*\leq M\leq c_2M^*$ for some universal constant $c_2>1$. Additionally, $g^P$ and $h$ are all $c_1$-Lipschitz for some constant $c_1.$ 
\end{enumerate}
\end{assumption}

\begin{assumption}[Absolute continuity]
\label{ass:absCon}
The Radon-Nikodym derivative satisfies $\frac{d\mu_{\boldsymbol{x}^Q}}{d\mu_{\boldsymbol{x}^P}}(\boldsymbol{x}) \leq c_1$ almost surely for some universal constant $c_1$.
\end{assumption}
Assumption \ref{ass:subGaussianNoiseAndBoundedFunction} is standard in the nonparametric regression literature and Assumption \ref{ass:absCon} ensures that any good source population estimator \(\widehat{s}\) also behaves well on the target domain:
$\mathbb{E}\Big[ |\widehat{s}(\boldsymbol{x}^Q) - s^P(\boldsymbol{x}^Q)|^2 \Big] \lesssim \mathbb{E}\Big[ |\widehat{s}(\boldsymbol{x}^P) - s^P(\boldsymbol{x}^P)|^2 \Big]$.

The intuition behind this is that the source marginal distribution for \(X^P\) should provide full coverage over the target marginal distribution for \(X^Q\), allowing the accuracy of non-parametric regression to transfer effectively. An assumption similar to the second statement is made in \cite{fan2025robust}.

The following theorem presents an oracle-type inequality for the error bound.  Define
\begin{itemize}
	\item neural network approximation errors: 
	\begin{align*}
		\delta_a^h &= \inf_{g \in \pazocal{G}(L - 1, r + |\pazocal{J}| + 1, N, M, T)} \, \sup_{\kappa \in[-M, M]} \norm{g(\boldsymbol{f}^Q, \boldsymbol{u}_{\pazocal{J}}^Q, \kappa) -h(\boldsymbol{f}^Q, \boldsymbol{u}_{\pazocal{J}}^Q, \kappa)}_\infty^2\\
		\delta^P_a &= \inf_{g \in \pazocal{G}(L^P - 1, r^P + |\pazocal{J}^P|, N^P, M^P, T^P)} \norm{g(\boldsymbol{f}^P, \boldsymbol{u}_{\pazocal{J}^P}^P) - g^P(\boldsymbol{x}_{\pazocal{J}^P}^P)}_\infty^2\\
		\delta^Q_a &= \inf_{g \in \pazocal{G}(L - 1, r + |\pazocal{J} \cup \pazocal{J}^P|, N, M, T)} \norm{g(\boldsymbol{f}^Q, \boldsymbol{u}_{\pazocal{J} \cup \pazocal{J}^P}^Q) - g^Q(\boldsymbol{f}^Q, \boldsymbol{u}^Q_{\pazocal{J} \cup \pazocal{J}^P})}_\infty^2
	\end{align*}
	\item stochastic errors:
	\begin{align*}
		\delta_s^h &= \frac{(N^2 L + N \overline{r}) L \log(T N n^Q)}{n^Q} + \lambda|\pazocal{J}|\\
		\delta_s^{P} &= \frac{\Big((N^P)^2 L^P + N^P \overline{r} \Big) L^P \log(T N^P n^P)}{n^P} + \lambda^P|\pazocal{J}^P|\\
		\delta_s^Q &= \frac{(N^2 L + N \overline{r}) L \log(T N n^Q)}{n^Q} + \lambda|\pazocal{J} \cup \pazocal{J}^P|
	\end{align*}
	\item latent factor estimation errors: 
	\begin{align*}
		\delta^h_f = \frac{|\pazocal{J} | r \cdot \overline{r}}{p \cdot \nu^2_{min}(\boldsymbol{H}^Q)}, \ \delta_f^P = \frac{|\pazocal{J}^P| r \cdot \overline{r}}{p \cdot \nu^2_{min} (\boldsymbol{H}^P)}, \ \delta^Q_f = \frac{|\pazocal{J} \cup \pazocal{J}^P| r \cdot \overline{r}}{p \cdot \nu^2_{min}(\boldsymbol{H}^Q)}
	\end{align*}
\end{itemize}

\begin{theorem}[Oracle-type bound for SMART-FAN-Lasso]
\label{thm:fineTuning}
Suppose that Assumptions \ref{ass:boundedness}, \ref{ass:weakDependence}, \ref{ass:subGaussianNoiseAndBoundedFunction}, and \ref{ass:absCon} hold. Consider the SMART-FAN-LASSO defined via \eqref{eq:FTFAST1}-\eqref{eq:FTFAST5}, with  hyperparameters satisfying 
\begin{itemize}
    \item $N^P \geq 2(r + |\pazocal{J}^P|),\ N \geq 2(r + |\pazocal{J}|)$
    \item $T^P \geq c_1[\nu_{min}(\boldsymbol{H}^P)]^{-1}| \pazocal{J}^P|r, \ T \geq c_1[\nu_{min}(\boldsymbol{H}^Q)]^{-1}|\pazocal{J}|(r + 1)$
    \item $\lambda^P \geq c_2 (n^P)^{-1} \Big(\log(p n^P (N^P + \overline{r})) + L^P \log(T^P N^P) \Big), \ \lambda \geq c_2 (n^Q)^{-1} \Big( \log(p n^Q(N + \overline{r})) + L \log(TN) \Big)$
    \item $\tau^{-1}\geq c_3 (r + 1) p \Big[\Big((T^P N^P)^{L^P + 1}(N^P + \overline{r}) n^P \Big) \lor \Big((TN)^{L + 1}(N + \overline{r}) n^Q \Big) \Big]$
\end{itemize}
for some universal constants $c_1$-$c_3$ and $\overline{r}\geq r$. 
Then, with probability at least $1 - 6e^{-t}$ with respect to the target and source data, for $n^Q$ (and thus $n^P$) large enough, we have
\begin{equation}
    \label{eq:fineTuningResult}
    \begin{split}
        \pazocal{E}^Q(\widehat{m}_{\FT}) + \widehat{\pazocal{E}}^Q (\widehat{m}_{\FT})
        &\leq c_4 \Big\{\Delta^Q \land \Big(\Delta^P+\frac{t}{n^P}\Big) + \Delta^h + \frac{t}{n^Q} \Big\}\\
        &\leq 2c_4 \Big\{\Delta^Q \land \Delta^P + \Delta^h + \frac{t}{n^P \land n^Q}\Big\}
    \end{split}
\end{equation}
where   $\Delta^* = \delta_a^* + \delta_f^* + \delta_s^*$
for any $* \in \{P, Q, h\}$, and $c_4$ is a universal constant that depends only on the constants in Assumptions \ref{ass:boundedness}, \ref{ass:weakDependence}, and \ref{ass:subGaussianNoiseAndBoundedFunction}.
\end{theorem}

Theorem \ref{thm:fineTuning} establishes a high-probability bound on both the out-of-sample squared error and the in-sample mean squared error. The error bound is composed of \(\Delta^Q, \Delta^P,\) and \(\Delta^h\), which represent the errors from estimating \(g^Q\), \(g^P\), and \(h\), respectively. Notably, the convergence rate takes the form \(\Delta^Q \land \Delta^P + \Delta^h\), indicating that if \(\Delta^P\) is sufficiently small because \(g^P\) is well estimated (possibly from a large \(n^P\)), the fine-tuning procedure can improve the convergence rate from \(\Delta^Q\) to \(\Delta^h\), leveraging the useful information in \(g^P\) to enhance the estimation of \(g^Q\). However, the term \(\Delta^h\) is unavoidable, as \(h\) still needs to be estimated from the target data.

Each aggregate error bound \(\Delta^*\) consists of three components: the neural network approximation error \(\delta_a^*\) to the underlying regression function, the stochastic error \(\delta_s^*\) due to empirical risk minimization, and the error \(\delta_f^*\) from inferring the latent factors from observed covariates.

\subsection{Optimal bounds for hierarchical composition models}
An explicit rate can be further achieved by carefully selecting the hyperparameters to balance the trade-off between the approximation error $\delta_a^*$ and the stochastic error $\delta_s^*$. This choice requires understanding the neural network's approximation capability as follows:
\begin{align*}
    \sup_{\substack{m^* \in \pazocal{H}(d, l, \pazocal{P}, C) \\ \gamma(\pazocal{P}) = \gamma^*, \norm{m^*}_\infty \leq M^*}} \inf_{g \in \pazocal{G} (d, c_1,c_2 N, \infty, c_3 N^{c_4})} \norm{g - m^*}^2_\infty \lesssim N^{-4 \gamma^*}
\end{align*}
for some universal constants $c_1$ - $c_4$, when $d$ is upper bounded by a universal constant. See Theorem 4 of \cite{fan2024factor} for a rigorous statement. The following condition lists our choice of hyperparameters.

\begin{cond}
\label{cond:hyper}
The deep ReLU neural network hyperparameters satisfy that
\begin{enumerate}
    \item $r \leq \overline{r} \lesssim r + 1$;
    \item $c_1 \{ \log n^P \lor \log[ \nu_\min(\boldsymbol{H}^P)]^{-1} \} \leq \log T^P \lesssim \log n^P, \ c_1 \{\log n^Q \lor \log[ \nu_\min (\boldsymbol{H}^Q)]^{-1} \} \leq \log T \lesssim \log n^Q$;
    \item $c_1 \leq L^P \lesssim 1,\ c_1(n^P / \log n^P)^{1 / (4 \gamma^P + 2)} \leq N^P \lesssim (n^P / \log n^P)^{1 / (4 \gamma^P + 2)}$;
    \item $c_1 \leq L \lesssim 1, \ c_1(n^Q / \log n^Q)^{1 / (4 \Gamma + 2)} \leq N \lesssim (n^Q / \log n^Q)^{1 / (4 \Gamma + 2)}$.
\end{enumerate}
where $\Gamma = \gamma \boldsymbol{1} \{ n^P \geq n^Q\} + (\gamma^P \land \gamma) \boldsymbol{1} \{n^P < n^Q\}$. Here, $c_1$ is some universal constant which only depends on $l^P, \pazocal{P}^P, l, \pazocal{P}$.
\end{cond}

The following theorem shows the optimal rate for SMART-FAN-Lasso.

\begin{theorem}[Optimal rate for SMART-FAN-Lasso]
\label{thm:fineTuning2}

Suppose that Assumptions \ref{ass:transferability}, \ref{ass:boundedness}, \ref{ass:weakDependence}, \ref{ass:subGaussianNoiseAndBoundedFunction} and \ref{ass:absCon} hold. Consider the SMART-FAN-Lasso \eqref{eq:FTFAST1}-\eqref{eq:FTFAST5} with $r + |\pazocal{J} \cup \pazocal{J}^P| \leq c_1$ for some universal constant $c_1$ that satisfies Condition \ref{cond:hyper}, and $c_2 \frac{\log (p n^P)}{n^P} \leq \lambda^P \leq c_3 \frac{\log (p n^P)}{n^P}, c_2 \frac{\log (p n^Q)}{n^Q} \leq \lambda \leq c_3 \frac{\log (p n^Q)}{n^Q}, \tau^{-1} \geq (n^P \lor n^Q)^{c_4} p$ 
for some universal constants $c_2$ - $c_4$ independent of $n^P, n^Q$, and $p$. Furthermore, suppose that $n^P \gtrsim (\log p)^{2 \gamma^P + 1}$. Then, for any universal constant $K \geq 1$, with probability at least $1 - 6p^{-K}$ with respect to the source and target data, for $n^P$ and $n^Q$ large enough, we have
\begin{align}
    \pazocal{E}^Q(\widehat{m}_{\FT}) + \widehat{\pazocal{E}}^Q (\widehat{m}_{\FT})
    \leq c_5 K  \Big\{ \nonumber \Big[ & \frac{\log (n^P + n^Q)}{n^P + n^Q} \Big]^{ \frac{2 \gamma^P}{2 \gamma^P + 1}} + \Big( \frac{ \log n^Q }{n^Q} \Big)^{\frac{2 \gamma}{2 \gamma + 1}} \\
    &+ \frac{\log p}{n^Q} \mathbbm{1}(|\pazocal{J}| \not = 0)+ \frac{1 \land r}{[ \nu^2_\min(\boldsymbol{H}^P) \land \nu^2_\min(\boldsymbol{H}^Q)]p} \Big\}, 
     \label{eq:fineTuningUpperBound}
\end{align}
where $c_5$ is a constant that depends only on the constants in Assumptions \ref{ass:boundedness}, \ref{ass:weakDependence}, \ref{ass:subGaussianNoiseAndBoundedFunction}, and Condition \ref{cond:hyper}.
\end{theorem}

Theorem \ref{thm:fineTuning2} indicates the asymptotic convergence rate for our proposed SMART-FAN-Lasso.
The third term is dominated by the second term under mild conditions. We now explain each term in  \eqref{eq:fineTuningUpperBound}.

\begin{itemize}
    \item The term $\Big[ \frac{ \log (n^P + n^Q)}{n^P + n^Q} \Big]^{\frac{2 \gamma^P}{2 \gamma^P + 1}}$ represents the risk of estimating the source regression function $g^P$, with complexity parameter $\gamma^P$. It benefits from both the target and source samples, which is a key advantage of our fine-tuning procedure. When $n^P \gg n^Q$, we benefit from fine-tuning while still retaining the conventional rate when $n^P$ is small.
    \item The term $\Big( \frac{ \log n^Q}{n^Q} \Big)^{\frac{2 \gamma}{2 \gamma + 1}}$ captures the risk of estimating the remaining component of $g^Q$, namely $h$, with complexity parameter $\gamma$. Estimation of $h$ is based only on target data.
    \item The term $\frac{\log p}{n^Q}$ reflects the uncertainty in variable selection for $\pazocal{J}$ in the target regression function. The uncertainty of variable selection for the source regression function is dominated by the main terms and is omitted due to the condition $n^P \gtrsim (\log p)^{2 \gamma^P + 1}$.
    \item The term $\frac{1 \land r}{[\nu^2_\min(\boldsymbol{H}^P) \land \nu^2_\min(\boldsymbol{H}^Q)] p}$ represents the risk associated with inferring latent factors from observed covariates in the target and source domains. When $r \neq 0$ and $\nu^2_\min(\boldsymbol{H}^*) \asymp 1$ for $* \in \{P, Q\}$, this $p^{-1}$ term is typically negligible under the high-dimensional regime.
\end{itemize}

Furthermore, Theorem \ref{thm:fineTuning2} clearly shows how our fine-tuning procedure improves the statistical convergence rate. Without transferring from the source data to the target domain, the optimal asymptotic risk rate we can obtain is given in \cite{fan2024factor} as
%\r{Q: with typos?} \textcolor{blue}{(A: By this sentence, I am going to give the risk rate of the transfer function h.)}
\begin{equation}
    \label{eq:fastUpperBound}
    \Big( \frac{ \log n^Q}{n^Q} \Big)^{\frac{ 2\gamma^P}{2 \gamma^P + 1}} + 
    \Big(\frac{\log n^Q}{n^Q} \Big)^{\frac{2 \gamma}{2 \gamma + 1}} +  \frac{\log p}{n^Q} + \frac{1 \land r}{ \nu^2_\min(\boldsymbol{H}^P) \land \nu^2_\min(\boldsymbol{H}^Q) p}.
\end{equation}
Comparing \eqref{eq:fineTuningUpperBound} with this formula, we observe that with the following conditions
\begin{equation}
    \label{eq:fasterCond}
    n^P \gg n^Q, \quad \gamma^P < \gamma,
\end{equation}
i.e., the source sample size is large enough, the source regression function \(g^P\) is indeed more complex than the remaining \(h\) component, we can accelerate the risk convergence via the fine-tuning estimator. On the other hand, even if the condition \eqref{eq:fasterCond} is not satisfied, we maintain the optimal convergence rate as in \eqref{eq:fastUpperBound}, without involving source data. In this sense, our estimator is both ``efficient" in improving the rate as in \eqref{eq:fineTuningUpperBound} and ``robust" in avoiding negative transfer.

\begin{remark}
Given $\pazocal{J} = \emptyset$, the term $\frac{\log p}{n^Q}$ in \eqref{eq:fineTuningUpperBound} is eliminated  in Theorem \ref{thm:fineTuning}. This case corresponds to a scenario where the idiosyncratic noise does not contribute to learning $h$, which is possible because the information of the idiosyncratic noise is fully captured through $g^P$. Additionally, by letting $\lambda \to \infty$ in \eqref{eq:FTFAST4}, \eqref{eq:FTFAST4} simplifies to $\widehat{h}(\cdot) = \argmin_{h \in \pazocal{G} (L, r + 1, N, M, T)} \frac{1}{n^Q} \sum_{i = 1}^{n^Q} \Big(y^Q_i - h( \Big[ \widetilde{\boldsymbol{f}}^Q_i, \widehat{s}_i \Big] ) \Big)^2$. 
This optimization corresponds to the augmented FAR-NN procedure mentioned in \cite{fan2024factor}, with an added entry $\widehat{s}_i$. Following the proof structure of Theorem \ref{thm:fineTuning2} and treating $\lambda \to \infty$ with $\pazocal{J} = \emptyset$, we reach the upper bound:
$\Big[ \frac{ \log (n^P + n^Q)}{n^P + n^Q} \Big]^{\frac{2 \gamma^P}{2 \gamma^P + 1}} + \Big( \frac{\log n^Q}{n^Q} \Big)^{\frac{2 \gamma}{2 \gamma + 1}} + \frac{1 \land r}{[\nu^2_\min(\boldsymbol{H}^P) \land \nu^2_\min(\boldsymbol{H}^Q)]p}$. 
The formula above mitigates the curse of high dimensionality and still converges as \(n_Q\) grows to infinity, even if \(n_Q\) is not large compared with \(p\).
\end{remark}

\subsection{Minimax optimal lower bound}
Define a family of distributions of $(\boldsymbol{f}, \boldsymbol{u}, \boldsymbol{x})$ as
\begin{align}
    \label{eq:lb-factor-dist}
    \begin{split}
        \pazocal{P}(p, r, \rho) = \Big\{ \mu (\boldsymbol{f}, \boldsymbol{u}, \boldsymbol{x}): \quad &\mathrm{supp}(\boldsymbol{f}) \subset [-1, 1]^r, \, \mathrm{supp}(\boldsymbol{u}) \subset [-1, 1]^p, \, \mathbb{E}[\boldsymbol{f}] = \mathbb{E}[\boldsymbol{u}] = 0, \\
        &\boldsymbol{f} \text{ and } \boldsymbol{u} \text{ independent, both have independent components,} \\
        & \boldsymbol{x} = \boldsymbol{B} \boldsymbol{f} + \boldsymbol{u} \text{ with } \norm{\boldsymbol{B}}_{\max} \leq 1 \text{ and } \lambda_{\min}(\boldsymbol{B}^\top \boldsymbol{B}) \geq \rho \Big\}.
    \end{split}
\end{align} 
Besides the optimal upper bound obtained in Theorem \ref{thm:fineTuning2}, we also provide the following minimax lower bound for $\pazocal{E}^Q(\widehat{m})$, where $\widehat{m}$ is any general estimator depending on the target and source data.

\begin{theorem}[Minimax optimal lower bound]
\label{thm:minimaxLower}
Consider the i.i.d. samples $\{\{(\boldsymbol{x}^P_i,  \boldsymbol{y}^P_i)\}_{i \in [n^P]} \cup \{(\boldsymbol{x}^Q_j, \boldsymbol{y}^Q_j) \}_{j \in [n^Q]}\}$ from the model \eqref{eq:setting} with $ \{ \epsilon^P_i \}_{i \in [n^P]} \cup \{ \epsilon^Q_j \}_{j \in [n^Q]} \overset{i.i.d.}{\sim} \pazocal{N}(0, 1)$. Suppose that there exists some $(\beta^P, d^P) \in \pazocal{P}$ such that $\gamma^P = \gamma(\pazocal{P}^P) = \frac{\beta^P}{d^P}$ and $d^P \leq r + 1 \lesssim 1$ and that there exists some $(\beta^h, d^h) \in \pazocal{P}$ such that $\gamma = \gamma(\pazocal{P}) = \frac{\beta^h}{d^h}$ and $d^h \leq r + 1 \lesssim 1$.  Then, for $n^P$ and $n^Q$ large enough, we have
\begin{equation}
    \label{eq:fineTuningLowerBound}
    \inf_{\widehat{m}} \sup_{ \substack{ \mu_{(\boldsymbol{f}^Q, \boldsymbol{u}^Q, \boldsymbol{x}^Q)} = \mu_{(\boldsymbol{f}^P, \boldsymbol{u}^P, \boldsymbol{x}^P)} \in \pazocal{P} (p, r, \rho)\\
    \text{Assumption \ref{ass:transferability} holds with $|\pazocal{J}| = 1$} }} \pazocal{E}^Q (\widehat{m}) \geq c_2 \left\{  (n^P + n^Q)^{-\frac{2 \gamma^P}{2 \gamma^P + 1}} + (n^Q)^{-\frac{2 \gamma}{2 \gamma + 1}} + \frac{\log p}{n^Q} + \frac{1}{\rho} \right\}    
\end{equation}
for some universal constant $c_2$ independent of $n^P, n^Q$ and $\rho$, and the infimum is taken over all estimators based on the observations $\{ (\boldsymbol{x}^P_i, \boldsymbol{y}^P_i)\}_{i \in [n^P]} \cup \{ (\boldsymbol{x}^Q_j, \boldsymbol{y}^Q_j)\}_{j \in [n^Q]}$.
\end{theorem}

The fourth term $\frac{1}{\rho}$ in the lower bound \eqref{eq:fineTuningLowerBound} reduces to $\frac{1}{p}$ under the pervasiveness assumption (Assumption \ref{ass:pervasiveness}) where $\lambda_\min((\boldsymbol{B}^P)^\top \boldsymbol{B}^P) \asymp \lambda_\min((\boldsymbol{B}^Q)^\top \boldsymbol{B}^Q) \asymp p$. 
Thus, it matches the term $\frac{1 \land r}{[\nu^2_\min(\boldsymbol{H}^P) \land \nu^2_\min(\boldsymbol{H}^Q)]p}$ in \eqref{eq:fineTuningUpperBound} when $r \neq 0$ and $\nu^2_\min(\boldsymbol{H}^*) \asymp 1$ for $* \in \{P, Q\}$. Even when $r = 0$, the term $\frac{1}{p}$ is typically dominated by other main terms in \eqref{eq:fineTuningLowerBound} under the high-dimensional regime. Therefore, the upper bound we obtained in Theorem \ref{thm:fineTuning2} matches the optimal lower bound up to logarithmic factors in $n^P$ and $n^Q$.

% !TEX root = ../main.tex
\section{Simulation Studies}
\label{sec:sim}

\subsection{Covariate Shift Estimation}
\label{sec:sim-covariate}

\textbf{Candidate methods.}
We compare the reliability of our proposed transfer factor estimator $\widetilde{\boldsymbol{f}}^Q = p^{-1}(\widehat{\boldsymbol{W}}^{\TL})^T \boldsymbol{x}^Q$ with two other simple approaches under covariate shift:
$\widetilde{\boldsymbol{f}}^Q = p^{-1}(\widehat{\boldsymbol{W}}^{Q})^T \boldsymbol{x}^Q$ ({Vanilla-target factor}),
$\widetilde{\boldsymbol{f}}^Q = p^{-1}(\widehat{\boldsymbol{W}}^{P})^T \boldsymbol{x}^Q$  ({Vanilla-source factor}).
As the goal is to learn latent factors in the target domain via a diversified projection, a natural measure is the smallest singular value $\nu_{\mathrm{min}}(p^{-1} (\widehat{\boldsymbol{W}})^T \boldsymbol{B}^Q)$ for a given  diversified projection $\widehat{\boldsymbol{W}}$. 

\textbf{Data generating process.} The source and target covariates both follow a factor model with $r = 4$ factors. For the source data, we generate $b_{ij}^P \overset{\mathrm{i.i.d}}{\sim} \mathrm{Unif}(-\sqrt{3}, \sqrt{3})$ to form the loading matrix $\boldsymbol{B}^P = (b_{i,j}^P) \in \mathbb{R}^{p \times r}$, and draw $f_i^P, u_i^P \overset{\mathrm{i.i.d}}{\sim} \mathrm{Unif}(-1, 1)$ to form the latent factor vector $\boldsymbol{f}^P = (f_i^P) \in \mathbb{R}^{r}$ and the idiosyncratic component $\boldsymbol{u}^P = (u_i^P) \in \mathbb{R}^{p}$. For the target data, we shift the loading matrix to $b_{ij}^Q = b_{ij}^P + \epsilon_{ij}^{B}$ with $\epsilon_{ij}^{B} \sim 0.5 \times \mathrm{Rademacher}$ and again generate $f_i^Q, u_i^Q \overset{\mathrm{i.i.d}}{\sim} \mathrm{Unif}(-1, 1)$.

We evaluate performance at a fixed covariate dimension $p = 1000$. The source sample size is set to $n^P = \{100, 150, 200, 250, 300\}$, and we consider two target sample sizes, $n^Q \in \{7, 10\}$ with $n^{Q}= 10$ for illustrating the performance without transfer.  
We set $\delta = r \sqrt{\frac{\mathrm{log}(p)}{n^P + n^Q}} + r^2 \sqrt{ \frac{\mathrm{log}(r)}{n^P + n^Q}} + \frac{1}{\sqrt{p}}$ and repeat simulation $100$ times for each setting.

\textbf{Results.}
Figure~\ref{fig:sim-covariate} reports $\nu_{\mathrm{min}}(p^{-1} (\widehat{\boldsymbol{W}}^*)^T \boldsymbol{B}^Q)$ across different source sample sizes $n^P$  for $n^Q = 7$, where larger values indicate more stable recovery of the target loading space.  Vanilla-target factor estimation with $n^Q = 10$ is included for comparison.   Our transfer factor estimator attains the largest and most concentrated values. The main findings are as follows:

\begin{itemize}
    \item \textbf{Numerical Stability:} Our proposed $\widehat{\boldsymbol{W}}^{\TL}$ remains numerically stable and robust by adaptively borrowing strength from the source domain, whereas the target-only approach ($\widehat{\boldsymbol{W}}^Q$) is less effective when $n^Q = 7$. The transfered method with $n^{Q} = 7$ outperforms the target-only approach ($\widehat{\boldsymbol{W}}^Q$) even with $n^Q = 10$.
    %, often yielding near-zero singular values.

    \item \textbf{Shift Adaptation:} Our method aligns well the estimated loading space with the target domain by adaptively aggregating data, avoiding the persistent bias inherent in the source-only approach ($\widehat{\boldsymbol{W}}^P$), which fails to account for the distribution shift despite a large $n^P$.  
    
    \item \textbf{Factor Recovery Robustness:} Our model-selection strategy successfully recovers the latent factor space even under severe target sample size constraints, empirically confirming the theoretical guarantees established in Theorem~\ref{thm:factorTransfer}. 
\end{itemize}

\begin{figure}[htbp]
    \centering
    \includegraphics[width=0.5\textwidth]{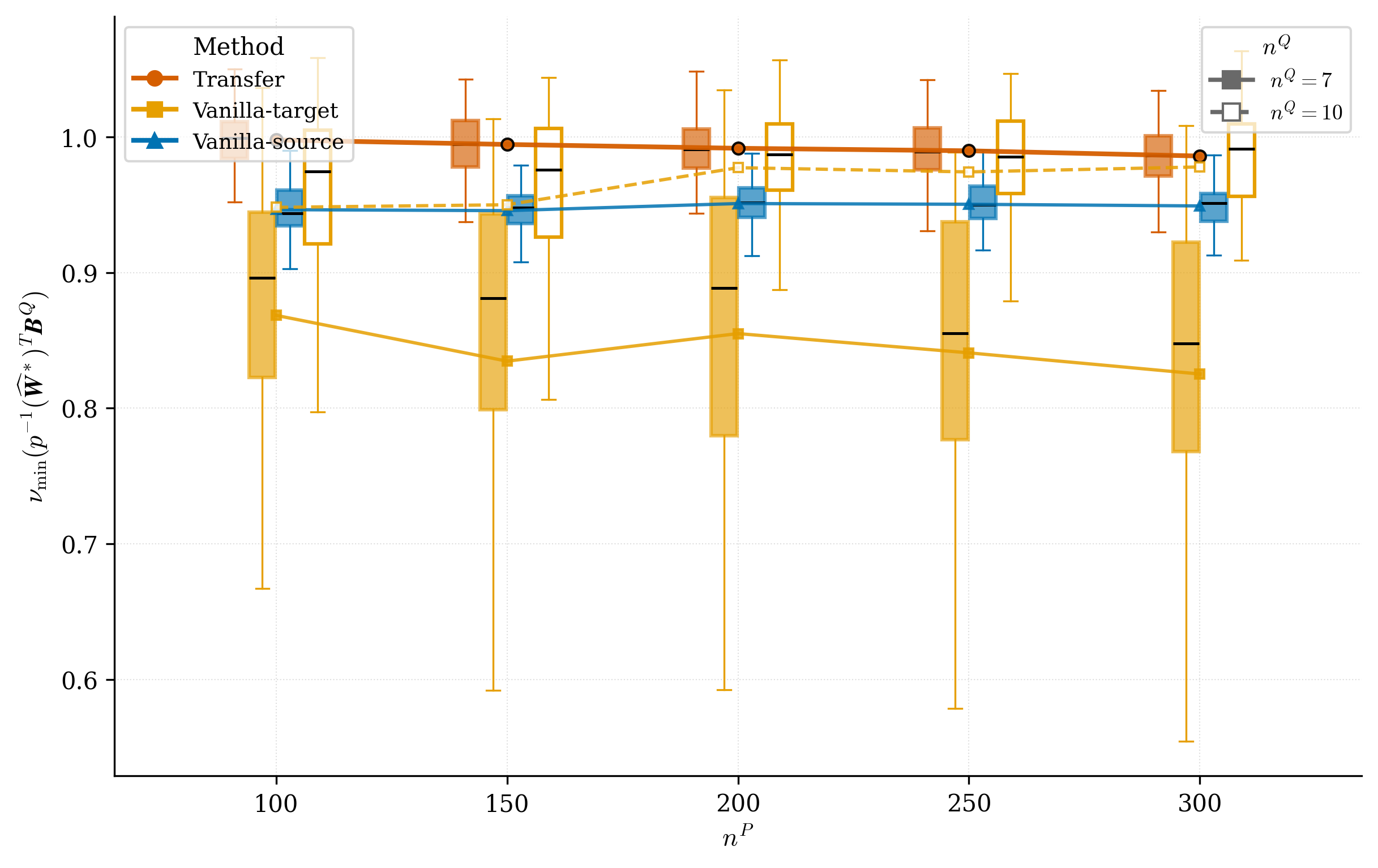}
    \begin{singlespace}
    \caption{Transfer learning of eigenspaces under covariate shift. The space alignment is measured by $\nu_{\min}(p^{-1}(\widehat{\boldsymbol{W}}^*)^T \boldsymbol{B}^Q)$, a critical quantity for eigenspace learning via diversified projection. The plot is based on 100 simulations. It compares transfer factor estimation $\widehat{\boldsymbol{W}}^{\TL}$ with $n^{Q}=7$ (solid blue), target-only factor estimation $\widehat{\boldsymbol{W}}^Q$ (solid yellow for $n^Q = 7$, hollow yellow for $n^Q =10$), and source-only factor estimation $\widehat{\boldsymbol{W}}^P$ (solid green).}
    \label{fig:sim-covariate}
    \end{singlespace}
\end{figure}

\subsection{Posterior Shift Estimation}
\label{sec:sim-posterior}

\textbf{Candidate methods.}
We evaluate the performance of our proposed method alongside several baseline and oracle approaches to quantify the benefits of integrating factor-augmentation with residual fine-tuning:  \textbf{SMART-FAN-Lasso}: The proposed transfer learning framework; \textbf{SMART-FT-Vanilla-NN}: Applies the same transfer learning pipeline as SMART-FAN-Lasso, but uses a vanilla neural network architecture for both source and target, operating on raw covariates $\boldsymbol{x}$; \textbf{FAST-NN (Source or Target Only)}: A single-domain baseline utilizing FAST-NN,  trained exclusively on source data and on target data separately; \textbf{Vanilla-NN (Source or Target Only)}: A  standard deep ReLU network, trained exclusively on source data and on target data separately; \textbf{Oracle-NN}: An idealized benchmark in which only the residual function $h$ is learned, using the true latent factors, true idiosyncratic components, and the true source regression function as inputs. 

\textbf{Data generating process.}
The source and target covariates are generated following the factor model structure $\boldsymbol{x}^* = \boldsymbol{B}^* \boldsymbol{f}^* + \boldsymbol{u}^*$ for $* \in \{P, Q\}$, similar to the previous experiment. In this setting, we introduce a subtle shift in the target loading matrix by setting $\epsilon_{ij}^{B} \sim 0.1 \times \mathrm{Rademacher}$. We evaluate the performance on a fixed dimension of the covariate $p = 10000$ for both the source and the target data. 
We take 
\begin{align*}
	g^P(\boldsymbol{f}^P, \boldsymbol{u}_{\mathcal{J}^P}^P) & = f_1^P + \mathrm{sin}(f_2^P) + f_3^P (f_4^P)^2 + u_1^P + (u_2^P)^2 u_3^P + \mathrm{log}(3 + u_4^P) + \mathrm{tan}(u_5^P) \\
	g^Q(\boldsymbol{f}^Q, \boldsymbol{u}_{\mathcal{J}^Q}^Q) & = h\left(\boldsymbol{f}^Q, \boldsymbol{u}_{\mathcal{J}}^Q, g^P(\boldsymbol{f}^Q, \boldsymbol{u}_{\mathcal{J}^P}^Q)\right) = f_1^Q - f_2^Q + u_3^Q + 0.5 \, g^P(\boldsymbol{f}^Q, \boldsymbol{u}_{\mathcal{J}^P}^Q) 
\end{align*}
and generate $y^P = g^P(\boldsymbol{f}^P, \boldsymbol{u}_{\mathcal{J}^P}^P) + \epsilon^P$ and $y^Q = g^Q(\boldsymbol{f}^Q, \boldsymbol{u}_{\mathcal{J}^Q}^Q) + \epsilon^Q$.

The noises are set to be $\epsilon^P, \epsilon^Q \overset{\mathrm{i.i.d}}{\sim} \pazocal{N}(0, 1/4)$. For the source data, we fix the sample sizes at $n_{\mathrm{train}}^P = 5000$, $n_{\mathrm{valid}}^P = 0.1 n_{\mathrm{train}}^P = 500$, $n_{\mathrm{unlabeled}}^P = 100$, and $n_{\mathrm{test}}^P = 100$, where the validation set is used for neural network tuning. For the target data, we vary the training sample size over $n_{\mathrm{train}}^Q \in \{10, 50, 100, 200, 300, 400, 500, 600, 700, 800, 900, 1000\}$ and fix $n_{\mathrm{valid}}^Q = 0.1 n_{\mathrm{train}}^Q$, $n_{\mathrm{unlabeled}}^Q = 100$, and $n_{\mathrm{test}}^Q = 100$ for each $n_{\mathrm{train}}^Q$. The unlabeled covariate set is used for estimating the transfer projection matrix, and the threshold $\delta$ is set to be the same as in Section~\ref{sec:sim-covariate}.
We replicate the data-generating process for $100$ times for each setting. Neural network architectures and training procedures are detailed in Appendix~\ref{app:numerical}. 

\textbf{Results.}
Figure~\ref{fig:sim-posterior} shows the test RMSE across candidate methods as the target sample size $n_{\mathrm{train}}^Q$ increases. The following points summarize the findings:

\begin{itemize}
    \item \textbf{Efficacy of Fine-Tuning:} SMART-FAN-Lasso consistently attains the lowest RMSE among all practical methods. The significant performance gain over source or target-only models (FAST-NN and Vanilla-NN) underscores the critical role of SMART fine-tuning in adapting to distribution shifts.
    \item \textbf{Benefit of Factor-Augmentation:} By outperforming SMART-Vanilla-NN across all sample regimes, SMART-FAN-Lasso demonstrates that the factor-augmented sparse throughput architecture is superior for high-dimensional transfer learning.
    \item \textbf{Convergence to Oracle:} As the target sample size $n^Q$ grows, the performance of SMART-FAN-Lasso converges toward the Oracle-NN, empirically validating the result in Theorem~\ref{thm:fineTuning2}. Note that $n^P = 5000$ is fixed, so the estimation error of the source model is essentially fixed, independent of $n^Q$. %This is why there is little further improvement once $n^Q \geq 2000$, a regime where transfer learning is less needed.
\end{itemize}

\begin{figure}[htbp]
    \centering
    \includegraphics[width=0.5\textwidth]{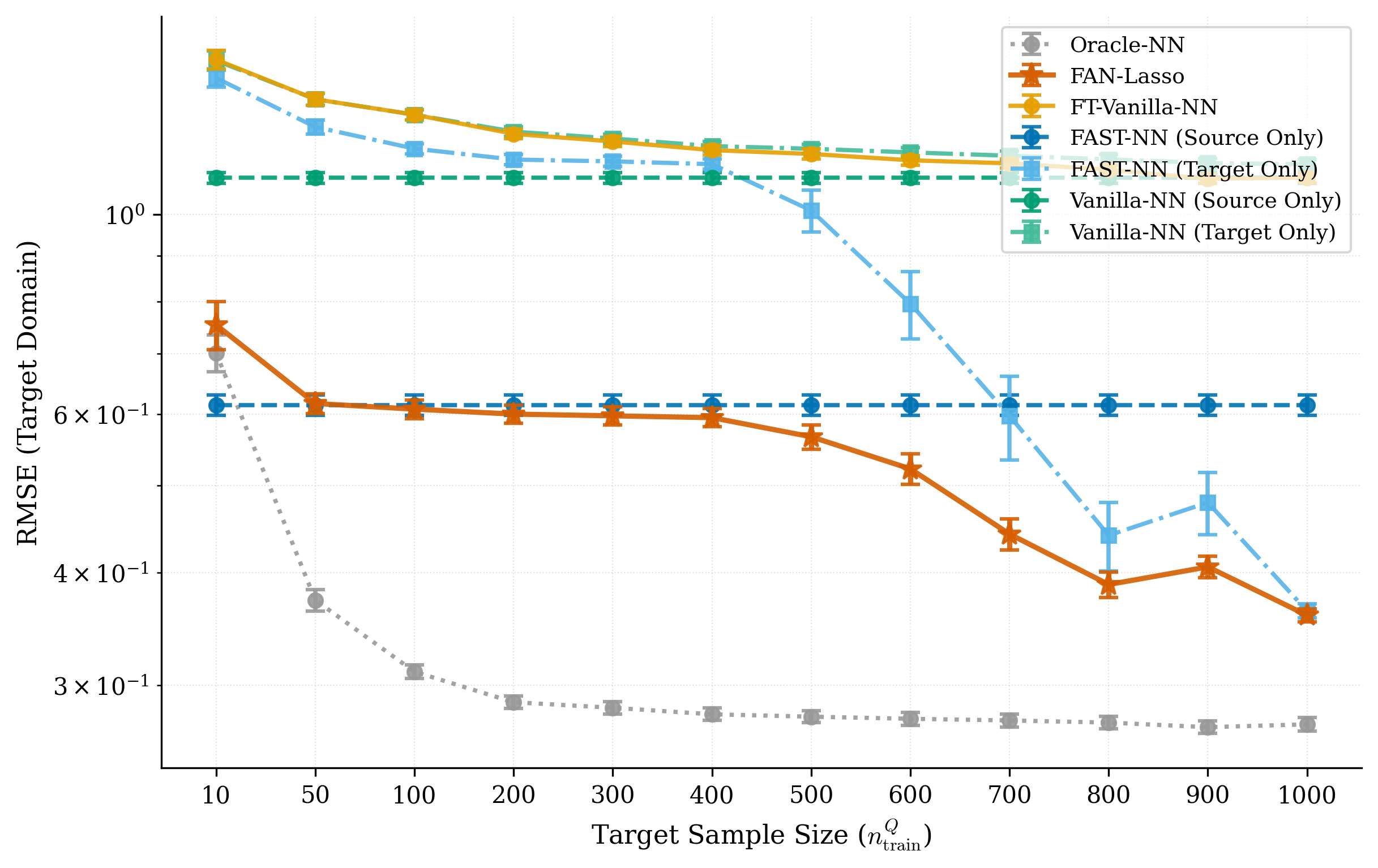}
    \begin{singlespace}
    \caption{Method Comparison: Target RMSE (with 95\% CI) vs.\ Target Sample Size ($n_{\text{train}}^Q$). The plot compares SMART-FAN-Lasso, SMART-Vanilla-NN, source or target-only baseline models (FAST-NN and Vanilla-NN), and the Oracle-NN. }
    \label{fig:sim-posterior}
    \end{singlespace}
\end{figure}

\section{Real Data Application}
\label{sec:application}

We apply the SMART-FAN-Lasso framework to the \emph{Communities and Crime} dataset from the UCI Machine Learning Repository \citep{redmond2002data, uci_crime_dataset}, which integrates socio-economic indicators from the 1990 US Census with law enforcement data and crime statistics from the 1995 FBI reports. The dataset comprises communities characterized by a comprehensive set of features spanning demographic composition, socio-economic conditions, family structure, housing characteristics, and law enforcement resources. The prediction task is to estimate \texttt{ViolentCrimesPerPop}, the per capita rate of violent crimes (murder, rape, robbery, and assault) per 100,000 population, normalized to $[0,1]$.

Following standard preprocessing to remove non-predictive identifiers and attributes with excessive missingness, the cleaned dataset comprises 1,993 communities characterized by 100 features. To construct a realistic transfer learning scenario reflecting spatial heterogeneity in crime patterns, we stratify communities by population density. Since urban-rural boundaries vary across classification schemes \citep{ingram2012nchs, uscensus_urbanrural}, we employ an approximate median split on the normalized population variable to create domains with distinct density profiles. The resulting \textbf{source domain} contains 1,595 higher-density communities (80.03\%), and the \textbf{target domain} contains 398 lower-density communities (19.97\%). This stratification reflects the practical scenario of adapting models trained on data-rich urban areas to predict crime in data-scarce lower-density regions, where limited data availability often necessitates transfer learning approaches.

We compare the same candidate methods except for the Oracle-NN as in Section~\ref{sec:sim-posterior}. For each domain, we randomly partition the data into 80\% training, 10\% validation, and 10\% test sets. For the factor estimation component of SMART-FAN-Lasso and FAST-NN, we reserve 10\% of the training data for diversified projection matrix estimation, with the remaining 90\% used for supervised fitting. For each method, we perform hyperparameter tuning on the validation set, selecting the optimal network depth from $L \in \{4, 5, 6\}$ and width from $N \in \{250, 350, 450\}$ (for transfer methods, both source-stage and target-stage architectures are tuned jointly) with the number of working factors $\overline{r}=3$. We repeat this process for 10 independent random splits. Table~\ref{tab:application-results} summarizes the target (lower-density) test performance. SMART-FAN-Lasso achieves the lowest RMSE, demonstrating substantial improvement over source-only baselines and outperforming fine-tuning Vanilla-NN, consistent with our simulation findings.

\begin{table}[htbp]
    \centering
    \resizebox{\textwidth}{!}{%
    \begin{tabular}{l|cccccc}
        \toprule
         & \textbf{FAN-Lasso} & FT-Vanilla-NN & \shortstack{FAST-NN \\ (Source Only)} & \shortstack{Vanilla-NN \\ (Source Only)} & \shortstack{Vanilla-NN \\ (Target Only)} & \shortstack{FAST-NN \\ (Target Only)} \\
        \midrule
        RMSE $\pm$ SE & $\mathbf{0.1333 \pm 0.0088}$ & $0.1387 \pm 0.0104$ & $0.1383 \pm 0.0088$ & $0.1422 \pm 0.0101$ & $0.1431 \pm 0.0111$ & $0.1574 \pm 0.0104$ \\
        \bottomrule
    \end{tabular}%
    }
    \begin{singlespace}
    \caption{Real Data Application: Prediction Performance on Communities and Crime Dataset (Lower-Density Communities). Methods are ordered by RMSE. RMSE and standard error (SE) are computed over 10 simulation replicates.}
    \label{tab:application-results}
    \end{singlespace}
\end{table}

\bibliographystyle{apalike2}

\bibliography{ref}

@article{fan2024noise,
  title={How do noise tails impact on deep ReLU networks?},
  author={Fan, Jianqing and Gu, Yihong and Zhou, Wen-Xin},
  journal={The Annals of Statistics},
  volume={52},
  number={4},
  pages={1845--1871},
  year={2024},
  publisher={Institute of Mathematical Statistics}
}

@article{lu2021deep,
	title={Deep network approximation for smooth functions},
	author={Lu, Jianfeng and Shen, Zuowei and Yang, Haizhao and Zhang, Shijun},
	journal={SIAM Journal on Mathematical Analysis},
	volume={53},
	number={5},
	pages={5465--5506},
	year={2021},
	publisher={SIAM}
}

@article{zhong2024neural,
	title={Neural networks for partially linear quantile regression},
	author={Zhong, Qixian and Wang, Jane-Ling},
	journal={Journal of Business \& Economic Statistics},
	volume={42},
	number={2},
	pages={603--614},
	year={2024},
	publisher={Taylor \& Francis}
}

@article{wang2017asymptotics,
	title={Asymptotics of empirical eigenstructure for high dimensional spiked covariance},
	author={Wang, Weichen and Fan, Jianqing},
	journal={Annals of statistics},
	volume={45},
	number={3},
	pages={1342},
	year={2017}
}

@book{Stewart90,
  author = {Stewart, G. W. and guang Sun, Ji},
  publisher = {Academic Press},
  title = {Matrix Perturbation Theory},
  year = 1990
}

@book{tsybakov2008introduction,
  title={Introduction to Nonparametric Estimation},
  author={Tsybakov, A.B.},
  isbn={9780387790527},
  lccn={2008939894},
  series={Springer Series in Statistics},
  url={https://books.google.com/books?id=mwB8rUBsbqoC},
  year={2008},
  publisher={Springer New York}
}

@book{gyorfi2002distribution,
  title={A Distribution-Free Theory of Nonparametric Regression},
  author={Gy{\"o}rfi, L. and Kohler, M. and Krzyzak, A. and Walk, H.},
  isbn={9780387954417},
  lccn={20021151},
  series={Springer Series in Statistics},
  url={https://books.google.com/books?id=Ovmb9oGBlo0C},
  year={2002},
  publisher={Springer New York}
}

@article{farrell2018DeepNeural,
	author = {Farrell, Max H and Liang, Tengyuan and Misra, Sanjog},
	journal = {Econometrica},
	month = jan,
	number = {1},
	pages = {181-213},
	title = {Deep neural networks for estimation and inference},
	volume = {89},
	year = {2021},
	doi = {10.3982/ECTA16901}}

@article{schmidt2020nonparametric,
  title={Nonparametric regression using deep neural networks with ReLU activation function},
  author={Schmidt-Hieber, Johannes},
  journal={The Annals of Statistics},
  volume={48},
  number={4},
  pages={1875--1897},
  year={2020},
  publisher={Institute of Mathematical Statistics}
}

@article{weyl1909berBQ,
  title={{\"U}ber beschr{\"a}nkte quadratische formen, deren differenz vollstetig ist},
  author={Hermann Von Weyl},
  journal={Rendiconti del Circolo Matematico di Palermo (1884-1940)},
  year={1909},
  volume={27},
  pages={373-392},
  url={https://api.semanticscholar.org/CorpusID:122374162}
}

@article{fan2025robust,
  title={Robust transfer learning with unreliable source data},
  author={Fan, Jianqing and Gao, Cheng and Klusowski, Jason M},
  journal={The Annals of Statistics},
  volume={53},
  number={4},
  pages={1728--1752},
  year={2025},
  publisher={Institute of Mathematical Statistics}
}

@article{cape2017subspace,
  title={The two-to-infinity norm and singular subspace geometry with applications to high-dimensional statistics},
  author={Joshua Cape and Minh Tang and Carey E. Priebe},
  journal={The Annals of Statistics},
  year={2017},
  url={https://api.semanticscholar.org/CorpusID:51767254}
}

@article{fan2022learning,
author = {Jianqing Fan and Yuan Liao},
title = {Learning Latent Factors From Diversified Projections and Its Applications to Over-Estimated and Weak Factors},
journal = {Journal of the American Statistical Association},
volume = {117},
number = {538},
pages = {909--924},
year = {2022},
publisher = {ASA Website},
doi = {10.1080/01621459.2020.1831927}
}

@article{zhou2023cross,
  title={Cross-fitted residual regression for high-dimensional heteroscedasticity pursuit},
  author={Zhou, Le and Zou, Hui},
  journal={Journal of the American Statistical Association},
  volume={118},
  number={542},
  pages={1056--1065},
  year={2023},
  publisher={Taylor \& Francis}
}

@article{zhao2024residual,
  title={Residual channel attention based sample adaptation few-shot learning for hyperspectral image classification},
  author={Zhao, Yuefeng and Sun, Jingqi and Hu, Nannan and Zai, Chengmin and Han, Yanwei},
  journal={Scientific Reports},
  volume={14},
  number={1},
  pages={26746},
  year={2024},
  publisher={Nature Publishing Group UK London}
}

@article{ankile2025residual,
  title={Residual Off-Policy RL for Finetuning Behavior Cloning Policies},
  author={Ankile, Lars and Jiang, Zhenyu and Duan, Rocky and Shi, Guanya and Abbeel, Pieter and Nagabandi, Anusha},
  journal={arXiv preprint arXiv:2509.19301},
  year={2025}
}

@article{fan2001variable,
  title={Variable selection via nonconcave penalized likelihood and its oracle properties},
  author={Fan, Jianqing and Li, Runze},
  journal={Journal of the American statistical Association},
  volume={96},
  number={456},
  pages={1348--1360},
  year={2001},
  publisher={Taylor \& Francis}
}

@article{ohn2022nonconvex,
  title={Nonconvex sparse regularization for deep neural networks and its optimality},
  author={Ohn, Ilsang and Kim, Yongdai},
  journal={Neural computation},
  volume={34},
  number={2},
  pages={476--517},
  year={2022},
  publisher={MIT Press One Rogers Street, Cambridge, MA 02142-1209, USA journals-info~…}
}

@article{pan2009survey,
  title={A survey on transfer learning},
  author={Pan, Sinno Jialin and Yang, Qiang},
  journal={IEEE Transactions on Knowledge and Data Engineering},
  volume={22},
  number={10},
  pages={1345--1359},
  year={2009},
  publisher={IEEE}
}

@article{ben2010theory,
  title={A theory of learning from different domains},
  author={Ben-David, Shai and Blitzer, John and Crammer, Koby and Kulesza, Alex and Pereira, Fernando and Vaughan, Jennifer Wortman},
  journal={Machine Learning},
  volume={79},
  number={1},
  pages={151--175},
  year={2010},
  publisher={Springer}
}

@article{li2022transfer,
  title={Transfer learning for high-dimensional linear regression: Prediction, estimation and minimax optimality},
  author={Li, Sai and Cai, T Tony and Li, Hongzhe},
  journal={Journal of the Royal Statistical Society Series B: Statistical Methodology},
  volume={84},
  number={1},
  pages={149--173},
  year={2022},
  publisher={Oxford University Press}
}

@article{cai2024transfer,
  title={Transfer learning for nonparametric regression: Non-asymptotic minimax analysis and adaptive procedure},
  author={Cai, T Tony and Pu, Hongming},
  journal={arXiv preprint arXiv:2401.12272},
  year={2024}
}

@article{tian2023transfer,
  title={Transfer learning under high-dimensional generalized linear models},
  author={Tian, Ye and Feng, Yang},
  journal={Journal of the American Statistical Association},
  volume={118},
  number={544},
  pages={2684--2697},
  year={2023},
  publisher={Taylor \& Francis}
}

@book{quinonero2022dataset,
  title={Dataset Shift in Machine Learning},
  author={Qui{\~n}onero-Candela, Joaquin and Sugiyama, Masashi and Schwaighofer, Anton and Lawrence, Neil D},
  year={2022},
  publisher={MIT Press}
}

@article{gretton2009covariate,
  title={Covariate shift by kernel mean matching},
  author={Gretton, Arthur and Smola, Alex and Huang, Jiayuan and Schmittfull, Marcel and Borgwardt, Karsten and Sch{\"o}lkopf, Bernhard and others},
  journal={Dataset Shift in Machine Learning},
  volume={3},
  number={4},
  pages={5},
  year={2009}
}

@article{ma2023optimally,
  title={Optimally tackling covariate shift in RKHS-based nonparametric regression},
  author={Ma, Cong and Pathak, Reese and Wainwright, Martin J},
  journal={The Annals of Statistics},
  volume={51},
  number={2},
  pages={738--761},
  year={2023},
  publisher={Institute of Mathematical Statistics}
}

@article{kpotufe2021marginal,
  title={Marginal singularity and the benefits of labels in covariate-shift},
  author={Kpotufe, Samory and Martinet, Guillaume},
  journal={The Annals of Statistics},
  volume={49},
  number={6},
  pages={3299--3323},
  year={2021},
  publisher={JSTOR}
}

@article{yang2024doubly,
  title={Doubly robust calibration of prediction sets under covariate shift},
  author={Yang, Yachong and Kuchibhotla, Arun Kumar and Tchetgen Tchetgen, Eric},
  journal={Journal of the Royal Statistical Society Series B: Statistical Methodology},
  volume={86},
  number={4},
  pages={943--965},
  year={2024},
  publisher={Oxford University Press US}
}

@article{cai2025semi,
  title={Semi-supervised triply robust inductive transfer learning},
  author={Cai, Tianxi and Li, Mengyan and Liu, Molei},
  journal={Journal of the American Statistical Association},
  volume={120},
  number={550},
  pages={1037--1047},
  year={2025},
  publisher={Taylor \& Francis}
}

@article{hanneke2019value,
  title={On the value of target data in transfer learning},
  author={Hanneke, Steve and Kpotufe, Samory},
  journal={Advances in Neural Information Processing Systems},
  volume={32},
  year={2019}
}

@article{tripuraneni2020theory,
  title={On the theory of transfer learning: The importance of task diversity},
  author={Tripuraneni, Nilesh and Jordan, Michael and Jin, Chi},
  journal={Advances in Neural Information Processing Systems},
  volume={33},
  pages={7852--7862},
  year={2020}
}

@article{kumar2022fine,
  title={Fine-tuning can distort pretrained features and underperform out-of-distribution},
  author={Kumar, Ananya and Raghunathan, Aditi and Jones, Robbie and Ma, Tengyu and Liang, Percy},
  journal={arXiv preprint arXiv:2202.10054},
  year={2022}
}

@article{hu2022lora,
  title={LoRA: Low-rank adaptation of large language models},
  author={Hu, Edward J and Shen, Yelong and Wallis, Phillip and Allen-Zhu, Zeyuan and Li, Yuanzhi and Wang, Shean and Wang, Lu and Chen, Weizhu and others},
  journal={ICLR},
  volume={1},
  number={2},
  pages={3},
  year={2022}
}

@article{dettmers2023qlora,
  title={QLoRA: Efficient finetuning of quantized LLMs},
  author={Dettmers, Tim and Pagnoni, Artidoro and Holtzman, Ari and Zettlemoyer, Luke},
  journal={Advances in Neural Information Processing Systems},
  volume={36},
  pages={10088--10115},
  year={2023}
}

@inproceedings{lester2021power,
  title={The Power of Scale for Parameter-Efficient Prompt Tuning},
  author={Lester, Brian and Al-Rfou, Rami and Constant, Noah},
  booktitle={Proceedings of the 2021 Conference on Empirical Methods in Natural Language Processing},
  pages={3045--3059},
  year={2021},
  publisher={Association for Computational Linguistics}
}

@article{bastani2021predicting,
  title={Predicting with proxies: Transfer learning in high dimension},
  author={Bastani, Hamsa},
  journal={Management Science},
  volume={67},
  number={5},
  pages={2964--2984},
  year={2021},
  publisher={INFORMS}
}

@article{petersen2018optimal,
  title={Optimal approximation of piecewise smooth functions using deep ReLU neural networks},
  author={Petersen, Philipp and Voigtlaender, Felix},
  journal={Neural Networks},
  volume={108},
  pages={296--330},
  year={2018},
  publisher={Elsevier}
}

@article{kohler2021rate,
  title={On the rate of convergence of fully connected deep neural network regression estimates},
  author={Kohler, Michael and Langer, Sophie},
  journal={The Annals of Statistics},
  volume={49},
  number={4},
  pages={2231--2249},
  year={2021},
  publisher={JSTOR}
}

@article{farrell2021deep,
  title={Deep neural networks for estimation and inference},
  author={Farrell, Max H and Liang, Tengyuan and Misra, Sanjog},
  journal={Econometrica},
  volume={89},
  number={1},
  pages={181--213},
  year={2021},
  publisher={Wiley Online Library}
}

@article{bartlett2019nearly,
  title={Nearly-tight VC-dimension and pseudodimension bounds for piecewise linear neural networks},
  author={Bartlett, Peter L and Harvey, Nick and Liaw, Christopher and Mehrabian, Abbas},
  journal={Journal of Machine Learning Research},
  volume={20},
  number={63},
  pages={1--17},
  year={2019}
}

@article{stock2002forecasting,
  title={Forecasting using principal components from a large number of predictors},
  author={Stock, James H and Watson, Mark W},
  journal={Journal of the American Statistical Association},
  volume={97},
  number={460},
  pages={1167--1179},
  year={2002},
  publisher={Taylor \& Francis}
}

@article{fan2024factor,
  title={Factor augmented sparse throughput deep ReLU neural networks for high dimensional regression},
  author={Fan, Jianqing and Gu, Yihong},
  journal={Journal of the American Statistical Association},
  volume={119},
  number={548},
  pages={2680--2694},
  year={2024},
  publisher={Taylor \& Francis}
}

@article{bai2008large,
  title={Large dimensional factor analysis},
  author={Bai, Jushan and Ng, Serena and others},
  journal={Foundations and Trends in Econometrics},
  volume={3},
  number={2},
  pages={89--163},
  year={2008},
  publisher={Now Publishers, Inc.}
}

@article{stock2002macroeconomic,
  title={Macroeconomic forecasting using diffusion indexes},
  author={Stock, James H and Watson, Mark W},
  journal={Journal of Business \& Economic Statistics},
  volume={20},
  number={2},
  pages={147--162},
  year={2002},
  publisher={Taylor \& Francis}
}

@article{forni2005generalized,
  title={The generalized dynamic factor model: One-sided estimation and forecasting},
  author={Forni, Mario and Hallin, Marc and Lippi, Marco and Reichlin, Lucrezia},
  journal={Journal of the American Statistical Association},
  volume={100},
  number={471},
  pages={830--840},
  year={2005},
  publisher={Taylor \& Francis}
}

@article{onatski2012asymptotics,
  title={Asymptotics of the principal components estimator of large factor models with weakly influential factors},
  author={Onatski, Alexei},
  journal={Journal of Econometrics},
  volume={168},
  number={2},
  pages={244--258},
  year={2012},
  publisher={Elsevier}
}

@article{chudik2011weak,
  title={Weak and strong cross-section dependence and estimation of large panels},
  author={Chudik, Alexander and Pesaran, M Hashem and Tosetti, Elisa},
  journal={Econometrics Journal},
  volume={14},
  number={1},
  pages={C45--C90},
  year={2011},
  publisher={Oxford University Press}
}

@article{johnstone2009consistency,
  title={On consistency and sparsity for principal components analysis in high dimensions},
  author={Johnstone, Iain M and Lu, Arthur Yu},
  journal={Journal of the American Statistical Association},
  volume={104},
  number={486},
  pages={682--693},
  year={2009},
  publisher={Taylor \& Francis}
}

@article{paul2007asymptotics,
  title={Asymptotics of sample eigenstructure for a large dimensional spiked covariance model},
  author={Paul, Debashis},
  journal={Statistica Sinica},
  volume={17},
  number={4},
  pages={1617--1642},
  year={2007},
  publisher={JSTOR}
}

@article{ge2023provable,
  title={On the provable advantage of unsupervised pretraining},
  author={Ge, Jiawei and Tang, Shange and Fan, Jianqing and Jin, Chi},
  journal={arXiv preprint arXiv:2303.01566},
  year={2023}
}

@article{ge2023maximum,
  title={Maximum likelihood estimation is all you need for well-specified covariate shift},
  author={Ge, Jiawei and Tang, Shange and Fan, Jianqing and Ma, Cong and Jin, Chi},
  journal={arXiv preprint arXiv:2311.15961},
  year={2023}
}

@article{ingram2012nchs,
  title={NCHS urban-rural classification scheme for counties.},
  author={Ingram, Deborah D and Franco, Sheila J},
  journal={Vital and health statistics. Series 2, Data evaluation and methods research},
  number={154},
  pages={1--65},
  year={2012}
}

@misc{uscensus_urbanrural,
  author = {{U.S. Census Bureau}},
  title = {Urban and Rural},
  year = {2024},
  url = {https://www.census.gov/programs-surveys/geography/guidance/geo-areas/urban-rural.html},
  note = {Accessed: January 29, 2026. Last revised: December 16, 2024}
}

@article{redmond2002data,
  title={A data-driven software tool for enabling cooperative information sharing among police departments},
  author={Redmond, Michael A and Baveja, Alok},
  journal={European Journal of Operational Research},
  volume={141},
  number={3},
  pages={660--678},
  year={2002},
  publisher={Elsevier}
}

@misc{uci_crime_dataset,
  author = {Redmond, Michael},
  title = {{Communities and Crime Data Set}},
  year = {2009},
  howpublished = {UCI Machine Learning Repository},
  note = {\url{https://archive.ics.uci.edu/dataset/183/communities+and+crime}}
}

\newpage
\appendix
\phantomsection
\section*{\LARGE\bfseries Appendix}
\addcontentsline{toc}{section}{Appendix}

% !TEX root = ../main.tex
\noindent\textbf{Overview of Appendix.}
The appendix is organized as follows.  Section \ref{app:covariate} contains all the proofs for the transferred factor estimate in Section \ref{sec:factorTransfer}. Section \ref{app:posterior} gives all the proofs for the fine-tuning estimator in Section \ref{sec:fineTuning}.
Section \ref{app:aux} includes the proofs for all auxiliary results, along with technical lemmas presented in Sections \ref{app:covariate} and \ref{app:posterior}.   
For the empirical data, some additional details are given in Section~\ref{app:numerical}.

% !TEX root = ../main.tex
\section[Proofs for the transferred factor estimate]{Proofs for the transferred factor estimate in Section~\ref{sec:factorTransfer}}
\label{app:covariate}

In this section, we detail the proofs for the transferred factor estimator.  Lemma \ref{lemma:factorTransferLemma} establishes a Frobenius norm error bound for the transferred covariance estimator, following a case-by-case analysis of whether the factor loading discrepancy between source and target exceeds a certain threshold. Theorem \ref{thm:factorTransfer} then follows from Lemma \ref{lemma:factorTransferLemma} and a standard eigenvalue perturbation argument. 

\subsection{Proof of Theorem \ref{thm:factorTransfer}}

For any $* \in \{P, Q, A, \TL\}$, define $\widehat{\boldsymbol{V}}^*$ and $\widehat{\boldsymbol{W}}^*$ for the construction of the diversified projection matrices to be
\begin{align*}
    \widehat{\boldsymbol{V}}^* := [\widehat{\boldsymbol{v}}^*_1, \dots, \widehat{\boldsymbol{v}}^*_{\overline{r}}], \quad \widehat{\boldsymbol{W}}^* := p^{1/2}[\widehat{\boldsymbol{v}}^*_1, \dots, \widehat{\boldsymbol{v}}^*_{\overline{r}}].
\end{align*}
Let $\boldsymbol{S} = \boldsymbol{B}^Q (\boldsymbol{B}^Q)^\top$, which has the eigen-decomposition $\boldsymbol{S} = \boldsymbol{V}^Q \Lambda (\boldsymbol{V}^Q)^\top$, where $\Lambda$ is an $r \times r$ diagonal matrix and $(\boldsymbol{V}^Q)^\top \boldsymbol{V}^Q = \boldsymbol{I}_r$. It follows from the identification condition that $\boldsymbol{B}^Q = \boldsymbol{V}^Q \boldsymbol{\Lambda}^{1/2}$.

A key component of the proof is the Frobenius norm of the matrix $\widehat{\boldsymbol{\Sigma}}^{\TL} - \boldsymbol{S}$, i.e., the perturbation bound of the covariance matrix. We need the following technical lemma to bound it. The proof is given in Appendix \ref{app:aux-lemma-1}.

\begin{lemma}
\label{lemma:factorTransferLemma}
Under Assumptions \ref{ass:factorLoadingDiff} - \ref{ass:pervasiveness}, for some universal constant $c_4$ we have
\begin{equation}
    \label{eq:factorTransferLemma}
    \norm{\widehat{\boldsymbol{\Sigma}}^{\TL} - \boldsymbol{S}}_F \lesssim p \Big(\delta \land \varepsilon^A(t)\Big)
\end{equation}
with probability at least $1 - 6e^{-t}$ for any $t > 0$ such that $\delta \geq c_4 \varepsilon^Q(t)$.
\end{lemma}

\begin{proof}[Proof of Theorem \ref{thm:factorTransfer}]
We first prove the upper bound. Note that the maximum singular value satisfies
\begin{align*}
    \nu_{max} \left(p^{-1} (\widehat{\boldsymbol{W}}^{\TL})^\top \boldsymbol{B}^Q \right) &= \sup_{\boldsymbol{u} \in \mathbb{S}^{\overline{r} - 1}} p^{-1} \norm{(\boldsymbol{B}^Q)^\top \widehat{\boldsymbol{W}}^{\TL} \boldsymbol{u}}_2\\
    &\leq p^{-1/2} \norm{\boldsymbol{\Lambda}^{1/2}}_2 \sup_{\boldsymbol{u} \in \mathbb{S}^{\overline{r} - 1}}\norm{(\boldsymbol{V}^{\TL})^\top \widehat{\boldsymbol{V}}^{\TL} \boldsymbol{u}}_2\\
    &\leq p^{-1/2} \norm{\boldsymbol{\Lambda}^{1/2}}_2 \max \Big \{\sup_{\boldsymbol{u} \in \mathbb{S}^{\overline{r} - 1}} \norm{(\boldsymbol{V}^Q)^\top \widehat{\boldsymbol{V}}^Q \boldsymbol{u}}_2, \sup_{\boldsymbol{u} \in \mathbb{S}^{\overline{r} - 1}} \norm{(\boldsymbol{V}^A)^\top \widehat{\boldsymbol{V}}^A \boldsymbol{u}}_2 \Big\}.
\end{align*}
The pervasiveness assumption (Assumption \ref{ass:pervasiveness}) implies that
\begin{equation}
    \label{eq:LambdaEigenvalues}
    \boldsymbol{\Lambda}_{ii} \asymp p, \quad \forall i = 1, \dots, r.
\end{equation}
Hence, it follows from \eqref{eq:LambdaEigenvalues} and orthogonality of $\boldsymbol{V}^Q, \widehat{\boldsymbol{V}}^Q, \boldsymbol{V}^A, \widehat{\boldsymbol{V}}^A$ that 
\begin{align*}
    \nu_{max} \left(p^{-1} (\widehat{\boldsymbol{W}}^{\TL})^\top \boldsymbol{B}^Q \right) &\leq p^{-1/2} \norm{\boldsymbol{\Lambda}^{1/2}}_2 \max \Big \{\sup_{\boldsymbol{u} \in \mathbb{S}^{\overline{r} - 1}} \norm{(\boldsymbol{V}^Q)^\top \widehat{\boldsymbol{V}}^Q \boldsymbol{u}}_2, \sup_{\boldsymbol{u} \in \mathbb{S}^{\overline{r} - 1}} \norm{(\boldsymbol{V}^A)^\top \widehat{\boldsymbol{V}}^A \boldsymbol{u}}_2 \Big\}\\
    &\lesssim \max \Big \{\sup_{\boldsymbol{u} \in \mathbb{S}^{\overline{r} - 1}} \norm{(\boldsymbol{V}^Q)^\top \widehat{\boldsymbol{V}}^Q \boldsymbol{u}}_2, \sup_{\boldsymbol{u} \in \mathbb{S}^{\overline{r} - 1}} \norm{(\boldsymbol{V}^A)^\top \widehat{\boldsymbol{V}}^A \boldsymbol{u}}_2 \Big\}\\
    &\lesssim \sup_{\boldsymbol{u} \in \mathbb{S}^{\overline{r} - 1}} \norm{\boldsymbol{u}}_2\lesssim 1.
\end{align*}
Next, we prove the lower bound. Let $\widehat{\boldsymbol{V}}^{\TL}_r \in \mathbb{R}^{p \times r}$ be the leftmost $r \leq \overline{r}$ columns, i.e. top-$r$ eigenvectors, of $\widehat{\boldsymbol{V}}^{\TL}$, then 
\begin{equation}
    \label{eq:factorTransferProof2}
    \begin{split}
        \nu_{min} \left(p^{-1} (\widehat{\boldsymbol{W}}^{\TL})^\top \boldsymbol{B}^Q \right) &\geq p^{-1/2} \nu_{min} \left( (\widehat{ \boldsymbol{V}}^{\TL}_r)^\top \boldsymbol{B}^Q \right) \\
        &\geq p^{-1/2} \nu_{min} \left( (\widehat{\boldsymbol{V}}^{\TL}_r)^\top \boldsymbol{V}^Q \right) \nu_{min}(\boldsymbol{\Lambda}^{1/2}) \gtrsim \nu_{min} \left( (\widehat{\boldsymbol{V}}^{\TL}_r)^\top \boldsymbol{V}^Q \right).
    \end{split}
\end{equation}
where the last inequality follows from \eqref{eq:LambdaEigenvalues}. Therefore, it suffices to bound $\nu_{min}\left((\widehat{\boldsymbol{V}}^{\TL}_r)^\top \boldsymbol{V}^Q \right)$ from below.

Let $\lambda_1, \dots, \lambda_p$ denote the eigenvalues of $\boldsymbol{S}$ in a non-increasing order, and $\widehat{\lambda}_1^{\TL}, \dots, \widehat{\lambda}_p^{\TL}$ denote the eigenvalues of $\widehat{\boldsymbol{\Sigma}}^{\TL}$ in a non-increasing order. It is obvious that
\begin{equation}
    \label{eq:factorTransferProof3}
    \lambda_r \geq C_1 p,\quad \lambda_{r + 1} = 0.
\end{equation}
For simplicity, we rewrite $p^{-1} \norm{ \widehat{\boldsymbol{\Sigma}}^{\TL} - \boldsymbol{S}}_F$ by $\widehat{\delta}$. Furthermore, Weyl's Theorem \citep{weyl1909berBQ} indicates that
\begin{equation}
    \label{eq:weylTheorem}
    |\widehat{\lambda}_i^{\TL}-\lambda_i| \leq \norm{\widehat{\boldsymbol{\Sigma}}^{\TL} - \boldsymbol{S}}_2 \leq \norm{\widehat{\boldsymbol{\Sigma}}^{\TL} - \boldsymbol{S}}_F = p \widehat{\delta}, \quad \forall i = 1, \dots, p.
\end{equation} 
We claim that
\begin{equation}
    \label{eq:factorTransferProof4}
    \nu_{min} \left((\widehat{\boldsymbol{V}}^{\TL}_r)^\top \boldsymbol{V}^Q \right) \geq 1 - \frac{2 \widehat{\delta}}{C_1}.
\end{equation}
which is trivial when $\widehat{\delta}>C_1/2$ since in this case $\nu_{min} \left( (\widehat{\boldsymbol{V}}^{\TL}_r)^\top \boldsymbol{V}^Q \right)$ is always non-negative. It thus suffices to prove \eqref{eq:factorTransferProof4} when $\widehat{\delta}\leq C_1/2$. Combining \eqref{eq:factorTransferProof3} and \eqref{eq:weylTheorem}, we bound the eigen-gap between $\widehat{\boldsymbol{\Sigma}}^{\TL}$ and $S$ by
\begin{align*}
    \widetilde{\delta}: &= \inf \Big\{ |\widehat{\lambda}^{\TL}-\lambda|: \widehat{\lambda}^{\TL}\in (-\infty,\widehat{\lambda}^{\TL}_{r+1}],\ \lambda\in[\lambda_r,\lambda_1] \Big\}\\
    &= \lambda_r-\widehat{\lambda}^{\TL}_{r+1}\\
    &\geq |\lambda_r-\lambda_{r+1}|-|\lambda_{r+1}-\widehat{\lambda}^{\TL}_{r+1}|\\
    &\geq C_1p - p\widehat{\delta} \geq C_1p/2.
\end{align*}
It follows from Theorem V.3.6 in \cite{Stewart90} that
\begin{equation}
    \label{eq:angleDiff}
    \norm{\sin \boldsymbol{\Theta}(\widehat{\boldsymbol{V}}_r^{\TL}, \boldsymbol{V}^Q)}_F \leq \frac{p \widehat{\delta}}{\widetilde{\delta}} \leq \frac{2 \widehat{\delta}}{C_1},
\end{equation}
Here $\sin \boldsymbol{\Theta}(\widehat{\boldsymbol{V}}_r^{\TL}, \boldsymbol{V}^Q)$ is a $r\times r$ diagonal matrix satisfying
\begin{align*}
    \left[\sin \boldsymbol{\Theta} (\widehat{\boldsymbol{V}}_r^{\TL}, \boldsymbol{V}^Q)\right]^2 + \left[\cos \boldsymbol{\Theta} (\widehat{\boldsymbol{V}}_r^{\TL}, \boldsymbol{V}^Q)\right]^2 = \boldsymbol{I}_r
\end{align*}
where $\cos \boldsymbol{\Theta} (\widehat{\boldsymbol{V}}_r^{\TL}, \boldsymbol{V}^Q)$ is a $r\times r$ diagonal matrix of singular values of $(\widehat{\boldsymbol{V}}_r^{\TL})^\top \boldsymbol{V}^Q$. Hence,
\begin{align*}
    \nu^2_{min}\left( (\widehat{\boldsymbol{V}}^{\TL}_r)^\top \boldsymbol{V}^Q \right) &= \norm{\cos \boldsymbol{\Theta} (\widehat{\boldsymbol{V}}_r^{\TL}, \boldsymbol{V}^Q)}^2_2 \ \geq 1 -\norm{\sin \boldsymbol{\Theta} (\widehat{\boldsymbol{V}}_r^{\TL}, \boldsymbol{V}^Q)}^2_2\\
    &\geq 1 - \norm{\sin \boldsymbol{\Theta} (\widehat{\boldsymbol{V}}_r^{\TL}, \boldsymbol{V}^Q)}^2_F \ \geq 1 -\Big( \frac{2 \widehat{\delta}}{C_1} \Big)^2\\
    &\geq 1 - \frac{2 \widehat{\delta}}{C_1}.  
\end{align*}
by the fact that $\sqrt{1-x^2}\geq x^2$ for any $x\in [0,1].$ Thus we have proven \eqref{eq:factorTransferProof4}, and the theorem directly follows by plugging in Lemma \ref{lemma:factorTransferLemma} into \eqref{eq:factorTransferProof4}.
\end{proof}

\subsection{Proof of Proposition \ref{prop:pertubation}}
\begin{proof}
The first statement was proved in Lemma \ref{lemma:factorTransferLemma}. The second statement is obvious by combining \eqref{eq:angleDiff} and the following known inequality
\begin{align*}
    \min_{\boldsymbol{R} \in \pazocal{O}(r)} \norm{\widehat{\boldsymbol{V}}_r^{\TL} \boldsymbol{R} - \boldsymbol{V}^Q}_F \leq \sqrt{2} \norm{\sin \boldsymbol{\Theta}(\widehat{\boldsymbol{V}}_r^{\TL}, \boldsymbol{V}^Q)}_F.
\end{align*}
The minimizer, i.e., the best rotation of basis, can be given by the singular value decomposition (SVD) of $(\widehat{\boldsymbol{V}}_r^{\TL})^\top \boldsymbol{V}^Q$. See \cite{cape2017subspace} for details.
\end{proof}
% !TEX root = ../main.tex
\section[Proofs for the fine-tuning estimator]{Proofs for the fine-tuning estimator in Section~\ref{sec:fineTuning}}
\label{app:posterior}

\subsection{Additional technical notations}
For any two functions $m_1, m_2$, define
\begin{align*}
    \norm{m_1-m_2}_2^2 = \EE_{(\boldsymbol{f}^Q, \boldsymbol{u}^Q)} \Big[ |m_1-m_2|^2 \Big].
\end{align*}
Similar to $\pazocal{E}^Q(m)$ and $\widehat{\pazocal{E}}^Q(m)$, for any function $m: \R^{p} \rightarrow \R$, we define
\begin{align*}
    \pazocal{E}_P(m) &\overset{def}{=} \EE_{(\boldsymbol{f}^P, \boldsymbol{u}^P)} \Big[ |m(\boldsymbol{x}^P) - g^P(\boldsymbol{x}^P_{\pazocal{J}^P})|^2 \Big], \\
    \widehat{\pazocal{E}}_P(m) &\overset{def}{=} \frac{1}{n^P} \sum_{i = 1}^{n^P} |m(\boldsymbol{x}_i^P) - g^P(\boldsymbol{x}^P_{i, \pazocal{J}^P})|^2.
\end{align*}
For any fixed function $s: \R^{p} \rightarrow \R$, define the function class $\pazocal{F}_s$,
\begin{align*}
    \pazocal{F}_s = \{m(x; \boldsymbol{W}^Q, g, \boldsymbol{\Theta}, s) =& g(\Big[p^{-1}(\boldsymbol{W}^Q)^\top \boldsymbol{x}, \mathrm{trun}_M(\boldsymbol{\Theta}^\top \boldsymbol{x}), s(\boldsymbol{x}) \Big]):\\
    & g \in \pazocal{G}(L, \overline{r} + N + 1, N, M, T), \boldsymbol{\Theta} \in \R^{p \times N}, \norm{\boldsymbol{\Theta}}_\max \leq T\}
\end{align*}
and
\begin{align*}
    \pazocal{F}_{s, \kappa} = \Big\{ m \in \pazocal{F}_s: \sum_{i, j} \psi_\tau(\boldsymbol{\Theta}_{i, j}) \leq \kappa \Big\}
\end{align*}
for any $\kappa > 0$. Similar to the representation of deep ReLU networks, for any $m \in \pazocal{F}_s$, we can write
\begin{align*}
    m(\boldsymbol{x}) = \pazocal{L}_{L + 1} \circ \bar{\sigma} \circ \pazocal{L}_{L} \circ \bar{\sigma} \circ \cdots \circ \pazocal{L}_2 \circ \bar{\sigma} \circ \pazocal{L}_1 \circ \phi \circ \ca L_0(\boldsymbol{x}, s(\boldsymbol{x})).
\end{align*}
Here, $\phi: \mathbb{R}^{\overline{r} + N + 1} \to \mathbb{R}^{\overline{r} + N + 1}$ satisfies
\begin{align*}
    [\phi(\boldsymbol{v})]_i = \begin{cases}
        T_M(\boldsymbol{v}_i) & \qquad \overline{r} < i \leq \overline{r} + N,\\
        \boldsymbol{v}_i & \qquad \text{otherwise,}
    \end{cases}
\end{align*} 
and $\pazocal{L}_0$ satisfies that $\pazocal{L}_0(\boldsymbol{x}, s(\boldsymbol{x})) = (p^{-1} \boldsymbol{x}^\top \boldsymbol{W}, \boldsymbol{x}^\top \boldsymbol{\Theta}, s(\boldsymbol{x}))^\top$. Additionally, for the weights $W_1$ of the neural network first layer,  define
\begin{align*}
    \pazocal{F}^0_s = \{m \in \pazocal{F}_s: \text{the last column of $\boldsymbol{W}_1$ are all zeros}\}, \quad  \pazocal{F}^0_{s, \kappa} = \Big\{m \in \pazocal{F}^0_s: \sum_{i, j} \psi_\tau(\boldsymbol{\Theta}_{i, j}) \leq \kappa \Big\}.
\end{align*}
It is obvious that $\pazocal{F}^0_s \subset \pazocal{F}_s$ and $\pazocal{F}^0_{s, \kappa} \subset \pazocal{F}_{s, \kappa}$. The definition of $\pazocal{F}^0_s$ specifies a sub-family of $\pazocal{F}_s$ where the choice of $s$ does not affect its elements. 

We abuse notation slightly to let $\pr$ denote the probability measure with respect to the target and source data, depending on the context.
Finally, we define the following two quantities of interest:
\begin{align}
    \label{eq:fast-rate-v_n}
    v_{n^Q} &= \left(N^2 L + N\overline{r} \right) \frac{L \log (T N n^Q)}{n^Q}, \\
    \label{eq:fast-rate-varrho_n}
    \varrho_{n^Q} &= \frac{\log (n p(N + \overline{r})) + L \log (T N)}{n^Q}.
\end{align}

\subsection{Technical lemmas}
Before starting the proofs fo the main results, we first list a few necessary techncial lemmas as follows.
\begin{lemma}
\label{lemma:m-lip-weight}

Let $\boldsymbol{\theta}(m) = \left\{\boldsymbol{\Theta}, \left(\boldsymbol{W}_l, \boldsymbol{b}_l \right)_{l = 1}^{L + 1} \right\}$ be the set of parameters for function $m \in \pazocal{F}_{s, \kappa}$, $\boldsymbol{\theta}(\breve{m}) = \left\{\breve{\boldsymbol{\Theta}}, (\breve{\boldsymbol{W}}_l, \breve{\boldsymbol{b}}_l)_{l = 1}^{L + 1} \right\}$ be the set of parameters for function $\breve{m} \in \pazocal{F}_{s, \kappa}$. Define the parameter distance with respect to the neural network as follows
\begin{align*}
    d(\boldsymbol{\theta}(m), \boldsymbol{\theta}(\breve{m})) = \max_{1 \leq l \leq L + 1} \left( \norm{\boldsymbol{b}_l - \breve{\boldsymbol{b}}_{l}}_\infty \lor \norm{\boldsymbol{W}_l - \breve{\boldsymbol{W}}_l}_{\max} \right).
\end{align*}
If $M \ge 1$, then the following holds
\begin{equation}
    \label{eq:lemma:m-lip-weight:result}
    \begin{split}
        \norm{m(\boldsymbol{x}) - \tilde{m}(\boldsymbol{x})}_{\infty, [-K, K]^p} & \leq (M \lor K\norm{W}_{\max}) (L + 1)T^{L}(N + 1)^{L + 1} d(\boldsymbol{\theta}(m), \boldsymbol{\theta}(\breve{m})) \\
        & + K T^{L + 1}N^L (N + \overline{r} + 1) p \norm{\boldsymbol{\theta} -\breve{\boldsymbol{\theta}}}_{\max}.
    \end{split}
\end{equation}
The following three technical lemmas are restated from \cite{fan2024factor}, covering results on the \(\epsilon\)-net covering number and the empirical process. Their proofs are the same as in \cite{fan2024factor}, with only minor changes to the exact values of universal constants, so we omit them here. These changes come from the small difference between \eqref{eq:lemma:m-lip-weight:result} and Lemma 8 of \cite{fan2024factor}, as we add \(s(\cdot)\) as an extra feature entry. Refer to \cite{fan2024factor} for the full proof and further illustrations.
\end{lemma}

\begin{lemma}[Restatement of Lemma 7 of \cite{fan2024factor}]
\label{lemma:g-infty-cover-number}

There exists a universal constant $c_1$ such that for any $K > 0, \ \delta > 2 \tau K T^{L + 1} N^{L + 1} (N + \overline{r} + 1) p$ and $N, L \ge 2$, 
\begin{align*}
    \log \pazocal{N}(\delta, \pazocal{F}_{s, \kappa}, \norm{\cdot}_{\infty, [-K, K]^p}) \le c_1 \Bigg\{ & \left( N^2 L + N \overline{r} \right) \left[ L \log B N + \log \left( \frac{M \lor K \norm{\boldsymbol{W}}_{\max}}{\delta} \lor 1 \right)\right] \\
    + & s \left[L \log (B N) + \log p + \log \left( \frac{K(N + \overline{r})}{\delta} \lor 1 \right) \right] \Bigg\}.
\end{align*}
\end{lemma}

\begin{lemma}[Restatement of Lemma 9 of \cite{fan2024factor}]
\label{lemma:equivalence-population-empirical-l2-regularized}

Suppose $\widetilde{m}$ is a fixed function in $\pazocal{F}_s$. Under the conditions of Lemma \ref{lemma:fast}, there exists some universal constants $c_1$ - $c_3$ such the event \begin{align*}
    \pazocal{C}_t = \Big\{ \forall m(\boldsymbol{x}; \boldsymbol{W}^Q, g, \boldsymbol{\Theta}, s)\in \pazocal{F}_s, ~~ \frac{1}{2} \norm{m - \tilde{m}}^2_2 \leq & \norm{m - \tilde{m}}_{n^Q}^2 \\
    & + 2 \lambda \sum_{i, j} \psi_\tau ( \boldsymbol{\Theta}_{ij} ) + c_1 \left(v_{n^Q} + \rho_{n^Q} + \frac{t}{n^Q} \right) \Big\}
\end{align*} 
satisfies $\mathbb{P} \left[ \pazocal{C}_{t} \right] \geq 1 - e^{-t}$ for any $t > 0$ as long as $\lambda \geq c_2 \varrho_{n^Q}$ and $\tau^{-1} \geq c_3 (r + 1) b (T N)^{L + 1} (N + \overline{r}) p n^Q$.
\end{lemma}

\begin{lemma}[Restatement of Lemma 10 of \cite{fan2024factor}]
\label{lemma:weighted-empirical-process-regularized}

Suppose $\tilde{m}$ is a fixed function in $\pazocal{F}_s$. Under the conditions of Lemma \ref{lemma:fast}, there exists some universal constants $c_1$ - $c_2$ such that for any fixed $\epsilon \in (0, 1)$, the event 
\begin{align*}
    \pazocal{B}_{t, \epsilon} = \Big\{ \forall m(\boldsymbol{x}; \boldsymbol{W}^Q, g, \boldsymbol{\Theta}, s) \in \pazocal{F}_s,\ & \frac{4}{n^Q} \sum_{i = 1}^{n^Q} \varepsilon_i (m(\boldsymbol{x}^Q_i) - \tilde{m}(\boldsymbol{x}^Q_i)) - \lambda \sum_{i, j} \psi_\tau({\boldsymbol{\Theta}}_{ij}) \\
    \leq & \epsilon \norm{m - \tilde{m}}_{n^Q}^2 + \frac{c_1}{\epsilon} \left(v_{n^Q} + \varrho_{n^Q} + \frac{t}{n^Q} \right) \Big\}
\end{align*} 
occurs with probability at least $1 - e^{-t}$ for any $t > 0$ as long as $\lambda \geq \frac{c_2 \varrho_{n_Q}}{\epsilon}$ and $\tau^{-1} \geq 4(r + 1) b (T N)^{L + 1} (N + \overline{r}) p n^Q$.
\end{lemma}

\subsection{The non-parametric regression error}
In this section, we aim to present a general error bound for non-parametric regression. The generality lies in its independence from the fine-tuning procedure. Since fine-tuning is performed in two stages in Section \ref{sec:fineTuning}, with each stage using independent training data, we can first fix the pretrained signal $\widehat{s}$ to be used as input and focus on deriving the error bound for the remaining estimation steps.

Fixing the pretrained model output plays a crucial role in the proof. The following lemma provides an oracle-type inequality for factor-augmented non-parametric regression given any pretrained model $s(\cdot)$, which serves as the key technical result underlying Theorems \ref{thm:fineTuning} and \ref{thm:fineTuning2}. This result builds on the established oracle-type inequalities for factor-augmented estimators \cite{fan2024factor}, extending them by incorporating an additional pretrained term $s(\cdot)$.

The primary challenge lies in decoupling the source and target errors when the pretrained function $s$ is arbitrarily fixed. Notably, due to Lipschitz continuity, the target approximation error depends on the $L_2$ norm of $s$, rather than its $L_\infty$ norm. The proof builds on a careful analysis of the approximation–estimation tradeoff within a $L_1$-clipped-norm-regularized model class.

\begin{lemma}[Oracle-type bound for non-parametric regression]
\label{lemma:fast}

Suppose that Assumptions \ref{ass:boundedness}, \ref{ass:weakDependence} and \ref{ass:subGaussianNoiseAndBoundedFunction} holds. Given any $s(\cdot): \R^p \rightarrow [-M, M]$, we consider the factor augmented model as
\begin{equation}
\label{eq:FTFAST6}
    \widehat{h}(\cdot), \widehat{\boldsymbol{\Theta}} = \argmin_{h \in \pazocal{G}(L, r + N + 1, N, M, T), \boldsymbol{\Theta} \in \mathbb{R}^{p \times N}} \frac{1}{n^Q} \sum_{i = 1}^{n^Q} \Big(y_i - h( \Big[ \widetilde{\boldsymbol{f}}^Q_i, \mathrm{trun}_M( \boldsymbol{\Theta}^\top \boldsymbol{x}^Q_i), s(\boldsymbol{x}_i^Q) \Big] ) \Big)^2 + \lambda \sum_{i, j} \psi_\tau(\boldsymbol{\Theta}_{i, j})
\end{equation}
with
\begin{equation}
\label{eq:FTFAST7}
    \widehat{m}(\boldsymbol{x}) = \widehat{h}( \Big[p^{-1}(\boldsymbol{W}^Q)^\top \boldsymbol{x}, \mathrm{trun}_M( \widehat{\boldsymbol{\Theta}}^\top \boldsymbol{x}), s(\boldsymbol{x} ) \Big]).
\end{equation}
Furthermore, we suppose that
\begin{itemize}
    \item $N \geq 2( r + |\pazocal{J}|),$
    \item $T \geq c_1[\nu_{min}(\boldsymbol{H}^Q)]^{-1} |\pazocal{J}| (r + 1)$
    \item $\lambda \geq c_2 (n^Q)^{-1} \Big( \log ( p n^Q (N + \overline{r}) ) + L \log(T N) \Big)$
    \item $\tau^{-1} \geq c_3 (r + 1) p \Big( (T N)^{L + 1} (N + \overline{r}) n^Q \Big)$
\end{itemize}

for some universal constants $c_1$ - $c_3$ and $\overline{r} \geq r$. Define

\begin{itemize}
    \item neural network approximation errors: 
    \begin{align*}
        \delta_a &= \inf_{g \in \pazocal{G} (L - 1, r + |\pazocal{J}| + 1, N, M, T)} \sup_{\kappa \in[-M, M]} \norm{g(\boldsymbol{f}^Q, \boldsymbol{u}^Q_{\pazocal{J}}, \kappa) -h(\boldsymbol{f}^Q, \boldsymbol{u}^Q_{\pazocal{J}}, \kappa)}_\infty^2 + \norm{s(\boldsymbol{x}^Q) - g^P(\boldsymbol{x}^Q_{\pazocal{J}^P})}_2^2\\
        \delta^0_a &= \inf_{g \in \pazocal{G}(L - 1, r + |\pazocal{J} \cup \pazocal{J}^P|, N, M, T)} \norm{g(\boldsymbol{f}^Q, \boldsymbol{u}^Q_{\pazocal{J} \cup \pazocal{J}^P}) - g^Q(\boldsymbol{f}^Q, \boldsymbol{u}^Q_{\pazocal{J} \cup \pazocal{J}^P})}_\infty^2
    \end{align*}
    \item latent factor inference errors: 
    \begin{align*}
        \delta_f &= \frac{|\pazocal{J} |r \cdot \overline{r}}{p \cdot \nu^2_{min}(\boldsymbol{H}^Q)},\\
        \delta^0_f &= \frac{|\pazocal{J} \cup \pazocal{J}^P| r \cdot \overline{r}}{p \cdot \nu^2_{min}(\boldsymbol{H}^Q)}
    \end{align*}
    \item stochastic errors: 
    \begin{align*}
        \delta_s &= \frac{(N^2 L + N \overline{r}) L \log(T N n^Q)}{n^Q} + \lambda|\pazocal{J}|,\\
        \delta^0_s &=  \frac{(N^2 L + N \overline{r}) L \log(T N n^Q)}{n^Q} + \lambda|\pazocal{J} \cup \pazocal{J}^P|
    \end{align*}
\end{itemize}
Then, with probability at least $1 - 3e^{-t}$ with respect to the target data, for $n^Q$ large enough, we have
\begin{align*}
    \pazocal{E}^Q (\widehat{m}) + \widehat{\pazocal{E}}^Q (\widehat{m}) \leq c_4 \Big\{ \Big(\delta_a  + \delta_f + \delta_s \Big) \land \Big(\delta^0_a  + \delta^0_f + \delta^0_s \Big) + \frac{t}{n^Q} \Big\}
\end{align*}
where $c_4$ is a universal constant that depends only on the constants in Assumptions \ref{ass:boundedness}, \ref{ass:weakDependence} and \ref{ass:subGaussianNoiseAndBoundedFunction}.
\end{lemma}

A straightforward corollary of Lemma \ref{lemma:fast} is the case where the original $s = g^P = 0$, $\pazocal{J}^P = \emptyset$, and $h$ in the lemma's context is replaced by $g^P$, representing the scenario without the extra feature entry $s(\cdot)$. To facilitate its application to the proof of Theorem \ref{thm:fineTuning}, we convert the non-parametric regression on the target data to that on the source data. This requires no additional effort, given the generality of the result for any data sample.

\begin{corollary}
\label{col:fast}

Suppose that Assumptions \ref{ass:boundedness}, \ref{ass:weakDependence} and \ref{ass:subGaussianNoiseAndBoundedFunction} holds. Suppose that
\begin{itemize}
    \item $N^P \geq 2(r + |\pazocal{J}^P|),$
    \item $T^P \geq c_1[\nu_{min}(\boldsymbol{H}^P)]^{-1} |\pazocal{J}^P| (r + 1)$
    \item $\lambda \geq c_2 (n^P)^{-1} \Big( \log( p n^P (N^P + \overline{r})) + L^P \log(T^P N^P) \Big)$
    \item $\tau^{-1} \geq c_3 (r + 1) p \Big( (T^P N^P)^{L^P + 1}(N^P + \overline{r}) n^P \Big)$
\end{itemize}
for some universal constants $c_1$ - $c_3$ and $\overline{r} \geq r$. Then, with probability at least $1 - 3e^{-t}$ with respect to the source data, for $n^P$ large enough, the pretrained model $\widehat{s}(\boldsymbol{x})$ satisfies
\begin{align*}
    \pazocal{E}_P(\widehat{s}) + \widehat{\pazocal{E}}_P(\widehat{s}) \leq c_4 \Big\{ \Delta^P + \frac{t}{n^P} \Big\},
\end{align*}
where $\Delta^P$ is defined in Theorem \ref{thm:fineTuning} and $c_4$ is a universal constant that depends only on the constants in Assumptions \ref{ass:boundedness}, \ref{ass:weakDependence} and \ref{ass:subGaussianNoiseAndBoundedFunction}.
\end{corollary}

\subsection{Proof of Theorem \ref{thm:fineTuning}}
\begin{proof}
Define the following two quantities where $\hat s$ is plugged in compared to $\delta_a$,
\begin{align*}
    \widehat{\delta}_a &= \inf_{g \in \pazocal{G}(L - 1, r + |\pazocal{J}| + 1, N, M, T)} \sup_{\kappa \in[-M, M]} \norm{g(\boldsymbol{f}^Q, \boldsymbol{u}^Q_{\pazocal{J}}, \kappa) - h(\boldsymbol{f}^Q, \boldsymbol{u}^Q_{\pazocal{J}}, \kappa)}_\infty^2 + \norm{\widehat{s}(\boldsymbol{x}^Q) - g^P(\boldsymbol{x}^Q_{\pazocal{J}^P})}_2^2,\\
    \widehat{\Delta} &= \widehat{\delta}_a + \delta_f^h + \delta_s^h,
\end{align*}
and the following two events
\begin{align*}
    \pazocal{A} &= \left\{ \pazocal{E}^Q (\widehat{m}_{FT}) + \widehat{\pazocal{E}}^Q (\widehat{m}_{FT}) \leq c_4 \Big \{\Delta^Q \land \widehat{\Delta} + \frac{t}{n^Q} \Big\} \right\},\\
    \pazocal{B} &= \left\{ \pazocal{E}^P (\widehat{s}) + \widehat{\pazocal{E}}^P (\widehat{s}) \leq c_4 \Big\{ \Delta^P + \frac{t}{n^P} \Big\} \right\}.
\end{align*}
By Lemma \ref{lemma:fast} and Corollary \ref{col:fast}, we see that each of $\ca A$ and $\ca B$ holds with probability at least $1-3e^{-t}$. Hence, the event $\ca A\cap \ca B$ holds with probability at least $1-6e^{-t}$.

Note that by Assumption \ref{ass:absCon}, $\frac{d\mu_{\boldsymbol{x}^Q}}{d\mu_{\boldsymbol{x}^P}}(\boldsymbol{x}) \leq c_1$. Then, as $\pazocal{A} \cap \pazocal{B}$ holds,
\begin{align*}
    \norm{\widehat{s}(\boldsymbol{x}^Q) - g^P (\boldsymbol{x}^Q_{\pazocal{J}^P})}_2^2 &= \int |\widehat{s}(\boldsymbol{x}^Q) - g^P (\boldsymbol{x}_{\pazocal{J}^P}^Q)|^2 d\mu_{\boldsymbol{x}^Q}\\
    &\lesssim \int |\widehat{s}(\boldsymbol{x}^P) - g^P(\boldsymbol{x}_{\pazocal{J}^P}^Q)|^2 d\mu_{\boldsymbol{x}^P}\\
    &\lesssim \Delta^P + \frac{t}{n^P}.
\end{align*}
Since
\begin{align*}
    \delta_a^h = \inf_{g \in \pazocal{G}(L - 1, r + |\pazocal{J}| + 1, N, M, T)} \, \sup_{\kappa \in[-M, M]} \norm{g(\boldsymbol{f}^Q, \boldsymbol{u}_{\pazocal{J}}^Q, \kappa) -h(\boldsymbol{f}^Q, \boldsymbol{u}_{\pazocal{J}}^Q, \kappa)}_\infty^2,
\end{align*}
then 
\begin{align*}
    \widehat{\delta}_a \leq \norm{\widehat{s}(\boldsymbol{x}^Q) - g^P (\boldsymbol{x}^Q_{\pazocal{J}^P})}_2^2 + \delta_a^h.
\end{align*}
As a result, we have $\pazocal{A} \cap \pazocal{B}$ holds,
\begin{align*}
    \widehat{\Delta} &= \widehat{\delta}_a + \delta_f^h + \delta_s^h\\
    &\lesssim \norm{\widehat{s}(\boldsymbol{x}^Q) - g^P (\boldsymbol{x}^Q_{\pazocal{J}^P})}_2^2 + \delta_a^h + \delta_f^h + \delta_s^h\\
    &\lesssim \Delta^P + \frac{t}{n^P} + \Delta^h
\end{align*}
We also have $\Delta^h \lesssim \Delta^Q$, because by definition, 
\begin{align*}
    \delta_a^h \lesssim \delta_a^Q, \ \delta_f^h \lesssim \delta_f^Q, \ \delta_s^h \lesssim \delta_s^Q
\end{align*}
Therefore, when $\pazocal{A} \cap \pazocal{B}$ holds, we also have the following as 
\begin{align*}
    \pazocal{E}^Q (\widehat{m}_{FT}) + \widehat{\pazocal{E}}^Q (\widehat{m}_{FT}) &\lesssim c_4 \Big\{ \Delta^Q \land \widehat{\Delta} + \frac{t}{n^Q} \Big\}\\
    &\lesssim c_4 \Big\{ \Delta^Q \land \Big(\Delta^P + \frac{t}{n^P} + \Delta^h \Big) + \frac{t}{n^Q} \Big\}\\
    &\lesssim c_4 \Big\{ \Delta^Q \land \Big(\Delta^P + \frac{t}{n^P} \Big) + \Delta^h + \frac{t}{n^Q} \Big\}\\
    &\lesssim c_4 \Big\{ \Delta^Q \land \Delta^P + \frac{t}{n^P} + \Delta^h + \frac{t}{n^Q} \Big\}\\
    &\leq 2 c_4 \Big\{ \Delta^Q \land \Delta^P + \Delta^h + \frac{t}{n^P \wedge n^Q} \Big\}
\end{align*}
which completes the proof.
\end{proof}

\subsection{Proof of Theorem \ref{thm:fineTuning2}}
\begin{proof}
We treat $K$ as a parameter throughout the proof.
By the problem setting $r \lor| \pazocal{J}^P| \leq c_1$ and Assumption \ref{ass:boundedness}, it further implies
\begin{align*}
    \norm{\boldsymbol{B}^P_{\pazocal{J}^P}}_2 \leq \norm{\boldsymbol{B}^P_{\pazocal{J}^P}}_F \leq c_1 \sqrt{ |\pazocal{J}^P| r} \leq c_1^2.
\end{align*}
Hence, by Assumption \ref{ass:transferability} and  Proposition \ref{prop:transferabilityBenefit}, we have
\begin{equation}
    \label{eq:dcmFamilies}
    \begin{split}
        g^P(\boldsymbol f,\boldsymbol u_{\pazocal{J}^P}) &\in \pazocal{H}(r+|\pazocal{J}^P|, l^P, \pazocal{P}^P, c_0)\\
        % g^P([\boldsymbol{B}^P_{\pazocal{J}^P, :}] \boldsymbol{f} + \boldsymbol{u}_{\pazocal{J}^P}) &\in \pazocal{H}(r + |\pazocal{J}^P|, l^P + 1, \pazocal{P}^P \cup \{\infty, r + |\pazocal{J}^P| \}, c_0 \lor c_1^2),\\
        g^Q(\boldsymbol{f}, \boldsymbol{u}_{\pazocal{J} \cup \pazocal{J}^P}) &\in \pazocal{H}(r + |\pazocal{J} \cup \pazocal{J}^P|, l + l^P + 1, \pazocal{P} \cup \pazocal{P}^P \cup \{\infty, r + |\pazocal{J}^P|\}, c_0 \lor c_1^2),\\
        h &\in \pazocal{H}(r + |\pazocal{J}| + 1, l, \pazocal{P}, c_0),\\
        \gamma(\pazocal{P}) &= \gamma, \quad \gamma(\pazocal{P}^P \cup \{\infty, r + |\pazocal{J}^P|) = \gamma^P, \quad \gamma(\pazocal{P} \cup \pazocal{P}^P \cup \{\infty, r + |\pazocal{J}^P|\}) = \gamma \land \gamma^P.
    \end{split}
\end{equation}
Applying Theorem 4 of \cite{fan2024factor} to \eqref{eq:dcmFamilies}, we have
\begin{align*}
    \delta_a^P \lesssim (N^P)^{-4 \gamma^P}, \quad \delta_a^Q \lesssim N^{-4 \gamma \land \gamma^P}, \quad \delta_a^h \lesssim N^{-4 \gamma}. 
\end{align*}
By plugging in our choice of hyperparameters in Condition \ref{cond:hyper}, we obtain
\begin{equation}
    \label{eq:DeltaValue}
    \begin{split}
        \Delta^P &\lesssim \Big(\frac{\log n^P}{n^P} \Big)^{\frac{2 \gamma^P}{2 \gamma^P + 1}} + \frac{\log n^P}{n^P} + \frac{\log p}{n^P} + \frac{1 \land r}{\nu^2_\min(\boldsymbol{H}^P) p}\\
        &\lesssim \Big(\frac{\log n^P}{n^P} \Big)^{\frac{2 \gamma^P}{2 \gamma^P + 1}} +  \frac{1 \land r}{\nu^2_\min(\boldsymbol{H}^P) p}+\frac{\log p}{n^P},\\
        \Delta^Q &\lesssim \Big( \frac{\log n^Q}{n^Q} \Big)^{\frac{2 (\gamma \land \gamma^P)}{2 \Gamma + 1}} + \Big(\frac{\log n^Q}{n^Q} \Big)^{\frac{2 \Gamma}{2 \Gamma + 1}} + \frac{\log n^Q}{n^Q} + \frac{\log p}{n^Q} + \frac{1 \land r}{\nu^2_\min(\boldsymbol{H}^Q) p}\\
        &\lesssim \Big(\frac{\log n^Q}{n^Q} \Big)^{\frac{2 (\gamma \land \gamma^P)}{2 \Gamma + 1}} + \Big( \frac{\log n^Q}{n^Q} \Big)^{\frac{2 \Gamma}{2 \Gamma + 1}} + \frac{\log p}{n^Q} + \frac{1 \land r}{\nu^2_\min(\boldsymbol{H}^Q) p}, \\
        \Delta^h &\lesssim \Big( \frac{\log n^Q}{n^Q} \Big)^{\frac{2 \gamma}{2 \Gamma + 1}} + \Big(\frac{\log n^Q}{n^Q} \Big)^{\frac{2 \Gamma}{2 \Gamma + 1}} + \frac{1 \land r}{\nu^2_\min(\boldsymbol{H}^Q) p}+\frac{\log p}{n^Q}. 
    \end{split}
\end{equation}
By letting $t = K \log p$, \eqref{eq:fineTuningResult} in Theorem \ref{thm:fineTuning} becomes
\begin{align*}
    \pazocal{E}^Q (\widehat{m}_{FT}) + \widehat{\pazocal{E}}^Q (\widehat{m}_{FT}) \leq & c_4 \Big\{ \Delta^Q \land \Big( \Delta^P + \frac{K \log p}{n^P} \Big) + \Delta^h + \frac{K \log p}{n^Q} \Big\}\\
    \lesssim & K \Big\{ \Big(\frac{\log n^Q}{n^Q} \Big)^{\frac{2 (\gamma \land \gamma^P)}{2 \Gamma + 1}} \land \Big( \frac{\log n^P}{n^P} \Big)^{\frac{2 \gamma^P}{2 \gamma^P + 1}} + \Big( \frac{\log n^Q}{n^Q} \Big)^{\frac{2 \gamma}{2 \Gamma + 1}} + \Big( \frac{\log n^Q}{n^Q} \Big)^{\frac{2 \Gamma}{2 \Gamma + 1}}\\
    &+ \frac{\log p}{n^Q} + \frac{1 \land r}{[\nu^2_\min(\boldsymbol{H}^P) \land \nu^2_\min(\boldsymbol{H}^Q)] p} \Big\}
\end{align*}
with probability$1 - 6p^{-K}$ with respect to the target and source data.

It remains to show that $A \lesssim B$, where
\begin{align*}
   A &\overset{def}{=} \Big( \frac{\log n^Q}{n^Q} \Big)^{\frac{2 (\gamma \land \gamma^P)}{2 \Gamma + 1}} \land \Big( \frac{\log n^P}{n^P} \Big)^{\frac{2 \gamma^P}{2 \gamma^P + 1}} + \Big( \frac{\log n^Q}{n^Q} \Big)^{\frac{2 \gamma}{2 \Gamma + 1}} + \Big( \frac{\log n^Q}{n^Q} \Big)^{\frac{2 \Gamma}{2 \Gamma + 1}},\\
   B &\overset{def}{=} \Big[ \frac{\log (n^P + n^Q)}{n^P + n^Q} \Big]^{\frac{2 \gamma^P}{2 \gamma^P + 1}} + \Big( \frac{\log n^Q}{n^Q} \Big)^{\frac{2 \gamma}{2 \gamma + 1}}. 
\end{align*}
We divide its proof into three cases as follows.
\paragraph{Case 1: $\gamma\leq\gamma^P$.} In this case, we have $\gamma \land \gamma^P = \Gamma = \gamma$, which further implies
\begin{align*}
    A &\lesssim \Big( \frac{\log n^Q}{n^Q} \Big)^{\frac{2 \gamma}{2 \gamma + 1}} \land \Big(\frac{\log n^P}{n^P} \Big)^{\frac{2 \gamma^P}{2 \gamma^P + 1}} + \Big(\frac{\log n^Q}{n^Q} \Big)^{\frac{2 \gamma}{2 \gamma + 1}}\\
    &\lesssim \Big( \frac{\log n^Q}{n^Q} \Big)^{\frac{2 \gamma}{2 \gamma + 1}} \lesssim B.
\end{align*}

\paragraph{Case 2: $\gamma > \gamma^P, n^Q > n^P$.} In this case, we have $\gamma \land \gamma^P = \Gamma = \gamma^P$. Hence, we have
\begin{align*}
    A &\lesssim \Big( \frac{\log n^Q}{n^Q} \Big)^{\frac{2 \gamma^P}{2 \gamma^P + 1}} \land \Big(\frac{\log n^P}{n^P} \Big)^{\frac{2 \gamma^P}{2 \gamma^P + 1}} + \Big( \frac{\log n^Q}{n^Q} \Big)^{\frac{2 \gamma}{2 \gamma^P + 1}} + \Big( \frac{\log n^Q}{n^Q} \Big)^{\frac{2 \gamma^P}{2 \gamma^P + 1}}\\
    &\lesssim \Big( \frac{\log n^Q}{n^Q} \Big)^{\frac{2 \gamma^P}{2 \gamma^P + 1}} + \Big( \frac{\log n^Q}{n^Q} \Big)^{\frac{2 \gamma}{2\gamma^P + 1}}\\
    &\lesssim \Big( \frac{\log n^Q}{n^Q} \Big)^{\frac{2 \gamma^P}{2 \gamma^P + 1}}\\
    &\lesssim \Big[ \frac{\log (2n^Q)}{2n^Q} \Big]^{\frac{2 \gamma^P}{2 \gamma^P + 1}}\\
    &\lesssim \Big[ \frac{\log (n^P + n^Q)}{n^P + n^Q} \Big]^{\frac{2 \gamma^P}{2 \gamma^P + 1}} \lesssim B.
\end{align*}

\paragraph{Case 3: $\gamma > \gamma^P, n^Q \leq n^P$.} In this case, we have $\gamma \land \gamma^P = \gamma^P, \Gamma = \gamma$, which implies
\begin{align*}
    A &\lesssim \Big( \frac{\log n^Q}{n^Q} \Big)^{\frac{2 \gamma^P}{2 \gamma + 1}} \land \Big( \frac{\log n^P}{n^P} \Big)^{\frac{2 \gamma^P}{2 \gamma^P + 1}} + \Big( \frac{\log n^Q}{n^Q} \Big)^{\frac{2 \gamma}{2 \gamma + 1}}\\
    &\lesssim \Big( \frac{\log n^P}{n^P} \Big)^{\frac{2 \gamma^P}{2 \gamma^P + 1}} + \Big( \frac{\log n^Q}{n^Q} \Big)^{\frac{2 \gamma}{2 \gamma + 1}} \lesssim B.
\end{align*}

Combining Cases 1-3 above completes the proof.
\end{proof}

\subsection{Proof of Theorem \ref{thm:minimaxLower}}
The proof separates regimes—source-limited and target-limited settings, and derives corresponding nonparametric lower bounds for each case.

\begin{proof}
We always assume that Assumption \ref{ass:transferability} holds throughout the proof.
By setting $g^P \equiv 0$, $g^Q(\boldsymbol{f}, \boldsymbol{u}_{\pazocal{J}}) = h$ belongs to the family $\pazocal{H}(r + |\pazocal{J}|, l, \pazocal{P}, c_0)$, where $c_0$ is the universal constant defined in Assumption \ref{ass:transferability}. Thus, Lemma 1 of \cite{fan2024factor} shows that
\begin{align*}
    \inf_{\widehat{m}}\sup_{\substack{\mu_{(\boldsymbol{f}^Q, \boldsymbol{u}^Q, \boldsymbol{x}^Q)} = \mu_{(\boldsymbol{f}^P, \boldsymbol{u}^P, \boldsymbol{x}^P)} \in \pazocal{P}(p, r, \rho) \\ \text{Assumption \ref{ass:transferability} holds with $|\pazocal{J}| = 1$}}} \pazocal{E}^Q (\widehat{m}) \geq \inf_{\widehat{m}} \sup_{\substack{\mu_{(\boldsymbol{f}^Q, \boldsymbol{u}^Q, \boldsymbol{x}^Q)} \in \pazocal{P}(p, r, \rho)\\
    g^Q(\boldsymbol{f}, \boldsymbol{u}_{\pazocal{J}}) \in \pazocal{H} (r + 1, l, \pazocal{P}, c_0), \ |\pazocal{J}| = 1} } \pazocal{E}^Q (\widehat{m}) \gtrsim \frac{1}{\rho}.
\end{align*}
Therefore, it suffices to show that
\begin{equation}
    \label{eq:lowerBoundGoal}
    \inf_{\widehat{m}} \sup_{\substack{\mu_{(\boldsymbol{f}^Q, \boldsymbol{u}^Q, \boldsymbol{x}^Q)} = \mu_{(\boldsymbol{f}^P, \boldsymbol{u}^P, \boldsymbol{x}^P)} \in \pazocal{P}(p, r, \rho) \\ \text{Assumption \ref{ass:transferability} holds with $|\pazocal{J}| = 1$}}} \pazocal{E}^Q (\widehat{m}) \gtrsim (n^P + n^Q)^{-\frac{2 \gamma^P}{2 \gamma^P + 1}} + (n^Q)^{-\frac{2 \gamma}{2 \gamma + 1}} + \frac{\log p}{n^Q}.
\end{equation}
We divide the proof of \eqref{eq:lowerBoundGoal} into the following two cases:

\paragraph{Case 1: $g^P \equiv 0$.} In this case, the source data has no informative power for $g^Q$, and we have
\begin{align*}
    \inf_{\widehat{m}} \sup_{\substack{\mu_{(\boldsymbol{f}^Q, \boldsymbol{u}^Q, \boldsymbol{x}^Q)} = \mu_{(\boldsymbol{f}^P, \boldsymbol{u}^P, \boldsymbol{x}^P)} \in \pazocal{P}(p, r, \rho) \\ \text{Assumption \ref{ass:transferability} holds with $|\pazocal{J}| = 1$}}} \pazocal{E}^Q(\widehat{m}) &\geq \inf_{\widehat{m}} \sup_{\substack{\mu_{(\boldsymbol{f}^Q, \boldsymbol{u}^Q, \boldsymbol{x}^Q)} \in \pazocal{P}(p, r, \rho)\\
    g^Q(\boldsymbol{f}, \boldsymbol{u}_{\pazocal{J}}) \in \pazocal{H} (r + 1, l, \pazocal{P}, c_0), \ |\pazocal{J}| = 1}} \pazocal{E}^Q(\widehat{m}(\boldsymbol{x}^Q)) \\ 
    &\geq \inf_{\breve{m}} \sup_{\substack{\mu_{(\boldsymbol{f}^Q, \boldsymbol{u}^Q, \boldsymbol{x}^Q)} \in \pazocal{P}(p, r, \rho)\\
    g^Q(\boldsymbol{f}, \boldsymbol{u}_{\pazocal{J}}) \in \pazocal{H}(r + 1, l, \pazocal{P}, c_0), \ |\pazocal{J}| = 1}} \pazocal{E}^Q(\breve{m}(\boldsymbol{f}^Q, \boldsymbol{u}^Q)),
\end{align*}
where the last inequality holds because, when constructing $\breve{m}$, we have additional access to the latent factor structures $(\boldsymbol{f}^Q, \boldsymbol{u}^Q)$ and the factor loading matrix $\boldsymbol{B}^Q$, which provides strictly more information than only knowing $\boldsymbol{x}^Q$ used to construct $\widehat{m}$.

From the theorem setting, we can find $(\beta^h, d^h) \in \pazocal{P}$ such that $\gamma = \frac{\beta^h}{d^h}$ and $d^h \leq r + 1$. Therefore, it follows from the well-known minimax optimal lower bound result in Theorem 3.2 of \cite{gyorfi2002distribution} for the family $\pazocal{F}_{d^h, \beta^h, c_0} \subset \pazocal{H}(r + 1, l, \pazocal{P}, c_0)$ that
\begin{equation}
    \label{eq:lower1}
    \inf_{\breve{m}} \sup_{\substack{\mu_{(\boldsymbol{f}^Q, \boldsymbol{u}^Q, \boldsymbol{x}^Q)} \in \pazocal{P}(p, r, \rho)\\
    g^Q(\boldsymbol{f}, \boldsymbol{u}_{\pazocal{J}}) \in \pazocal{H} (r + 1, l, \pazocal{P}, c_0), \ |\pazocal{J}| = 1}} \pazocal{E}^Q( \breve{m}(\boldsymbol{f}^Q, \boldsymbol{u}^Q)) \geq \inf_{\breve{m}} \sup_{\substack{\mu_{\boldsymbol{x}} \sim \mathrm{Uni}[-1, 1]^{d^h}\\
    g^Q(\boldsymbol{x}) \in \pazocal{F}(d^h, \beta^h, c_0)}} \pazocal{E}^Q (\breve{m}(\boldsymbol{x})) \gtrsim (n^Q)^{-\frac{2 \gamma}{2 \gamma + 1}},
\end{equation}
where $\mathrm{Uni}[-1, 1]^{d}$ is the uniform distribution over $[-1, 1]^{d}$ for any $d \geq 1$.

Define the quantity $\delta := \sqrt{\frac{\log p}{2 n^Q}}$. For any $j \in [p]$, let $\pr_j$ be the probability measure under which $\boldsymbol{f}^Q \sim \mathrm{Uni}[-1, 1]^{r}$ and $\boldsymbol{u}^Q \sim \mathrm{Uni}[-1, 1]^{p}$ are independent random variables, and
\begin{align*}
    y^Q = m^{(j)}(\boldsymbol{u}^Q) + \epsilon^Q, \quad m^{(j)}(\boldsymbol{u}) := \delta \boldsymbol{u}_{(j)} \mathbf{1} \{j \geq 1\} \in \pazocal{H} (l, r + 1, \pazocal{P}),
\end{align*}
where $\boldsymbol{u}_{(j)}$ is the $j$-entry of $\boldsymbol{u} \in \R^p$. It follows from the $KL$-divergence of two Gaussian random variables that
\begin{align*}
    KL(\pr_0 \Vert \pr_j) = \frac{1}{2} n^Q \EE|m^{(0)} - m^{(j)}(\boldsymbol{u})|^2 \leq \frac{1}{6} n^Q \delta^2 \Longrightarrow \frac{1}{p} \sum_{j = 1}^p KL(\pr_0 \Vert \pr_j) \leq \frac{1}{6} n^Q \cdot \frac{\log p}{2 n^Q}< \frac{1}{8} \log p.
\end{align*}
Therefore, for any $j \neq j'$, we have
\begin{align*}
    \norm{m^{(j)} - m^{(j')}} = \sqrt{\EE|m^{(j)} - m^{(j')}|^2} = \sqrt{\frac{2}{3}} \delta \geq \frac{\log p}{3 n^Q}.
\end{align*}
By Theorem 2.7 of \cite{tsybakov2008introduction} where $M$ is replaced by $p$ that
\begin{equation}
    \label{eq:lower2}
    \inf_{\breve{m}} \sup_{\substack{\mu_{(\boldsymbol{f}^Q, \boldsymbol{u}^Q, \boldsymbol{x}^Q)} \in \pazocal{P}(p, r, \rho)\\ 
    g^Q(\boldsymbol{f}, \boldsymbol{u}_{\pazocal{J}}) \in \pazocal{H} (r + 1, l, \pazocal{P}, c_0), \ |\pazocal{J}| = 1}} \pazocal{E}^Q (\breve{m}(\boldsymbol{f}^Q, \boldsymbol{u}^Q)) \gtrsim \delta^2 \gtrsim \frac{\log p}{n^Q}.
\end{equation}
Combining \eqref{eq:lower1} and \eqref{eq:lower2}, we have
\begin{equation}
    \label{eq:lower3}
    \inf_{\widehat{m}} \sup_{\substack{\mu_{(\boldsymbol{f}^Q, \boldsymbol{u}^Q, \boldsymbol{x}^Q)} = \mu_{(\boldsymbol{f}^P, \boldsymbol{u}^P, \boldsymbol{x}^P)} \in \pazocal{P}(p, r, \rho) \\ \text{Assumption \ref{ass:transferability} holds with $|\pazocal{J}| = 1$}}} \pazocal{E}^Q (\widehat{m}) \gtrsim (n^Q)^{-\frac{2 \gamma}{2 \gamma + 1}} + \frac{\log p}{n^Q}.
\end{equation}

\paragraph{Case 2: $h(\boldsymbol{f}, \boldsymbol{u}_{\pazocal{J}}, s) \equiv s$.} In this case, it holds that $g^Q \equiv g^P$, and $(\boldsymbol{x}^Q, \boldsymbol{y}^Q) \sim (\boldsymbol{x}^P, \boldsymbol{y}^P)$ when $\mu_{(\boldsymbol{f}^Q, \boldsymbol{u}^Q, \boldsymbol{x}^Q)} = \mu_{(\boldsymbol{f}^P, \boldsymbol{u}^P, \boldsymbol{x}^P)} \in \pazocal{P}(p, r, \rho)$. 
Note that
\begin{align*}
    \inf_{\widehat{m}} \sup_{\substack{\mu_{(\boldsymbol{f}^Q, \boldsymbol{u}^Q, \boldsymbol{x}^Q)} = \mu_{(\boldsymbol{f}^P, \boldsymbol{u}^P, \boldsymbol{x}^P)} \in \pazocal{P}(p, r, \rho) \\ \text{Assumption \ref{ass:transferability} holds with $|\pazocal{J}| = 1$}}} \pazocal{E}^Q (\widehat{m}) &\geq \inf_{\widehat{m}} \sup_{\substack{\mu_{(\boldsymbol{f}^Q, \boldsymbol{u}^Q, \boldsymbol{x}^Q)} = \mu_{(\boldsymbol{f}^P, \boldsymbol{u}^P, \boldsymbol{x}^P)} \in \pazocal{P}(p, r, \rho) \\ g^Q = g^P, \ g^P(\boldsymbol{x}_{\pazocal{J}^P}) \in \pazocal{H} (|\pazocal{J}^P|, l^P, \pazocal{P}^P, c_0)}} \pazocal{E}^Q (\widehat{m}) \\ 
    &\geq \inf_{\breve{m}} \sup_{\substack{\mu_{(\boldsymbol{f}^Q, \boldsymbol{u}^Q, \boldsymbol{x}^Q)} = \mu_{(\boldsymbol{f}^P, \boldsymbol{u}^P, \boldsymbol{x}^P)} \in \pazocal{P}(p, r, \rho) \\ g^Q = g^P, \ g^P(\boldsymbol{x}_{\pazocal{J}^P}) \in \pazocal{H} (|\pazocal{J}^P|, l^P, \pazocal{P}^P, c_0)}} \pazocal{E}^Q (\breve{m}(\boldsymbol{f}^Q, \boldsymbol{u}^Q)) \\ 
    &\geq \inf_{\breve{m}} \sup_{\substack{\mu_{(\boldsymbol{f}^Q, \boldsymbol{u}^Q, \boldsymbol{x}^Q)} = \mu_{(\boldsymbol{f}^P, \boldsymbol{u}^P, \boldsymbol{x}^P)} \in \pazocal{P}(p, r, \rho) \\ g^Q = g^P,\ g^P(\boldsymbol{x}_{\pazocal{J}^P}) \in \pazocal{H} (|\pazocal{J}^P|, l^P, \pazocal{P}^P, c_0) \\ \boldsymbol{B}^Q = 0, \boldsymbol{f}^Q \equiv 0}} \pazocal{E}^Q (\breve{m}(\boldsymbol{f}^Q, \boldsymbol{u}^Q)) \\ 
    &\geq \inf_{\breve{m}} \sup_{\substack{\boldsymbol{x}^Q, \boldsymbol{x}^P \sim \mathrm{Uni}[-1, 1]^p \\ g^Q = g^P,\ g^P(\boldsymbol{x}_{\pazocal{J}^P}) \in \pazocal{H} (|\pazocal{J}^P|, l^P, \pazocal{P}^P, c_0) \\ }} \pazocal{E}^Q (\breve{m}(\boldsymbol{x}^Q))
\end{align*}
where the second inequality holds due to the same reasoning applied in Case 1. 

From the theorem setting, we can find $(\beta^P, d^P) \in \pazocal{P}^P$ such that $\gamma^P = \frac{\beta^P}{d^P}$ and $d^P \leq r + 1$. Applying Theorem 3.2 of \cite{gyorfi2002distribution} again for the family $\pazocal{F}_{d^P, \beta^P, c_0}$, we see that
\begin{equation}
    \label{eq:lower4}
    \begin{split}
        \inf_{\widehat{m}} \sup_{\substack{\mu_{(\boldsymbol{f}^Q, \boldsymbol{u}^Q, \boldsymbol{x}^Q)} = \mu_{(\boldsymbol{f}^P, \boldsymbol{u}^P, \boldsymbol{x}^P)} \in \pazocal{P}(p ,r, \rho) \\ \text{Assumption \ref{ass:transferability} holds with $|\pazocal{J}| = 1$}}} \pazocal{E}^Q (\widehat{m}) &\geq \inf_{\breve{m}} \sup_{\substack{\boldsymbol{x}^Q, \boldsymbol{x}^P \sim \mathrm{Uni}[-1, 1]^p \\ g^Q = g^P, \ g^P(\boldsymbol{x}_{\pazocal{J}^P}) \in \pazocal{H} (|\pazocal{J}^P|, l^P, \pazocal{P}^P, c_0) \\ }} \pazocal{E}^Q(\breve{m}(\boldsymbol{x}^Q))\\
        &\geq \inf_{\breve{m}} \sup_{\substack{\boldsymbol{x}^Q, \boldsymbol{x}^P \sim \mathrm{Uni}[-1, 1]^p \\ g^Q = g^P, \ g^P(\boldsymbol{x}_{\pazocal{J}^P}) \in \pazocal{H} (|\pazocal{J}^P|, l^P, \pazocal{P}^P, c_0) \\ \pazocal{J}^P = [d^P]}} \pazocal{E}^Q(\breve{m}(\boldsymbol{x}^Q))\\
        &\geq \inf_{\breve{m}} \sup_{\substack{\boldsymbol{x}^Q, \boldsymbol{x}^P \sim \mathrm{Uni}[-1, 1]^p \\ g^Q(\boldsymbol{x}) = g^P(\boldsymbol{x}) \in \pazocal{F}_{d^P, \beta^P, c_0}}} \pazocal{E}^Q(\breve{m}(\boldsymbol{x}^Q))\\
        &\gtrsim (n^P + n^Q)^{-\frac{2 \gamma^P}{2 \gamma^P + 1}}.
    \end{split}
\end{equation}
Therefore, we obtain \eqref{eq:lowerBoundGoal} by combining \eqref{eq:lower3} and \eqref{eq:lower4}, thereby completing the proof.
\end{proof}
% !TEX root = ../main.tex
\section{Proofs of auxiliary results}
\label{app:aux}

\subsection{Proof of Proposition \ref{prop:transferabilityBenefit}}
\begin{proof}[Proof of Proposition \ref{prop:transferabilityBenefit}]

Given the definition of $h$, there exists $\Big\{h^{(1)}_0, \dots, h^{(t)}_0 \Big\} \in \pazocal{H}(r+|\pazocal{J}| + 1, 1, \pazocal{P}, c_0)$ and $h_1 \in \pazocal{H}(t, l - 1, \pazocal{P}, c_0)$ for some positive integer $t$ such that
\begin{align*}
    h(\boldsymbol{f}, \boldsymbol{u}_{\pazocal{J}}, \boldsymbol{z}) = h_1 \Big(h^{(1)}_0(\boldsymbol{f}, \boldsymbol{u}_{\pazocal{J}}, \boldsymbol{z}),\dots,h^{(t)}_0(\boldsymbol{f}, \boldsymbol{u}_{\pazocal{J}}, \boldsymbol{z})\Big)
\end{align*}
Recall that we treat $(\boldsymbol{f}, \boldsymbol{u}_{\pazocal{J} \cup \pazocal{J}^P})$ as the function input. Thus, each coordinate of $(\boldsymbol{f}, \boldsymbol{u}_{\pazocal{J}^P})$ (as a subset of $(\boldsymbol{f}, \boldsymbol{u}_{\pazocal{J} \cup \pazocal{J}^P})$) can be viewed as an element in $\pazocal{H}(r + |\pazocal{J} \cup \pazocal{J}^P|, l^P, \pazocal{P} \cup \pazocal{P}^P, c_0)$. 

Note that 
\begin{align*}
    [\boldsymbol{B}^P_{\pazocal{J}^P, :}]\boldsymbol{f} + u_{\pazocal{J}^P}\in \pazocal{F}_{r + |\pazocal{J}^P|, \infty, \norm{\boldsymbol{B}^P_{\pazocal{J}^P,:}}_2}
\end{align*}
by the definition of the H\"older's smoothness. If we composite a function of $g^P(\cdot)$ outside, we will have
\begin{align*}
    g^P([\boldsymbol{B}^P_{\pazocal{J}^P, :}] \boldsymbol{f} + u_{\pazocal{J}^P}) \in & \pazocal{H} (r + |\pazocal{J}^P|, l^P + 1, \pazocal{P}^P \cup \{\infty, r + |\pazocal{J}^P|\}, c_0 \lor \norm{\boldsymbol{B}^P_{\pazocal{J}^P, :}}_2)\\
    \subset & \ca{H}(r + |\pazocal{J} \cup \pazocal{J}^P|, l^P + 1, \pazocal{P} \cup \pazocal{P}^P \cup \{\infty, r + |\pazocal{J}^P|\}, c_0 \lor \norm{B^P_{\pazocal{J}^P,:}}_2).  
\end{align*}
Thus, each coordinate of $\boldsymbol{f}, \boldsymbol{u}_{\pazocal{J}}, g^P([\boldsymbol{B}^Q_{\pazocal{J}^P, :}]\boldsymbol{f} + \boldsymbol{u}_{\pazocal{J}^P})$ is an element in $\pazocal{H}(r + |\pazocal{J} \cup \pazocal{J}^P|, l^P + 1, \pazocal{P} \cup \pazocal{P}^P \cup \{\infty, r + |\pazocal{J}J^P|\}, c_0 \lor \norm{\boldsymbol{B}^P_{\pazocal{J}^P,:}}_2).$ We introduce a parametrization of $\boldsymbol{f}$ and $\boldsymbol{u}_{\pazocal{J} \cup \pazocal{J}^P}$, 
\begin{align*}
    m^{(j)} \Big(\boldsymbol{f}, \boldsymbol{u}_{\pazocal{J} \cup \pazocal{J}^P} \Big) := & h^{(j)} \Big(\boldsymbol{f}, \boldsymbol{u}_{\pazocal{J}}, g^P([B^Q_{\pazocal{J}^P, :}] \boldsymbol{f} + \boldsymbol{u}_{\pazocal{J}^P}) \Big)\\
    \in & \pazocal{H}(r + |\pazocal{J} \cup \pazocal{J}^P|, l^P + 2, \pazocal{P} \cup \pazocal{P}^P \cup \{\infty, r + |\pazocal{J}^P|\}, c_0 \lor \norm{\boldsymbol{B}^P_{\pazocal{J}^P, :}}_2).  
\end{align*}
with $j = 1, \cdots, t$. By Definition \ref{def:hcm}, we can rewrite $g^Q$ as
\begin{align*}
    g^Q\Big(\boldsymbol{f}, \boldsymbol{u}_{\pazocal{J}}, g^P([\boldsymbol{B}^Q_{\pazocal{J}^P, :}] \boldsymbol{f} + \boldsymbol{u}_{\pazocal{J}^P})\Big) = h_1 \Big(m^{(1)}(\boldsymbol{f}, \boldsymbol{u}_{\pazocal{J} \cup \pazocal{J}^P}), \dots, m^{(t)}(\boldsymbol{f}, \boldsymbol{u}_{\pazocal{J} \cup \pazocal{J}^P})\Big),
\end{align*}
Since $h_1 \in \pazocal{H}(t, l - 1, \pazocal{P}, c_0)$, and we have
\begin{align*}
    g^Q\Big(\boldsymbol{f}, \boldsymbol{u}_{\pazocal{J}}, & g^P([\boldsymbol{B}^Q_{\pazocal{J}^P, :}] \boldsymbol{f} + \boldsymbol{u}_{\pazocal{J}^P})\Big) \in\\ & \pazocal{H} (r + |\pazocal{J} \cup \pazocal{J}^P|,l + l^P + 1, \pazocal{P} \cup \pazocal{P}^P \cup \{\infty, r + |\pazocal{J}^P|\}, c_0 \lor \norm{\boldsymbol{B}^P_{\pazocal{J}^P,:}}_2)
\end{align*}
The equality $\gamma(\pazocal{P} \cup \pazocal{P}^P \cup \{\infty, r + |\pazocal{J}^P|\}) = \min\{\gamma, \gamma^P\}$ is trival by the definition of $\gamma(\cdot)$.
\end{proof}

\subsection{Proof of Lemma \ref{lemma:factorTransferLemma}}
\label{app:aux-lemma-1}
\begin{proof}[Proof of Lemma \ref{lemma:factorTransferLemma}]

Define $\varepsilon^P(t) = r \sqrt{\frac{\log p + t}{n^P}} + r^2 \sqrt{\frac{\log r + t}{n^P}} + \frac{1}{\sqrt{p}}$. Firstly, from Lemma 5 of \cite{fan2024factor}, we already obtain that there exists a universal constant $C_1$ such that, by the simple union bound, with probability at least $1-6e^{-t}$, both of the following inequalities hold with $\boldsymbol{S} =  \boldsymbol{B}^Q (\boldsymbol{B}^Q)^\top$
\begin{align}
    \label{eq:fanLemma5eq1}
    \norm{\widehat{\boldsymbol{\Sigma}}^Q - \boldsymbol{S}}_F & \leq C_1 \cdot p \cdot \varepsilon^Q(t),\\
    \label{eq:fanLemma5eq2}
    \norm{\widehat{\boldsymbol{\Sigma}}^P - \boldsymbol{B}^P (\boldsymbol{B}^P)^\top }_F & \leq C_1 \cdot p \cdot \varepsilon^P(t).
\end{align}
Denote the event that both \eqref{eq:fanLemma5eq1} and \eqref{eq:fanLemma5eq2} hold by $E_0$. We have that the probability that $E_0$ holds is at least $1-6e^{-t}$. Write $\boldsymbol{S}^A = \frac{n^P}{n^P + n^Q} \boldsymbol{B}^P (\boldsymbol{B}^P)^\top + \frac{n^Q}{n^P + n^Q}\boldsymbol{S}$. By Assumption \ref{ass:factorLoadingDiff}, it is easy to see that
\begin{equation}
    \label{eq:factorTransferLemmaProof1}
    \norm{\boldsymbol{S}^A - \boldsymbol{S}}_F \leq p \cdot \varepsilon.
\end{equation}
Then, it follows from \eqref{eq:fanLemma5eq1} and \eqref{eq:fanLemma5eq2} that under the event $E_0$, we have
\begin{equation}
    \label{eq:factorTransferLemmaProof2}
    \begin{aligned}
    \norm{\widehat{\boldsymbol{\Sigma}}^A - \boldsymbol{S}^A}_F &= \norm{\frac{n^P}{n^P + n^Q}(\widehat{\boldsymbol{\Sigma}}^P - \boldsymbol{B}^P (\boldsymbol{B}^P)^\top) + \frac{n^Q}{n^P + n^Q}(\widehat{\boldsymbol{\Sigma}}^Q - \boldsymbol{S})}_F\\
    &\leq \frac{n^P}{n^P + n^Q}\norm{(\widehat{\boldsymbol{\Sigma}}^P -\boldsymbol{B}^P (\boldsymbol{B}^P)^\top)}_F + \frac{n^Q}{n^P + n^Q} \norm{(\widehat{\boldsymbol{\Sigma}}^Q - \boldsymbol{S})}_F\\
    &\leq C_1 \cdot p \cdot (\frac{n^P}{n^P + n^Q} \varepsilon^P(t) + \frac{n^Q}{n^P + n^Q} \varepsilon^Q(t))\\
    &\leq C_1 \cdot p \cdot \Big( r \sqrt{\frac{\log p + t}{n^P + n^Q}} + r^2\sqrt{\frac{\log r + t}{n^P + n^Q}} + \frac{1}{\sqrt{p}} \Big).
    \end{aligned}  
\end{equation}
From \eqref{eq:factorTransferLemmaProof1} and \eqref{eq:factorTransferLemmaProof2}, we obtain that under the event $E_0$ we have
\begin{equation}
    \label{eq:factorTransferLemmaProof3}
    \norm{\widehat{\boldsymbol{\Sigma}}^A - \boldsymbol{S}}_F \ \leq \norm{\widehat{\boldsymbol{\Sigma}}^P - \boldsymbol{S}^A}_F \ + \ \norm{\boldsymbol{S}^A - \boldsymbol{S}}_F \ \leq (C_1 \lor 1 ) \cdot p \cdot \varepsilon^A(t)   .
\end{equation}
Fix $c_4 = 2C_1$, then $\delta \geq 2C_1 \varepsilon^Q(t)$ gives $C_1 \varepsilon^Q(t)\leq \delta/2$. We continue the proof by considering the following two cases separately.

\noindent\textbf{Case I:} 

Suppose $\varepsilon^A(t)> \frac{1}{2(C_1\lor 1)}\delta$.
By definition of $\widehat{\Sigma}^{\TL}$ we have
\begin{align*}
    \norm{\widehat{\boldsymbol{\Sigma}}^{\TL} - \widehat{\boldsymbol{\Sigma}}^Q}_F \ \leq \frac{n^P}{n^P + n^Q} \norm{\widehat{\boldsymbol{\Sigma}}^Q - \widehat{\boldsymbol{\Sigma}}^P}_F \mathbbm{1} \{p^{-1} \norm{\widehat{\boldsymbol{\Sigma}}^Q - \widehat{\boldsymbol{\Sigma}}^P}_F \leq \delta\} \leq p \delta,
\end{align*}
so from \eqref{eq:fanLemma5eq1} we have that with probability at least $1-6e^{-t}$
\begin{equation}
    \label{eq:factorTransferLemmaProof4}
    \norm{\widehat{\boldsymbol{\Sigma}}^{\TL} - \boldsymbol{S}}_F \leq \norm{\widehat{\boldsymbol{\Sigma}}^{\TL} - \widehat{\boldsymbol{\Sigma}}^Q}_F + \norm{\widehat{\boldsymbol{\Sigma}}^Q - \boldsymbol{S}}_F \ \leq p \delta + C_1 \cdot p \cdot \varepsilon^Q(t) \leq \frac{3}{2} p\delta \lesssim p \Big(\delta\land \varepsilon^A(t)\Big).
\end{equation}

\noindent\textbf{Case II:} 

Suppose $\varepsilon^A(t)\leq \frac{1}{2(C_1\lor 1)}\delta$. From \eqref{eq:factorTransferLemmaProof3}, we obtain that under the event $E_0$ we have
\begin{equation}
    \label{eq:factorTransferLemmaProof5}
    \begin{split}
        \norm{\widehat{\boldsymbol{\Sigma}}^A - \widehat{\boldsymbol{\Sigma}}^Q}_F & \leq \norm{\widehat{\boldsymbol{\Sigma}}^A - \boldsymbol{S}}_F + \norm{\boldsymbol{S} - \widehat{\boldsymbol{\Sigma}}^Q}_F\\
        & \leq (C_1 \lor 1) \cdot p \cdot \varepsilon^A(t) + C_1 \cdot p \cdot \varepsilon^Q(t)\\
        & \leq \delta/2 + \delta/2 = \delta,
    \end{split}
\end{equation}
which implies $\widehat{\boldsymbol{\Sigma}}^{\TL} = \widehat{\boldsymbol{\Sigma}}^{A}$. Therefore, it follows from \eqref{eq:factorTransferLemmaProof3} that, under the event $E_0$ we have
\begin{equation}
    \label{eq:factorTransferLemmaProof6}
    \norm{\widehat{\boldsymbol{\Sigma}}^{\TL} - \boldsymbol{S}}_F = \norm{\widehat{\boldsymbol{\Sigma}}^{A} - \boldsymbol{S}}_F \leq (C_1 \lor 1) \cdot p \cdot \varepsilon^A(t) \lesssim p \Big(\delta \land \varepsilon^A(t) \Big).
\end{equation}
Combining two inequalities \eqref{eq:factorTransferLemmaProof4} and \eqref{eq:factorTransferLemmaProof6} in the two different cases leads to the result.
\end{proof}

\subsection{Proof of Lemma \ref{lemma:m-lip-weight}}
The proof follows the same approach as Lemma 8 of \cite{fan2024factor}, with modifications only to the first layer \(\pazocal{L}_0\), where this paper additionally accounts for the \(s(x)\) entry. For the convenience of the reader, we provide the complete proof below.

\begin{proof}
For $m \in \pazocal{F}_{s, \kappa}$, we first recursively define
\begin{align*}
    m^{(l)}_+(\boldsymbol{x}) = \begin{cases}
        \phi \circ \pazocal{L}_0(\boldsymbol{x}, s(\boldsymbol{x})) & \qquad l = 1 \\
        \pazocal{L}_1 \circ m_+^{(l-1)}(\boldsymbol{x}) & \qquad l = 2 \\
        \pazocal{L}_{l-1} \circ \bar{\sigma} \circ m_+^{(l-1)}(\boldsymbol{x}) & \qquad l \in \{3, \cdots, L + 1\} 
    \end{cases},
\end{align*} 
and
\begin{align*}
    m^{(l)}_{-}(\boldsymbol{z}) = \begin{cases}
        \pazocal{L}_{L+1} \circ \bar{\sigma}(\boldsymbol{z}) & \qquad l = L + 1\\
        m_-^{(l + 1)} \circ \pazocal{L}_{l} \circ \bar{\sigma}(\boldsymbol{z}) & \qquad l \in \{2, \cdots, L\} \\
        m_-^{(l + 1)} \circ \pazocal{L}_{l}(\boldsymbol{z}) & \qquad l =1 
    \end{cases}.
\end{align*}
Similarly, we can define the corresponding functions $\breve{m}^{(l)}_+(\boldsymbol{x}): \mathbb{R}^p\to \mathbb{R}^{d_{l - 1}}$ and $\breve{m}^{(l)}_-(\boldsymbol{z}): \mathbb{R}^{d_{l-1}} \to \mathbb{R}$ for $\breve{m}$ that 
\begin{align*}
    \breve{m}(x) = \breve{\pazocal{L}}_{L + 1} \circ \bar{\sigma} \circ \breve{\pazocal{L}}_{L} \circ \bar{\sigma} \circ \cdots \circ \breve{\pazocal{L}}_2 \circ \bar{\sigma} \circ \breve{\pazocal{L}}_1 \circ \phi \circ \breve{\pazocal{L}}_0(\boldsymbol{x}, s(\boldsymbol{x})).
\end{align*}
We make the following two claims:
\begin{align}
    \label{eq:m-lip-weight-claim1}
    \norm{m_-^{(l)}(\boldsymbol{u}) - m_-^{(l)}(\boldsymbol{v})}_\infty \leq \begin{cases}
    (T N)^{L + 2 - l} \norm{\boldsymbol{u} - \boldsymbol{v}}_\infty & \qquad l \in \{2, \cdots, L + 1\} \\
    T^{L + 1} N^{L} (N + \overline{r} + 1) \norm{\boldsymbol{u} - \boldsymbol{v}}_\infty & \qquad l = 1
    \end{cases},
\end{align} 
and
\begin{align}
    \label{eq:m-lip-weight-claim2}
    \norm{m_+^{(l)}(\boldsymbol{x})}_\infty \leq (M \lor K\norm{\boldsymbol{W}}_\max) (T(N + 1))^{l - 1} \qquad \forall l \in \{1, \cdots, L + 1\}.
\end{align}

\begin{proof}[Proof of Claim \eqref{eq:m-lip-weight-claim1}]

For the first claim \eqref{eq:m-lip-weight-claim1}, note that for any $l \in \{1, \cdots, L\}$,
\begin{align}
    \label{eq:m-lip-weight-claim1:proof-induction-ineq}
    \norm{\pazocal{L}_{l}(\boldsymbol{u}) - \pazocal{L}_{l}(\boldsymbol{v})}_\infty \le \norm{\boldsymbol{W}_l^\top (\boldsymbol{u} - \boldsymbol{v})}_\infty \leq d_{l - 1} T \norm{\boldsymbol{u} - \boldsymbol{v}}_\infty,
\end{align} 
provided that $\norm{\boldsymbol{W}_l}_{\max} \leq T$. We prove the claim \eqref{eq:m-lip-weight-claim1} by induction. For the base case $l = L + 1$, the result follows directly from \eqref{eq:m-lip-weight-claim1:proof-induction-ineq} as
\begin{align*}
    \norm{m^{(l)}_{-}(\boldsymbol{u}) - m^{(l)}_{-}(\boldsymbol{v})}_\infty = \Norm{\pazocal{L}_{L + 1} \left(\bar{\sigma}(\boldsymbol{u}) - \bar{\sigma}(\boldsymbol{v})\right)}_\infty \le d_L T \norm{\boldsymbol{u} - \boldsymbol{v}} = T N \norm{\boldsymbol{u} - \boldsymbol{v}}.
\end{align*}
If the claim \eqref{eq:m-lip-weight-claim1} holds for $l + 1$ with $l \in \{2, 3,\cdots, L\}$, then we further have
\begin{align*}
    \norm{m^{(l)}_{-}(\boldsymbol{u}) - m^{(l)}_{-}(\boldsymbol{v})}_\infty &= \norm{m^{(l + 1)}_{-}(\boldsymbol{u}) \circ \pazocal{L}_{l} \circ \bar{\sigma}(\boldsymbol{u}) - m^{(l + 1)}_{-}(\boldsymbol{u}) \circ \pazocal{L}_{l} \circ \bar{\sigma}(\boldsymbol{v})}_\infty \\
    &{\le} (TN)^{L + 1 - l} \norm{\pazocal{L}_{l} \circ \bar{\sigma}(\boldsymbol{u}) - \pazocal{L}_{l} \circ \bar{\sigma}(\boldsymbol{v})}_\infty \\
    &{\le} (TN)^{L + 1 - l} \norm{\boldsymbol{W}^\top_{l - 1} (\bar{\sigma}(\boldsymbol{u}) - \bar{\sigma}(\boldsymbol{v}))}_\infty \\
    &{\le} (TN)^{L+1-l} (TN) \norm{\bar{\sigma}(\boldsymbol{u}) - \bar{\sigma}(\boldsymbol{v})}_\infty \\
    &{\le} (TN)^{L+2-l} \norm{\boldsymbol{u} - \boldsymbol{v}}_\infty.
\end{align*} Here, the first inequality follows from the induction hypothesis, the second inequality follows from \eqref{eq:m-lip-weight-claim1:proof-induction-ineq}, and the third inequality follows from the fact that \( |\sigma(x)| \leq |x| \) for the ReLU activation function. The case for \(l = 1\) proceeds by
\begin{align*}
    \norm{m^{(1)}_{-}(\boldsymbol{u}) - m^{(1)}_{-}(\boldsymbol{v})}_\infty &= \norm{m^{(2)}_{-}(\boldsymbol{u}) \circ \pazocal{L}_{1} \circ \bar{\sigma}(\boldsymbol{u}) - m^{(2)}_{-}(\boldsymbol{u}) \circ \pazocal{L}_{1} \circ \bar{\sigma}(\boldsymbol{v})}_\infty \\
    &{\le} (TN)^{L} \norm{\pazocal{L}_{1} \circ \bar{\sigma}(\boldsymbol{u}) - \pazocal{L}_{1} \circ \bar{\sigma}(\boldsymbol{v})}_\infty \\
    &{\le} (TN)^{L} \norm{\boldsymbol{W}^\top_{1} (\bar{\sigma}(\boldsymbol{u}) - \bar{\sigma}(\boldsymbol{v}))}_\infty \\
    &{\le} (TN)^{L} T(N + \overline{r} + 1) \norm{\bar{\sigma}(\boldsymbol{u}) - \bar{\sigma}(\boldsymbol{v})}_\infty \\
    &{\le} (TN)^{L} T(N + \overline{r} + 1) \norm{\boldsymbol{u} - \boldsymbol{v}}_\infty
\end{align*}
thereby concluding the proof.
\end{proof}

\begin{proof}[Proof of Claim \eqref{eq:m-lip-weight-claim2}]
Our proof will use the fact that
\begin{align}
\label{eq:m-lip-weight-claim2:proof-induction-ineq}
    \norm{\pazocal{L}(\boldsymbol{x})}_\infty = \norm{\boldsymbol{W}_l \boldsymbol{x} + \boldsymbol{b}_l}_\infty \le B + d_{l-1}B \norm{\boldsymbol{x}}_\infty
\end{align} 
provided $l\in \{1,\cdots, L\}$ repeatedly. We also prove \eqref{eq:m-lip-weight-claim2} by induction. For the base case \(l = 1\), it follows from the definition of \(\pazocal{L}_0\) and \(\phi\) that
\begin{align*}
    \norm{m^{(1)}_+ (\boldsymbol{x})}_\infty \le \norm{\pazocal{L}_0 (\boldsymbol{x}, s(\boldsymbol{x}))}_\infty \leq K \norm{\boldsymbol{W}}_{\max} \lor M.
\end{align*} 
where $\boldsymbol{x} \in [-K, K]^p$.
If \eqref{eq:m-lip-weight-claim2} holds for \(l - 1\) with \(l \in \{2, \dots, L + 1\}\), then combining \eqref{eq:m-lip-weight-claim2:proof-induction-ineq} with the fact \( |\sigma(x)| \leq |x| \) gives
\begin{align*}
    \norm{m_+^{(l)}(\boldsymbol{x})}_\infty \leq T + d_{l - 1} T \norm{m^{(l - 1)}_+ (\boldsymbol{x})}_\infty &\leq T + (T N) (T(N+1))^{l - 2} (K \norm{\boldsymbol{W}}_{\max} \lor M) \\
    &\leq (K \norm{\boldsymbol{W}}_{\max} \lor M) (T(N + 1))^{l - 1}.
\end{align*} 
which completes the proof of claim \eqref{eq:m-lip-weight-claim2}.
\end{proof}

Now, we return to the proof of Lemma \ref{lemma:m-lip-weight}. The above construction of \(m^{(l)}_+\) and \(m^{(l)}_-\) implies that
\begin{align*}
    m(\boldsymbol{x}) = m^{(l)}_- \circ m^{(l)}_+(\boldsymbol{x}) =  m^{(l)}_- \circ \pazocal{L}_{l-1} \circ \bar{\sigma} \circ m^{(l-1)}_+(\boldsymbol{x})
\end{align*}
Consequently, it follows from a series of triangle inequalities that
\begin{align}
\label{eq:proof:lemma:m-lip-weight:telescope}
\begin{split}
    |m(\boldsymbol{x}) - \breve{m}(\boldsymbol{x})| = &\left|\pazocal{L}_{L+1} \circ \bar{\sigma}  \circ m_+^{(L+1)}(\boldsymbol{x}) - \breve{\pazocal{L}}_{L+1} \circ \bar{\sigma} \circ m_+^{(L+1)}(\boldsymbol{x})\right| \\
    +& \sum_{l=1}^{L} \left|\breve{m}_-^{(l+1)} \circ \pazocal{L}_l \circ \bar{\sigma} \circ m^{(l)}_+(\boldsymbol{x}) - \breve{m}_-^{(l+1)} \circ \breve{\pazocal{L}}_l \circ \bar{\sigma} \circ m^{(l)}_+(\boldsymbol{x})  \right|\\
    +& \left|\breve{m}_-^{(1)} \circ \phi \circ \pazocal{L}_0(\boldsymbol{x}, s(\boldsymbol{x})) - \breve{m}_-^{(1)} \circ \phi \circ \breve{\pazocal{L}}_0(\boldsymbol{x}, s(\boldsymbol{x}))\right|.
\end{split}
\end{align}
Using Claim \eqref{eq:m-lip-weight-claim1} with $l = 1$ gives
\begin{align*}
    \norm{\breve{m}_-^{(1)} \circ \phi \circ \pazocal{L}_0(\boldsymbol{x}, s(\boldsymbol{x})) - \breve{m}_-^{(1)} \circ \phi \circ \breve{\pazocal{L}}_0(\boldsymbol{x}, s(\boldsymbol{x}))} &\le T^{L+1} N^{L} (N+\overline{r}+1) \norm{\phi \circ \pazocal{L}_0(\boldsymbol{x}, s(\boldsymbol{x})) - \phi \circ \breve{\pazocal{L}}_0(\boldsymbol{x}, s(\boldsymbol{x}))}_\infty \\
    &\le T^{L+1} N^{L} (N+\overline{r}+1) \norm{\pazocal{L}_0(\boldsymbol{x}, s(\boldsymbol{x})) - \breve{\pazocal{L}}_0(\boldsymbol{x}, s(\boldsymbol{x}))}_\infty.
\end{align*}
For any $i \in \{1, \cdots, N\}$ and $\boldsymbol{x} \in [-K, K]^p$, we have
\begin{align}
\label{eq:proof:lemma:m-lip-weight:derivation1}
\begin{split}
    \norm{[\pazocal{L}_0(\boldsymbol{x}, s(\boldsymbol{x})]_{i + \overline{r}} - [\breve{\pazocal{L}}_0(\boldsymbol{x}, s(\boldsymbol{x}))]_{i + \overline{r}}} &= \norm{[\boldsymbol{\Theta}]_{:, i}^\top \boldsymbol{x} - [\breve{\boldsymbol{\Theta}}]_{:, i}^\top \boldsymbol{x}} \leq \norm{\boldsymbol{x}}_\infty \norm{[\boldsymbol{\Theta}]_{:, i} - [\breve{\boldsymbol{\Theta}}]_{:, i}}_1 \leq K p \norm{\boldsymbol{\Theta} - \breve{\boldsymbol{\Theta}}}_{\max},
\end{split}
\end{align}
$$$$
which further indicates
\begin{align}
\label{eq:proof:lemma:m-lip-weight:decomp1}
    \norm{\breve{m}_-^{(1)} \circ \phi \circ \pazocal{L}_0(\boldsymbol{x}, s(\boldsymbol{x})) - \breve{m}_-^{(1)} \circ \phi \circ \breve{\pazocal{L}}_0(\boldsymbol{x}, s(\boldsymbol{x}))} \le K T^{L+1} N^{L} (N + \overline{r} + 1) p\norm{\boldsymbol{\Theta} - \breve{\boldsymbol{\Theta}}}_{\max}.
\end{align} 
For any $l \in \{1, \cdots, L\}$, it follows from a similar argument that
\begin{align*}
    \norm{\pazocal{L}_l \circ \bar{\sigma}(\boldsymbol{z}) - \breve{\pazocal{L}}_l \circ \bar{\sigma}(\boldsymbol{z})}_\infty &= \norm{\boldsymbol{W}_l \bar{\sigma}(\boldsymbol{z}) - \breve{\boldsymbol{W}}_l \bar{\sigma}(\boldsymbol{z}) + \boldsymbol{b}_l - \breve{\boldsymbol{b}}_l}_\infty \\
    &\leq \norm{\boldsymbol{W}_l \bar{\sigma}(\boldsymbol{z}) - \breve{\boldsymbol{W}}_l \bar{\sigma}(\boldsymbol{z})}_\infty + \norm{\boldsymbol{b}_l - \breve{\boldsymbol{b}}_l}_\infty \\
    &= \max_{i\in \{1,\cdots,d_l\}} \left|[\boldsymbol{W}_l]_{i,:} \bar{\sigma}(\boldsymbol{z}) - [\breve{\boldsymbol{W}}_l]_{i,:} \bar{\sigma}(\boldsymbol{z}) \right| + \norm{\boldsymbol{b}_l - \breve{\boldsymbol{b}}_l}_\infty \\
    &\leq \max_{i\in \{1,\cdots,d_l\}} \norm{[\boldsymbol{W}_l]_{i,:} - [\breve{\boldsymbol{W}}_l]_{i,:} }_1 \norm{\bar{\sigma}(\boldsymbol{z})}_\infty + \norm{\boldsymbol{b}_l - \breve{\boldsymbol{b}}_l}_\infty \\
    &\leq \norm{\boldsymbol{W}_l - \breve{\boldsymbol{W}}_l}_{\max} d_{l-1} \norm{\boldsymbol{z}}_\infty + \norm{\boldsymbol{b}_l - \breve{\boldsymbol{b}}_l}_\infty,
\end{align*} 
which implies
\begin{align}
\label{eq:m-lip-weight-lell-sigma-bound}
    \norm{\pazocal{L}_l \circ \bar{\sigma}(\boldsymbol{z}) - \breve{\pazocal{L}}_l \circ \bar{\sigma}(\boldsymbol{z})}_\infty \le d(\theta(m), \theta(\breve{m})) (1 + d_{l-1} \norm{\boldsymbol{z}}_\infty).
\end{align}
It follows from \eqref{eq:m-lip-weight-claim2} and \eqref{eq:m-lip-weight-lell-sigma-bound} with $l = L + 1$ that
\begin{align}
\label{eq:proof:lemma:m-lip-weight:decomp2}
\begin{split}
    \left| \pazocal{L}_{L + 1} \circ \bar{\sigma}  \circ m_+^{(L + 1)}(\boldsymbol{x}) - \breve{\pazocal{L}}_{L+1} \circ \bar{\sigma} \circ m_+^{(L + 1)}(\boldsymbol{x}) \right| 
    &\leq d(\boldsymbol{\theta}(m), \boldsymbol{\theta}(\breve{m})) (1 + N \norm{m_+^{(L + 1)}(\boldsymbol{x})}_\infty) \\
    &\leq (M\lor K\norm{\boldsymbol{W}}_{\max}) T^L(N + 1)^{L + 1} d(\boldsymbol{\theta}(m), \boldsymbol{\theta}(\breve{m})),
\end{split}
\end{align} 
given that $(T(N + 1))^L (M \lor K \norm{\boldsymbol{W}}_{\max}) \geq 1$.

In addition, for any $l\in \{1, \cdots, L\}$, combing \eqref{eq:m-lip-weight-claim1}, \eqref{eq:m-lip-weight-claim2} and \eqref{eq:m-lip-weight-lell-sigma-bound} gives
\begin{align}
\label{eq:proof:lemma:m-lip-weight:decomp3}
\begin{split}
    &\Big|\breve{m}_-^{(l + 1)} \circ \pazocal{L}_l \circ \bar{\sigma} \circ m^{(l)}_+(\boldsymbol{x}) - \breve{m}_-^{(l+1)} \circ \breve{\pazocal{L}}_l \circ \bar{\sigma} \circ m^{(l)}_+ (\boldsymbol{x}) \Big|\\
    \leq& (TN)^{L + 1 - l} \norm{\pazocal{L}_l \circ \bar{\sigma} \circ m^{(l)}_+(\boldsymbol{x}) - \breve{\pazocal{L}}_l \circ \bar{\sigma} \circ m^{(l)}_+(\boldsymbol{x})}_\infty \\
    \leq& (TN)^{L + 1 - l} d(\theta(m), \theta(\breve{m})) (1 + d_{l - 1} \norm{m^{(l)}_+(\boldsymbol{x})}_\infty) \\
    \leq& (TN)^{L + 1 - l} d(\theta(m), \theta(\breve{m})) \left(1 + N(M \lor K \norm{\boldsymbol{W}}_{\max}) (T(N + 1))^{l-1}\right) \\
    \leq& (M\lor K\norm{\boldsymbol{W}}_{\max}) T^L (N + 1)^L d(\boldsymbol{\theta}(m), \boldsymbol{\theta}(\breve{m})).
\end{split}
\end{align} 
Plugging \eqref{eq:proof:lemma:m-lip-weight:decomp1}, \eqref{eq:proof:lemma:m-lip-weight:decomp2} and \eqref{eq:proof:lemma:m-lip-weight:decomp3} into \eqref{eq:proof:lemma:m-lip-weight:telescope} completes the proof.

\end{proof}

\subsection{Proofs of Lemma \ref{lemma:fast} and Corollary \ref{col:fast}}

\begin{proof}[Proof of Lemma \ref{lemma:fast}] 
Since the conditions and results do not involve any source data, we omit the superscript $Q$ on $\boldsymbol{W}^Q, \boldsymbol{B}^Q, \boldsymbol{H}^Q, y^Q, \boldsymbol{f}^Q, \boldsymbol{u}^Q, n^Q, \widetilde{f}^Q, \epsilon^Q$ and the subscript $Q$ on $\pazocal{E}_Q, \widehat{\pazocal{E}}_Q$, and $n^Q$ for notational simplicity, emphasizing that this general result is independent of any fine-tuning setting. Furthermore, we abbreviate $\boldsymbol{x}^Q$ as $\boldsymbol{x}$ for simplicity.

\paragraph{Step I: Bound the approximation error for $g^Q$.} 
The goal of this step is to show that there exists some $m(\boldsymbol{x}; \boldsymbol{W}, \widetilde{\boldsymbol{\Theta}}, \widetilde{g}, s) \in \pazocal{F}_s$ with $$\norm{\widetilde{\boldsymbol{\Theta}}}_0\leq |\pazocal{J}|\mathbf{1}\{\delta_f + \delta_a \leq \delta^0_f + \delta^0_a\}+ |\pazocal{J} \cup \pazocal{J}^P|\mathbf{1}\{\delta_f + \delta_a > \delta^0_f + \delta^0_a\}$$ such that $$\pazocal{E}(m)\lesssim (\delta_f + \delta_a)\land (\delta^0_f + \delta^0_a).$$ From Theorem 2 of \cite{fan2024factor} and the first step of its proof, we observe that there always exists some $m^0(\boldsymbol{x}; \boldsymbol{W}, \widetilde{\boldsymbol{\Theta}}, \widetilde{g}, s) \in \pazocal{F}^0_s \subset \pazocal{F}_s$ with $\norm{\widetilde{\boldsymbol{\Theta}}}_0 \leq |\pazocal{J} \cup \pazocal{J}^P|$ such that 
$$\pazocal{E}(m) \lesssim \delta^0_f + \delta^0_a.$$
This result corresponds to the approximation error bound without involving $s$ when we approximate $g^Q(\boldsymbol{f}^Q, \boldsymbol{u}^Q_{\pazocal{J} \cup \pazocal{J}^P})$. Therefore, it suffices to show that
there exists some $m(\boldsymbol{x}; \boldsymbol{W}, \widetilde{\boldsymbol{\Theta}}, \widetilde{g}, s) \in \pazocal{F}_s$ with $\norm{\widetilde{\boldsymbol{\Theta}}}_0\leq |\pazocal{J}|$ such that 
$$\pazocal{E}(m)\lesssim \delta_f + \delta_a$$
to obtain the goal of this step. We divide the derivation into two cases regarding whether $r = 0$.

\paragraph{Case 1: $r \geq 1$.} Let 
$$
\begin{aligned}
 \widetilde{m}^*(\boldsymbol{x}) &= h(\boldsymbol{H}^\dagger \widetilde{\boldsymbol{f}}, \boldsymbol{x}_{\pazocal{J}}-[\boldsymbol{B}_{\pazocal{J},:}] \boldsymbol{H}^\dagger \widetilde{\boldsymbol{f}},g^P(\boldsymbol{x}_{\pazocal{J}^P})),\\
 m^*(\boldsymbol{x}) &= h(\boldsymbol{H}^\dagger \widetilde{\boldsymbol{f}}, \boldsymbol{x}_{\pazocal{J}}-[\boldsymbol{B}_{\pazocal{J},:}] \boldsymbol{H}^\dagger \widetilde{\boldsymbol{f}},s(\boldsymbol{x})),
\end{aligned}
$$
where $\boldsymbol{H}^\dagger$ is the Moore–Penrose inverse of $\boldsymbol{H}$. For any $\boldsymbol{f},\boldsymbol{u}$, from the definition of $\boldsymbol{W}$ and $\boldsymbol{H}$,
we have
$$\widetilde{f} = \boldsymbol{H}\boldsymbol{f}+\boldsymbol{\xi},\quad \boldsymbol{\xi} = p^{-1}\boldsymbol{W}^\top \boldsymbol{u},$$
which indicates that
\begin{equation}
\label{eq:appEq1}
\begin{aligned}
|\widetilde{m}^*(\boldsymbol{x})-g^Q(\boldsymbol{f},\boldsymbol{u}_{\pazocal{J} \cup \pazocal{J}^P})|&\lesssim \norm{\boldsymbol{H}^\dagger \boldsymbol{\xi}}_2 + \norm{[\boldsymbol{B}_{\pazocal{J},:}] \boldsymbol{H}^\dagger \boldsymbol{\xi}}_2\\
&\leq (\norm{[\boldsymbol{B}_{\pazocal{J},:}]}_2 + 1)\norm{\boldsymbol{H}^\dagger}_2 \norm{\boldsymbol{\xi}}_2 \\
&\leq (\norm{[\boldsymbol{B}_{\pazocal{J},:}]}_2 + 1)\norm{\boldsymbol{H}^\dagger}_2 \norm{\boldsymbol{\xi}}_2 / \nu_\min( \boldsymbol{H})\\
&\leq (\norm{[\boldsymbol{B}_{\pazocal{J},:}]}_F + 1)\norm{\boldsymbol{H}^\dagger}_2 \norm{\boldsymbol{\xi}}_2 / \nu_\min( \boldsymbol{H})\\
&\lesssim \frac{\sqrt{|\pazocal{J}|r}}{\nu_\min(\boldsymbol{H})}\norm{\boldsymbol{\xi}}_2
\end{aligned}
\end{equation}

where the first inequality is due to the Lipschitz condition of $g^Q$ in Assumption \ref{ass:subGaussianNoiseAndBoundedFunction}, and the last inequality comes from the boundedness of $\norm{\boldsymbol{B}}_{\max}$ in Assumpiton \ref{ass:boundedness}. Following the linearity of expectation, we can bound the expectation of $\norm{\boldsymbol{\xi}}_2^2$ by
\begin{equation}
\label{eq:appEq2}
\begin{aligned}
\EE \norm{\boldsymbol{\xi}}_2^2 &= p^{-2}\EE\Big[\sum_{k = 1}^{\overline{r}} \Big(\sum_{j = 1}^p \boldsymbol{W}_{jk} \boldsymbol{u}_j \Big)^2 \Big]\\
&= p^{-2}\sum_{k = 1}^{\overline{r}}\sum_{j = 1}^p \boldsymbol{W}^2_{jk}\EE[\boldsymbol{u}_j^2] + \sum_{j \neq j'} \boldsymbol{W}_{jk}\boldsymbol{W}_{j'k} \EE[\boldsymbol{u}_j \boldsymbol{u}_{j'}]\\
&\leq \frac{\overline{r}}{p} \max_{j,k}| \boldsymbol{W}_{jk}|\max_j \EE[\boldsymbol{u}_j^2] + \frac{\overline{r}}{p^2} \max_{j,k}|\boldsymbol{W}_{jk}|^2 \sum_{j \neq j'} |\EE[\boldsymbol{u}_j\boldsymbol{u}_{j'}]|\lesssim \frac{\overline{r}}{p},
\end{aligned}
\end{equation}
where the last inequality applies Assumptions \ref{ass:boundedness} and \ref{ass:weakDependence}. From \eqref{eq:appEq1}, \eqref{eq:appEq2}, and the Lipschitz continuity of $h$, we have
\begin{equation}
\label{eq:appEq3}
\begin{aligned}
\pazocal{E}(m^*) &= \EE|m^*(\boldsymbol{x})-g^Q(\boldsymbol{f},\boldsymbol{u}_{\pazocal{J} \cup \pazocal{J}^P})|^2 \\
&\lesssim \EE|m^*(\boldsymbol{x}) - \widetilde{m}^*(\boldsymbol{x})|^2 + \EE|\widetilde{m}^*(\boldsymbol{x}) -g^Q(\boldsymbol{f},\boldsymbol{u}_{\pazocal{J} \cup \pazocal{J}^P})|^2\\
&\lesssim \norm{s(\boldsymbol{x}) - g^P(\boldsymbol{x}_{\pazocal{J}^P})}_2^2 + \frac{|\pazocal{J}|r \cdot \overline{r}}{(\nu_\min(\boldsymbol{H}))^2 p}\\
&\lesssim \delta_a + \delta_f.
\end{aligned}  
\end{equation}
It remains to find some $m \in \pazocal{F}_s$ that approximates $m^*$ well. We choose $g\in\pazocal{G}(L-1,r+|\pazocal{J}|+1,N,M,T)$ that minimzes $\sup_{\kappa\in[-M,M]}\norm{g(\boldsymbol{f},\boldsymbol{u}_{\pazocal{J}},\kappa)-h(\boldsymbol{f},\boldsymbol{u}_{\pazocal{J}},\kappa)}_\infty^2$. Then, we have
\begin{equation}
\label{eq:appEq4}
|g(\boldsymbol{f},\boldsymbol{u}_{\pazocal{J}},s(\boldsymbol{x}))-h(\boldsymbol{f},\boldsymbol{u}_{\pazocal{J}},s(\boldsymbol{x})|\lesssim \sqrt{\delta_a},\quad \forall (\boldsymbol{f},\boldsymbol{u}_{\pazocal{J}})\in [-2b, 2b]^r \times [-2b, 2b]^{|\pazocal{J}|}
\end{equation}
following the definition of $\delta_a$. Since $g$ is a ReLU network, we can write $g$ as
$$g(\boldsymbol{f}, \boldsymbol{u}_{\pazocal{J}}, s(\boldsymbol{x})) = \pazocal{L}^g_{L + 1}\circ \bar{\sigma}\circ \pazocal{L}^g_{L}\circ \bar{\sigma}\circ\cdots \circ \bar{\sigma}\circ \pazocal{L}^g_{3} \circ \bar{\sigma} \circ \pazocal{L}^g_{2}(f, \boldsymbol{u}_{\pazocal{J}}, s(\boldsymbol{x})).$$ 

Denote $\pazocal{J} = \{l_1, \dots, l_{|\pazocal{J}|} \}\subset \{1, \dots, p\}$. Then, we construct the approximation $m$ as
$$m(\boldsymbol{x}) = \pazocal{L}^g_{L + 1} \circ \bar{\sigma} \circ \pazocal{L}^g_{L} \circ \bar{\sigma} \circ \cdots \circ \bar{\sigma} \circ \pazocal{L}^g_{3} \circ \bar{\sigma} \circ \pazocal{L}^g_{2}\circ \bar{\sigma} \circ \pazocal{L}^g_{1}\circ \phi \circ \pazocal{L}^g_0(\boldsymbol{x}, s(\boldsymbol{x})),$$
where the $\pazocal{L}_0,\pazocal{L}_1$ and $\pazocal{L}_2$ are constructed as follows

\begin{enumerate}
\item For $\pazocal{L}_0: \R^{p + 1} \rightarrow \R^{\overline{r} + |\pazocal{J}| + 1}$, let
$$\pazocal{L}_0(\boldsymbol{x}, s(\boldsymbol{x})) = (p^{-1} \boldsymbol{x}^\top W, \boldsymbol{x}^\top \widetilde{\Theta}, s(\boldsymbol{x}))^\top = (\widetilde{f}^\top, \boldsymbol{x}_{\pazocal{J}}^\top, s(\boldsymbol{x}))^\top, \quad \widetilde{\Theta}_{ij} = \mathbf{1}\{i \leq |\pazocal{J}|, \ j = l_i\}.$$ 
Given the definition of $\phi$, we have $\phi \circ \pazocal{L}_0(\boldsymbol{x}, s(\boldsymbol{x}))=\pazocal{L}_0(\boldsymbol{x}, s(\boldsymbol{x}))$ given $M \geq r(b + 1) \geq \norm{\boldsymbol{x}_{\pazocal{J}}}_{\infty}$. Moreover, it is trivial that $\norm{\widetilde{\boldsymbol{\Theta}}}_0 = |\pazocal{J}|$.

\item For $\pazocal{L}_1: \R^{\overline{r} + |\pazocal{J}| + 1} \rightarrow \R^{2(\overline{r} + |\pazocal{J}| + 1)}$, let
$$
\pazocal{L}_1 \begin{bmatrix}
\widetilde{f}\\
\boldsymbol{x}_{\pazocal{J}}\\
s(\boldsymbol{x})
\end{bmatrix} =
\begin{bmatrix}
\boldsymbol{H}^\dagger & 0 & 0\\
-[\boldsymbol{B}_{\pazocal{J}, :}] \boldsymbol{H}^\dagger & \boldsymbol{I} & 0 \\
-\boldsymbol{H}^\dagger & 0 & 0\\
[\boldsymbol{B}_{\pazocal{J}, :}] \boldsymbol{H}^\dagger & -\boldsymbol{I} & 0\\
0 & 0 & 1\\
0 & 0 & -1
\end{bmatrix}
\begin{bmatrix}
\widetilde{\boldsymbol{f}}\\
\boldsymbol{x}_{\pazocal{J}}\\
s(\boldsymbol{x})
\end{bmatrix} + 0 =
\begin{bmatrix}
\boldsymbol{H}^\dagger \widetilde{\boldsymbol{f}}\\
\boldsymbol{x}_{\pazocal{J}} - [\boldsymbol{B}_{\pazocal{J}, :}] \boldsymbol{H}^\dagger \widetilde{\boldsymbol{f}}\\
-\boldsymbol{H}^\dagger \widetilde{\boldsymbol{f}}\\
-(\boldsymbol{x}_{\pazocal{J}} - [\boldsymbol{B}_{\pazocal{J}, :}] \boldsymbol{H}^\dagger \widetilde{\boldsymbol{f}})\\
s(\boldsymbol{x})\\
-s(\boldsymbol{x})
\end{bmatrix}
$$
\item Suppose that the weights $\pazocal{L}_2^g$ are $\boldsymbol{W}^g_2$ and $\boldsymbol{b}^g_2$. For $\pazocal{L}_2: \R^{2(\overline{r} + |\pazocal{J}| + 1)}\rightarrow \R^N$, given $\boldsymbol{u}\in \R^r,\ \boldsymbol{v} \in \R^{|\pazocal{J}|}$, let
$$
\pazocal{L}_2\begin{bmatrix}
    \boldsymbol{u}\\
    \boldsymbol{v}
\end{bmatrix}
= \begin{bmatrix}
    \boldsymbol{W}^g_2 & -\boldsymbol{W}^g_2
\end{bmatrix}
\begin{bmatrix}
    \boldsymbol{u}\\
    \boldsymbol{v}
\end{bmatrix} + \boldsymbol{b}^g_2.
$$
\end{enumerate}

It follows from the above construction that 
$$
\begin{aligned}
m(\boldsymbol{x}) &= g\Big(\sigma(\boldsymbol{H}^\dagger \widetilde{\boldsymbol{f}}) - \sigma(-\boldsymbol{H}^\dagger \widetilde{\boldsymbol{f}}),\sigma(\boldsymbol{x}_{\pazocal{J}} - [\boldsymbol{B}_{\pazocal{J},:}] \boldsymbol{H}^\dagger \widetilde{\boldsymbol{f}}) - \sigma(-(\boldsymbol{x}_{\pazocal{J}} - [\boldsymbol{B}_{\pazocal{J},:}] \boldsymbol{H}^\dagger \widetilde{\boldsymbol{f}})), \sigma(s(\boldsymbol{x})) - \sigma(-s(\boldsymbol{x}))\Big)\\
&= g(\boldsymbol{H}^\dagger \widetilde{\boldsymbol{f}}, \boldsymbol{x}_{\pazocal{J}} - [\boldsymbol{B}_{\pazocal{J}, :}] \boldsymbol{H}^\dagger \widetilde{\boldsymbol{f}}, s(\boldsymbol{x})).
\end{aligned}
$$
Moreover, all weights of $\pazocal{L}_1, \pazocal{L}_2, \dots, \pazocal{L}_{L + 1}$ is bounded by $T \lor (C_1 \frac{|\pazocal{J}|r}{\nu_\min(\boldsymbol{H})})$ for some constant $C_1$ as $\norm{\boldsymbol{\cdot}}_\max \leq \norm{\boldsymbol{\cdot}}_2$.

We are now in the position to upper bound $\EE|m(\boldsymbol{x}) - m^*(\boldsymbol{x})|^2$. Define the event
$$E = \Big\{\boldsymbol{H}^\dagger \widetilde{\boldsymbol{f}} \in [-2b, 2b]^r, \boldsymbol{x}_{\pazocal{J}} - [\boldsymbol{B}_{\pazocal{J}, :}] \boldsymbol{H}^\dagger \widetilde{\boldsymbol{f}}\in[-2b, 2b]^{|\pazocal{J}|} \Big\}.$$ 

We have that
\begin{equation*}
\begin{aligned}
\EE |m(\boldsymbol{x}) - m^*(\boldsymbol{x})|^2 &= \EE |m(\boldsymbol{x}) - m^*(\boldsymbol{x})|^2 \mathbf{1}_E + \EE |m(\boldsymbol{x}) - m^*(\boldsymbol{x})|^2 \mathbf{1}_{E^c}\\
&\leq \EE |g(\boldsymbol{H}^\dagger \widetilde{\boldsymbol{f}}, \boldsymbol{x}_{\pazocal{J}} - [\boldsymbol{B}_{\pazocal{J},:}] \boldsymbol{H}^\dagger \widetilde{\boldsymbol{f}}, s(\boldsymbol{x})) - h(\boldsymbol{H}^\dagger \widetilde{\boldsymbol{f}}, \boldsymbol{x}_{\pazocal{J}} - [\boldsymbol{B}_{\pazocal{J}, :}] \boldsymbol{H}^\dagger \widetilde{\boldsymbol{f}}, s(\boldsymbol{x}))|^2 \mathbf{1}_E\\
&+ \sup \Big(|m(\boldsymbol{x})| + |m^*(\boldsymbol{x})|\Big)^2\pr(E^c)\\
&\lesssim \delta_a + (M + M^*)^2 \mathbb{P}(E^c)\\
&\lesssim \delta_a + \pr(\sqrt{\norm{\boldsymbol{H}^\dagger \boldsymbol{\xi}}_2^2 + \norm{[\boldsymbol{B}_{\pazocal{J}, :}] \boldsymbol{H}^\dagger \boldsymbol{\xi}}_2^2}\geq b)
\end{aligned}    
\end{equation*}

By Markov's inequality, we can further bound as
\begin{equation}
\label{eq:appEq5}
\begin{aligned}
\EE |m(\boldsymbol{x}) - m^*(\boldsymbol{x})|^2&\lesssim \delta_a + \frac{1}{b^2} \EE[\norm{\boldsymbol{H}^\dagger \boldsymbol{\xi}}_2^2 + \norm{[\boldsymbol{B}_{\pazocal{J},:}] \boldsymbol{H}^\dagger \boldsymbol{\xi}}_2^2]\\
&\lesssim \delta_a + \delta_f.
\end{aligned}
\end{equation}
Combining \eqref{eq:appEq3} and \eqref{eq:appEq5} yields  
$$\mathbb{E} |m(\boldsymbol{x}) - g^Q(\boldsymbol{f}, \boldsymbol{u}_{\pazocal{J} \cup \pazocal{J}^P})|^2 \lesssim \delta_a + \delta_f.$$  
To complete Case 1, it remains to verify that $C_1 \frac{|\pazocal{J}|r}{\nu_{\min}(\boldsymbol{H})} \leq T$ when $2(|\pazocal{J}| + r) \leq N$ to ensure that $m$ belongs to $\pazocal{F}_s$. This condition is established through the padding argument in Section I.1 of \cite{fan2024factor}. 

\paragraph{Case 1: $r = 0$.} We have $\delta_f = 0$ and $\boldsymbol{x} = \boldsymbol{u}$ in this case. We choose $g \in \pazocal{G}(L - 1, r + |\pazocal{J}| + 1, N, M, T)$ that minimizes $\sup_{\kappa \in [-M, M]} \norm{g(\boldsymbol{x}_{\pazocal{J}}, \kappa) - h(\boldsymbol{x}_{\pazocal{J}}, \kappa)}_\infty^2$. Then, we have
$$
|g(\boldsymbol{x}_{\pazocal{J}}, s(\boldsymbol{x}))-h(\boldsymbol{x}_{\pazocal{J}}, s(\boldsymbol{x}))| \lesssim \sqrt{\delta_a}, \quad \forall \boldsymbol{x}_{\pazocal{J}} \in [-2b, 2b]^{|\pazocal{J}|}
$$
following the definition of $\delta_a$. The construction of $m$ is similar to Case 1. If $|\pazocal{J}| \leq N$, it follows from the padding argument in Section I.1 of \cite{fan2024factor} that there exists some $m \in \pazocal{F}_s$ and $\widetilde{g} \in \pazocal{G}(L, \overline{r} + N + 1, N, M, T)$ such that
$$m(\boldsymbol{x}) = \widetilde{g}(\widetilde{\boldsymbol{f}}, \widetilde{\boldsymbol{\Theta}}^\top \boldsymbol{x}, s(\boldsymbol{x})) = g(\boldsymbol{x}_{\pazocal{J}}, s(\boldsymbol{x})), \quad \widetilde{\boldsymbol{\Theta}}_{ij} = \mathbf{1}\{i \leq |\pazocal{J}|, j = l_i\}.$$ 
Therefore, we have
$$\EE |m(\boldsymbol{x}) - g^Q(\boldsymbol{f}, \boldsymbol{u}_{\pazocal{J} \cup \pazocal{J}^P})|^2 = \EE |m(\boldsymbol{x}) - h(\boldsymbol{x}_{\pazocal{J}}, s(\boldsymbol{x}))| \lesssim \delta_a,$$
which completes the proof of Case 2. We denote this $m$ as $\widetilde{m}$ for the rest of the proof.

Combining Cases 1 and 2 above, we know that there must exist some $\widetilde{m}(\boldsymbol{x}; \boldsymbol{W}, \widetilde{\Theta}, \widetilde{g}, s) \in \pazocal{F}_s$ with 
$$\norm{\widetilde{\boldsymbol{\Theta}}}_0\leq |\pazocal{J}|\mathbf{1}\{\delta_f + \delta_a \leq \delta^0_f + \delta^0_a\}+ |\pazocal{J} \cup \pazocal{J}^P|\mathbf{1}\{\delta_f + \delta_a > \delta^0_f + \delta^0_a\}$$ 
such that 
$$\pazocal{E}(\widetilde{m})\lesssim (\delta_f + \delta_a) \land (\delta^0_f + \delta^0_a).$$ 

\paragraph{Step II: Derive the basic inequality.}
Define the quantity
\begin{equation}
\label{eq:fast-rate-Pi}
\Pi := |\pazocal{J} | \mathbf{1} \{\delta_f + \delta_a\leq \delta^0_f + \delta^0_a\} + |\pazocal{J} \cup \pazocal{J}^P| \mathbf{1} \{\delta_f + \delta_a > \delta^0_f + \delta^0_a\}.
\end{equation}
It follows from \eqref{eq:FTFAST6} and \eqref{eq:FTFAST7} that
$$
\frac{1}{n}\sum_{i = 1}^n (y_i - \widehat{m}(\boldsymbol{x}_i))^2 + \lambda \sum_{i,j} \psi_\tau(\widehat{\boldsymbol{\Theta}}_{ij}) \leq \frac{1}{n} \sum_{i = 1}^n (y_i - \widetilde{m}(\boldsymbol{x}_i))^2 + \lambda \Pi.
$$
Plugging in the formula $y_i = g^Q(\boldsymbol{f}, \boldsymbol{u}_{\pazocal{J} \cup \pazocal{J}^P}) + \epsilon_i$, we further have
\begin{equation}
\label{eq:estEq1}
\norm{\widehat{m} - g^Q}_n^2 + \lambda \sum_{i, j} \psi_\tau (\widehat{\boldsymbol{\Theta}}_{ij}) \leq \norm{\widetilde{m} - g^Q}_n^2 + \frac{2}{n}\sum_{i = 1}^n \epsilon_i (\widehat{m}(\boldsymbol{x}_i) - \widetilde{m}(\boldsymbol{x}_i)) + \lambda \Pi.
\end{equation}
From the triangle inequality, we have
\begin{equation}
\label{eq:estEq2}
\norm{\widehat{m} - \widetilde{m}}_{n}^2 \leq 2\norm{\widehat{m} -g^Q}_{n}^2 + 2\norm{\widetilde{m} - g^Q}_{n}^2.
\end{equation}
Combining \eqref{eq:estEq1} and \eqref{eq:estEq2}, we derive the basic inequality
\begin{equation}
\label{eq:estBasic}
\norm{\widehat{m} - \widetilde{m}}_{n}^2 + 2\lambda \sum_{i, j} \psi_\tau (\widehat{\boldsymbol{\Theta}}_{ij}) \leq 4\norm{\widetilde{m} - g^Q}_{n}^2 + \frac{4}{n}\sum_{i = 1}^n \epsilon_i (\widehat{m}(\boldsymbol{x}_i) - \widetilde{m}(\boldsymbol{x}_i)) + 2 \lambda \Pi.
\end{equation}

\paragraph{Step III: Bound the empirical approximation error.}
Let $z_i = |\widetilde{m}(\boldsymbol{x}_i) - g^Q(\boldsymbol{f}_i, \boldsymbol{u}_{\pazocal{J} \cup \pazocal{J}^P})|$. It is obvious that $\{z_i\}_{i = 1, \dots, n}$ are i.i.d. samples with
$$|z_i| \leq (M + M^*)^2, \quad \text{var}(z_i) \leq \EE|z_i|^2 = \pazocal{E}(\widetilde{m}).$$
Since $\boldsymbol{W}$ is independent of $\{x_i\}_{i = 1}^n$, $\{\widetilde{m}(\boldsymbol{x}_i)\}_{i = 1}^n$ are also mutually independent. Therefore, the Bernstein inequality shows that for any $t > 0$, with probability at least $1-e^{-t}$ with respect to the target data, we have
$$
\begin{aligned}
\norm{\widetilde{m}-g^Q}_{n}^2=&\frac{1}{n}\sum_{i=1}^n |z_i|^2
\lesssim (\frac{1}{n}\sum_{i=1}^n \EE[z_i])^2+C_2\Big(\sqrt{\ca E(\widetilde{m})\frac{t}{n}}+\frac{t}{n}\Big)\\
\lesssim & \ca E(\widetilde{m}) + \frac{t}{n}\le C_3\Big( (\delta_f + \delta_a)\land (\delta^0_f + \delta^0_a) + \frac{t}{n}\Big)
\end{aligned}
$$
for some universal constants $C_2$ and $C_3$. Define the event
\begin{equation}
\label{eq:estE1}
E_1(t) = \Big\{\norm{\widetilde{m}-g^Q}_{n}^2\leq C_3\Big( (\delta_f + \delta_a)\land (\delta^0_f + \delta^0_a) + \frac{t}{n}\Big)\Big\},
\end{equation}
then $\pr(E_1(t))\geq 1-e^{-t}$.

\paragraph{Step IV: Bound the empirical excess risk.}
Define the event
$$
\begin{aligned}
E_2(t) = \Big\{\forall m(\boldsymbol{x}; \boldsymbol{W}, g, \boldsymbol{\Theta}, s) \in \pazocal{F}_s,\ &\frac{4}{n} \sum_{i = 1}^{n} \varepsilon_i (m(\boldsymbol{x}_i) - \widetilde{m}(\boldsymbol{x}_i)) - \lambda \sum_{i,j} \psi_\tau({\boldsymbol{\Theta}}_{ij}) \\
    \le& \frac{1}{2} \norm{m - \widetilde{m}}_{n}^2 + 2c_1 \left(v_{n} + \varrho_{n} + \frac{t}{n} \right)\Big\},
\end{aligned}
$$
where $c_1$ is the universal constant defined in Lemma \ref{lemma:weighted-empirical-process-regularized}. By Lemma \ref{lemma:weighted-empirical-process-regularized}, we have $\mathbb{P}(E_2(t)) \geq 1 - e^{-t}$ provided $\lambda \geq C_4 \varrho_n$ for some universal constant $C_4$. This is achievable by setting the universal constant $c_2$ in Lemma \ref{lemma:fast} sufficiently large.

Hence, given the definition of $v_n, \varrho_n, \Pi$ (Equation \eqref{eq:fast-rate-varrho_n}, \eqref{eq:fast-rate-v_n} and \eqref{eq:fast-rate-Pi}), It is notable theta
$$\varrho_n\lesssim \lambda,\quad (\delta_f + \delta_a)\land (\delta^0_f + \delta^0_a) + v_n+\lambda \Pi\lesssim (\delta_f + \delta_a + \delta_s) \land (\delta^0_f + \delta^0_a + \delta^0_s).$$
Under $E_1(t) \cap E_2(t)$, \eqref{eq:estBasic} thus reduces to
\begin{equation}
\label{eq:estBasic2}
\begin{aligned}
&\norm{\widehat{m} - \widetilde{m}}_{n}^2 + \lambda \sum_{i,j}\psi_\tau(\widehat{\boldsymbol{\Theta}}_{ij})\leq 4C_3\Big( (\delta_f + \delta_a)\land (\delta^0_f + \delta^0_a)\Big) + \frac{1}{2}\norm{m - \widetilde{m}}_{n}^2 + 2c_1 \Big(v_{n} + \varrho_{n} + \frac{t}{n} \Big) + 2\lambda\Pi\\
\Longrightarrow & \norm{\widehat{m} - \widetilde{m}}_{n}^2 + 2\lambda\sum_{i,j}\psi_\tau(\widehat{\boldsymbol{\Theta}}_{ij})\leq 8C_3\Big( (\delta_f + \delta_a)\land (\delta^0_f + \delta^0_a)\Big) + 4c_1 \Big(v_{n} + \varrho_{n} + \frac{t}{n} \Big) + 4\lambda\Pi\\
&~~~~~~~~~~~~~~~~~~~~~~~~~~~~~~~~~~~~\lesssim (\delta_f + \delta_a + \delta_s) \land (\delta^0_f + \delta^0_a+ \delta^0_s) + \frac{t}{n},
\end{aligned}
\end{equation}
Combining \eqref{eq:estE1} and \eqref{eq:estBasic2}, we have
\begin{equation}
\label{eq:estBasic3}
\norm{\widehat{m} - g^Q}_n^2\lesssim \norm{\widetilde{m}-g^Q}_{n}^2+\norm{\widehat{m}-\widetilde{m}}_{n}^2\lesssim (\delta^0_f + \delta^0_a+ \delta^0_s) + \frac{t}{n}
\end{equation}
under $E_1(t)\cap E_2(t)$, which completes the proof of upper bounding the empirical excess risk since $\pr(E_1(t)\cap E_2(t))\geq 1-2e^{-t}$.

\paragraph{Step V: Bound the excess risk.}
It remains the upper bound the excess risk (population error) $\pazocal{E}(\widehat{m}) = \norm{\widehat{m}-g^Q}_2^2$ to complete the proof of Lemma \ref{lemma:fast}. Define the event
$$E_3(t) = \left\{\forall m(\boldsymbol{x}; \boldsymbol{W}^Q, g, \boldsymbol{\Theta}, s)\in \pazocal{F}_s, ~~ \frac{1}{2}\norm{m - \widetilde{m}}^2_2 \le \norm{m - \widetilde{m}}_{n}^2 + 2\lambda \sum_{i,j} \psi_\tau({\boldsymbol{\Theta}}_{ij}) + c_1 \left(v_{n} + \varrho_{n} + \frac{t}{n}\right)\right\},$$
where $c_1$ is the universal constant defined in Lemma \ref{lemma:equivalence-population-empirical-l2-regularized}. By Lemma \ref{lemma:equivalence-population-empirical-l2-regularized}, as long as \(\lambda \geq C_5 \varrho_n\) for some universal constant \(C_5\), we have \(\pr(E_3(t)) \geq 1 - e^{-t}\). This is achievable by setting the universal constant \(c_2\) in Lemma \ref{lemma:fast} sufficiently large. 

Combining \eqref{eq:estBasic2} and the definition of $E_3(t)$, we have that under $E_1(t)\cap E_2(t)\cap E_3(t)$,
$$
\begin{aligned}
&\norm{\widehat{m}-\widetilde{m}}_2^2\lesssim (\delta_f + \delta_a +\delta_s)\land (\delta^0_f + \delta^0_a+ \delta^0_s) + \frac{t}{n}\\
\Longrightarrow& \norm{\widehat{m}-g^Q}_2^2\lesssim \norm{\widetilde{m}-g^Q}_{2}^2+\norm{\widehat{m}-\widetilde{m}}_{2}^2\lesssim  (\delta_f + \delta_a +\delta_s)\land (\delta^0_f + \delta^0_a+ \delta^0_s) + \frac{t}{n},
\end{aligned}
$$
which completes the proof as $\pr(E_1(t)\cap E_2(t)\cap E_3(t))\geq 1-3e^{-t}$.
\end{proof}

\begin{proof}[Proof of Corollary \ref{col:fast}]
The result follows directly by setting \(s = g^P = 0\) and \(\ca{J}^P = \emptyset\) in Lemma \ref{lemma:fast}. Under this setting, \(\delta_a + \delta_f + \delta_s\) is equal to \(\delta_a^0 + \delta_f^0 + \delta_s^0\), and the result holds by treating the sample for non-parametric regression as the source data and letting $h$ in Lemma \ref{lemma:fast} be $g^P$.
\end{proof}
% !TEX root = ../main.tex
\section{Empirical Implementation Details}
\label{app:numerical}

\subsection{Hyperparameter Specifications}
\noindent\textbf{Network Architecture.} All neural networks use depth $L = 4$ hidden layers with width $N = 100$ neurons. Factor-augmented models (FAST-NN and FAN-Lasso) use sparsity level $s = 50$, projection dimension $\overline{r} = 10$ for simulations ($\overline{r} = 3$ for real data).

\noindent\textbf{Regularization.} Factor-augmented models employ clipped $\ell_1$ regularization with threshold $\tau = 0.01$ and adaptive penalty weight $\lambda = 1.3 \log(p) / n_{\text{train}}$. 

\noindent\textbf{Optimization.} Adam optimizer with learning rate $\alpha = 0.001$, batch size $64$, maximum $200$ epochs, and mean squared error (MSE) loss. 

\subsection{Communities and Crime Dataset Attributes}

The original dataset from the UCI Machine Learning Repository contains 122 predictive attributes, 5 non-predictive identifiers, and 1 goal variable (\texttt{ViolentCrimesPerPop}). After removing the non-predictive identifiers (\texttt{state}, \texttt{county}, \texttt{community}, \texttt{communityname}, \texttt{fold}) and attributes with excessive missingness, the cleaned dataset retains 100 predictive features across 1,993 communities. All numeric features are pre-normalized to the $[0, 1]$ interval. 

The retained attributes by category are summarized in Table~\ref{tab:dataset-attributes}.

\begin{table}[htbp]
\centering
\small
\begin{tabular}{p{0.25\textwidth}p{0.7\textwidth}}
\hline
\textbf{Category} & \textbf{Variables} \\
\hline
Demographics \& Population (12) & population, householdsize, racepctblack, racePctWhite, racePctAsian, racePctHisp, agePct12t21, agePct12t29, agePct16t24, agePct65up, numbUrban, pctUrban \\
\hline
Socio-Economic (26) & medIncome, pctWWage, pctWFarmSelf, pctWInvInc, pctWSocSec, pctWPubAsst, pctWRetire, medFamInc, perCapInc, whitePerCap, blackPerCap, indianPerCap, AsianPerCap, OtherPerCap, HispPerCap, NumUnderPov, PctPopUnderPov, PctLess9thGrade, PctNotHSGrad, PctBSorMore, PctUnemployed, PctEmploy, PctEmplManu, PctEmplProfServ, PctOccupManu, PctOccupMgmtProf \\
\hline
Family Structure (13) & MalePctDivorce, MalePctNevMarr, FemalePctDiv, TotalPctDiv, PersPerFam, PctFam2Par, PctKids2Par, PctYoungKids2Par, PctTeen2Par, PctWorkMomYoungKids, PctWorkMom, NumIlleg, PctIlleg \\
\hline
Immigration (16) & NumImmig, PctImmigRecent, PctImmigRec5, PctImmigRec8, PctImmigRec10, PctRecentImmig, PctRecImmig5, PctRecImmig8, PctRecImmig10, PctSpeakEnglOnly, PctNotSpeakEnglWell, PctForeignBorn, PctBornSameState, PctSameHouse85, PctSameCity85, PctSameState85 \\
\hline
Housing (26) & PctLargHouseFam, PctLargHouseOccup, PersPerOccupHous, PersPerOwnOccHous, PersPerRentOccHous, PctPersOwnOccup, PctPersDenseHous, PctHousLess3BR, MedNumBR, HousVacant, PctHousOccup, PctHousOwnOcc, PctVacantBoarded, PctVacMore6Mos, MedYrHousBuilt, PctHousNoPhone, PctWOFullPlumb, OwnOccLowQuart, OwnOccMedVal, OwnOccHiQuart, RentLowQ, RentMedian, RentHighQ, MedRent, MedRentPctHousInc, MedOwnCostPctInc, MedOwnCostPctIncNoMtg \\
\hline
Law Enforcement \& Geography (7) & LemasSwornFT, LemasSwFTPerPop, LemasSwFTFieldOps, LemasSwFTFieldPerPop, LemasTotalReq, LemasPctPolicOnPatr, LemasGangUnitDeploy \\
\hline
\end{tabular}
\caption{Cleaned Dataset Attributes by Category}
\label{tab:dataset-attributes}
\end{table}

\end{document}